\documentclass{article}

\PassOptionsToPackage{numbers, sort, comma}{natbib}
\usepackage[preprint]{neurips_2026}

\usepackage[utf8]{inputenc} 
\usepackage[T1]{fontenc}    
\usepackage{hyperref}       
\usepackage{url}            
\usepackage{booktabs}       
\usepackage{amsmath}        
\usepackage{amsfonts}       
\usepackage{amsthm}         
\usepackage{nicefrac}       
\usepackage{microtype}      
\usepackage[table,svgnames]{xcolor} 
\usepackage{wrapfig}
\usepackage[most]{tcolorbox} 
\usepackage{tikz} 
\usepackage{graphicx}
\usepackage{multirow}
\usepackage{rotating}
\usepackage{makecell}
\usepackage[table]{xcolor}
\usepackage{colortbl}
\usepackage{array}
\usepackage{caption}
\usepackage{pythonhighlight}
\usepackage{algorithm}
\usepackage[noEnd=false,indLines=true,commentColor=gray,italicComments=false]{algpseudocodex}
\usepackage{svg}
\tcbuselibrary{skins, breakable}

\definecolor{darkgray}{HTML}{2D2D2D}
\definecolor{coolgray}{HTML}{8A9BB0}
\definecolor{promptbg}{HTML}{F4F6F8}
\definecolor{promptframe}{HTML}{C8D0DA}

\newtcolorbox{prompt}[1][llm-prompt]{%
  enhanced,
  breakable,
  arc=5pt,
  colback=promptbg,
  colframe=promptframe,
  boxrule=1pt,
  left=12pt, right=12pt, top=10pt, bottom=10pt,
  fontupper={\ttfamily\small\color{darkgray}},
  before upper={%
    {\ttfamily\footnotesize\color{coolgray}\$ #1}\par\smallskip%
  },
}

\newcommand{\ph}[1]{\textcolor{gray}{\{#1\}}}

\definecolor{rankone}{HTML}{85B7EB}   
\definecolor{ranktwo}{HTML}{CCE2F8}   
\definecolor{rank1}{HTML}{F3C99A}     
\definecolor{rank2}{HTML}{F9E6D2}     
 
\newcolumntype{C}[1]{>{\centering\arraybackslash}p{#1}}
\newcolumntype{V}[1]{>{\centering\arraybackslash}m{#1}}
\newcolumntype{L}[1]{>{\raggedright\arraybackslash}p{#1}}

\theoremstyle{plain}
\newtheorem{proposition}{Proposition}
\newtheorem{corollary}{Corollary}
\newtheorem{finding}{Question}

\newcommand{\cnum}[1]{%
  \tikz[baseline=(C.base)]\node[draw,circle,inner sep=0.6pt] (C) {\scriptsize #1};%
}

\newtcolorbox{corollarybox}{%
  enhanced,
  colback=gray!10,
  colframe=gray!60,
  leftrule=2pt,
  rightrule=0pt,
  toprule=0pt,
  bottomrule=0pt,
  arc=0pt,
  boxsep=0pt,
  left=8pt,
  right=8pt,
  top=6pt,
  bottom=6pt
}

\definecolor{findingbg}{RGB}{235,247,249}
\definecolor{findingframe}{RGB}{158,218,225}
\newtcolorbox{findingbox}{%
  enhanced,
  colback=findingbg,
  colframe=findingframe,
  leftrule=2pt,
  rightrule=0pt,
  toprule=0pt,
  bottomrule=0pt,
  arc=0pt,
  boxsep=2pt,
  left=8pt,
  right=8pt,
  top=6pt,
  bottom=6pt
}
\newenvironment{bbox}
  {\begin{findingbox}\begin{finding}}
  {\end{finding}\end{findingbox}}

\makeatletter
\newcommand{\appendixcontents}{%
  \section*{Table of Contents}
  \@starttoc{apc}%
}
\newcommand{\appsection}[1]{%
  \section{#1}%
  \addcontentsline{apc}{section}{\protect\numberline{\thesection}#1}%
}
\newcommand{\appsubsection}[1]{%
  \subsection{#1}%
  \addcontentsline{apc}{subsection}{\protect\numberline{\thesubsection}#1}%
}
\makeatother

\title{Back to the Beginning of Heuristic Design:\\Bridging Code and Knowledge with LLMs}

\author{%
  Nguyen Viet Tuan Kiet$^1$ \quad Bui Dinh Pham$^1{}^\star$ \quad Dao Van Tung$^1{}^\star$ \\
  \textbf{Tran Cong Dao}$^2{}^\star$ \quad \textbf{Huynh Thi Thanh Binh}$^1{}^\dagger$ \\
  $^1$Hanoi University of Science and Technology, Hanoi, Vietnam\\
  $^2$University of Sydney, Sydney, Australia\\
  ${}^\star$Co-second authors $^\dagger$Corresponding author: \texttt{binhht@soict.hust.edu.vn}
}

\begin{document}

\maketitle

\begin{abstract}
Large language models (LLMs) have recently advanced automatic heuristic design (AHD) for combinatorial optimization (CO), where candidate heuristics are iteratively proposed, evaluated, and refined. Most existing approaches search over executable programs and distill insights from execution feedback to guide later iterations. Because this process moves from low-level implementations to high-level principles, we refer to it as a \textit{bottom-up} paradigm. We argue that this view is incomplete and introduce a complementary \textit{top-down} perspective: knowledge becomes the primary search object and code merely instantiates and tests it, making what is learned explicit and reusable across problems and trajectories. We formalize this shift through a statistical-learning view that exposes a distortion--compression trade-off, and instantiate it in both population-based and tree-based AHD frameworks. Across CO and tasks beyond it, knowledge-first search improves discovery efficiency, transfer, and generalization, often outperforming code-centric pipelines, while combining both strategies yields further gains. Our results suggest that progress in AHD depends on iteratively constructing and evolving interpretable hypotheses that retain value beyond a single search trajectory.

\end{abstract}

\section{Introduction}

The rapid progress of large language models (LLMs) has catalyzed a broad spectrum of research directions centered on automating the discovery of structured computational artifacts, including automatic algorithm design (AAD)~\cite{liu2024evolution,novikov2025alphaevolve,agrawal2025gepa}, reward design (RD) for reinforcement learning~\cite{maeureka}, symbolic regression (SR)~\cite{holt2024automatically}, and scientific discovery~\cite{shojaee2025llmsr}. This momentum stems from a distinctive combination of capabilities: strong code synthesis~\cite{jiang2026survey}, access to extensive prior knowledge, and the ability to iteratively refine outputs through lightweight feedback and in-context adaptation~\cite{brown2020language,kojima2022large}. Together, these properties position LLMs not merely as passive generators of proposals, but as active components in search and optimization processes that evolve over multiple rounds.

Within this landscape, automatic heuristic design (AHD)~\cite{liu2024evolution,ye2024reevo} has emerged as a particularly effective paradigm for tackling combinatorial optimization (CO) problems. In AHD, the goal is to discover high-performing heuristics that generalize across problem instances, typically as components embedded within a broader solver. Solvers range from classical metaheuristics, such as ant colony optimization (ACO)~\cite{585892}, guided local search (GLS)~\cite{Voudouris2003} and large neighborhood search (LNS)~\cite{ye2025large}, to hybrid frameworks that combine neural combinatorial optimization (NCO)~\cite{kwon2020pomo} with heuristic signals. Rather than replacing these solvers, AHD aims to enhance their performance by refining key decision rules or scoring functions.

Recent LLM-driven frameworks instantiate heuristic design as a closed-loop process, where candidate solutions are proposed, evaluated, and iteratively refined. More recently, this paradigm has expanded into a diverse methodological landscape, encompassing evolutionary algorithms~\cite{liu2023algorithm,liu2024evolution,ye2024reevo,chentongchen2026hifoprompt}, Monte Carlo tree search~\cite{zheng2025monte,nguyen2026motif}, meta-optimization frameworks~\cite{shi2026generalizable}, and reinforcement learning-based approaches~\cite{huang2026calm}. These methods differ in how they explore the search space and leverage feedback signals, leading to a broad spectrum of design choices for LLM-driven search. Empirical results show that such pipelines can match or even surpass human-designed heuristics on standard benchmarks, while significantly reducing manual effort. 

Despite these advances, the dominant view of these systems remains largely operational: LLMs are treated as generators of candidate programs, and performance gains are attributed to faster exploration of the search space. This view obscures a more fundamental mechanism. In practice, LLM-guided search is rarely memoryless; it relies heavily on intermediate artifacts---reflections, critiques, summaries, and abstractions---that accumulate across iterations and implicitly shape future decisions. We argue that these artifacts are not incidental scaffolding but constitute a form of search knowledge that deserves to be treated as a first-class object of study.

Building on this observation, we contrast two complementary paradigms. In the prevailing \emph{bottom-up} view, candidate code is the primary search object and knowledge is inferred post-hoc from explored solutions; the quality and scope of acquired knowledge are therefore entangled with the trajectory of generated code. We instead propose a \emph{top-down} alternative in which knowledge is the primary search object: the algorithm searches in the space of reusable abstractions, while code is generated only as an instantiation for evaluation, and feedback updates the belief over knowledge rather than synthesizing new knowledge from code.

\paragraph{Contributions.}
In summary, we make four contributions:
\cnum{1} \textbf{\textit{Conceptually}}, we identify the distinction between code-centric and knowledge-centric search in LLM-based AHD, thereby making explicit a design axis that prior work largely leaves implicit, and reinterpret existing frameworks as predominantly bottom-up systems where knowledge is induced as a byproduct of evaluated programs.
\cnum{2} \textbf{\textit{Theoretically}}, we formulate AHD from a statistical learning perspective and show that top-down search induces a distortion--compression trade-off: it may lose fine-grained implementation details, but can reduce adaptive complexity by searching over a more compact knowledge space.
\cnum{3} \textbf{\textit{Methodologically}}, we propose top-down variants of population-based and tree-based AHD as a complementary paradigm that fills the underexplored role of knowledge as a first-class search object, and further introduce a dual code--knowledge variant that combines the strengths of both paradigms.
\cnum{4} \textbf{\textit{Empirically}}, we evaluate these paradigms across diverse settings, including cross-domain transfer, dual code--knowledge design, sparse evaluation, and tasks beyond CO, showing that top-down search improves efficiency, transferability, and generalization across both CO and non-CO benchmarks.

\section{Related Works}

\paragraph{Automatic Heuristic Design (AHD).}
AHD, originally studied under hyper-heuristics~\cite{burke2013hyper}, aims to automatically design heuristic components within a solver for a target problem class. Early GP-based methods~\cite{burke2009exploring} required carefully crafted representations, while recent LLM-based frameworks such as FunSearch~\cite{romera2024mathematical} enable program-level heuristic search at a much larger scale. We provide a more detailed discussion in Appendix~\ref{app:related_works}.

\emph{Population-based AHD} maintains and evolves a set of candidate programs. AEL~\cite{liu2023algorithm} and EoH~\cite{liu2024evolution} encode heuristics as searchable code functions, while ReEvo~\cite{ye2024reevo} introduces reflective evolution to refine programs using execution feedback. Later methods improve this code-centric search process through diversity control, performance prediction, or richer prompting, as in HSEvo~\cite{dat2025hsevo}, Hercules~\cite{wu2025efficient}, and HiFo~\cite{chentongchen2026hifoprompt}. However, even when these methods use reflections or knowledge pools, such insights are typically induced from evaluated code and used as auxiliary context for the same search process, rather than treated as the primary object being searched and transferred.

\emph{Tree-based AHD} methods replace population evolution with structured exploration. MCTS-AHD~\cite{zheng2025monte} combines Monte Carlo Tree Search with evolutionary operators to explore heuristic programs more comprehensively, allowing weak early candidates to remain available for later expansion. MOTIF~\cite{nguyen2026motif} extends this direction to multi-component heuristic design through multiple trees and game-inspired LLM-agent interactions. These methods improve how the code space is explored, but the tree nodes and search decisions are still organized around candidate heuristics or solver components, leaving the role of knowledge as a secondary mechanism rather than a first-class search object.

\paragraph{LLMs for Code Generation.}
Our top-down view is motivated by evidence that LLMs often benefit from explicit intermediate structures before producing final code. Least-to-Most Prompting~\cite{zhou2023leasttomost} and Parsel~\cite{zelikman2023parsel} show that decomposition can improve complex reasoning and algorithmic implementation. In code generation, planning-based methods such as \citet{zhang2023planning}, CodePlan~\cite{wen2025codeplan}, and Reasoning-as-Logic-Units~\cite{li2025reasoning} further show that plans, pseudocode, or logic units can guide program synthesis more effectively than direct decoding. These results support our premise that abstract knowledge can be useful before code is generated; however, they do not study how such knowledge should be searched, evaluated, and evolved in AHD.

\section{Bottom-Up and Top-Down AHD}
\label{main:method}

\subsection{Preliminary}

\paragraph{Notation.}
We represent a CO domain by $d=(\mathcal{X}_d,\mathcal{Y}_d,f_d)$, where $\mathcal{X}_d$ is the instance space, $\mathcal{Y}_d$ is the solution space, $\mathcal{Y}_d(\mathbf{x})\subseteq\mathcal{Y}_d$ is the set of feasible solutions for $\mathbf{x}\in\mathcal{X}_d$, and $f_d:\mathcal{X}_d\times\mathcal{Y}_d\to\mathbb{R}$ is the objective function, following recent work~\cite{nguyen2026motif}. Without loss of generality, we focus on minimization; maximization problems can be reduced to minimization by replacing $f_d$ with $-f_d$.

A solver family is parameterized by $\theta\in\Theta$, where each $\theta$ induces a solver $\pi_\theta:\mathcal{X}_d\to\mathcal{Y}_d$. The parameter $\theta$ may represent hyperparameters~\cite{xu2026autoep}, heuristic scoring rules~\cite{liu2024evolution,ye2024reevo,zheng2025monte}, program components, or combinations thereof~\cite{nguyen2026motif}. For an instance $\mathbf{x}\in\mathcal{X}_d$, the induced loss of $\theta$ is $\ell_d(\theta;\mathbf{x}) := f_d(\mathbf{x},\pi_\theta(\mathbf{x}))$.

As a running example, in the TSP, an instance $\mathbf{x}$ is a distance matrix, a solution $\mathbf{y}\in\mathcal{Y}_d$ is a tour permutation, and $f_d(\mathbf{x},\mathbf{y})$ is the corresponding tour length. In this case, $\pi_\theta$ may be a greedy construction heuristic, where $\theta$ parameterizes the scoring rule used to select the next city.

\paragraph{Problem Formulation.}
Let $\mathcal{P}$ be an unknown distribution over instances in $\mathcal{X}_d$, and let $\mathcal{D}=\{\mathbf{x}_i\}_{i=1}^n\sim\mathcal{P}^n$ be a dataset of $n$ instances drawn i.i.d.\ from $\mathcal{P}$. For $\theta\in\Theta$, we define the population risk and empirical risk by $L_{\mathcal{P}}(\theta) := \mathbb{E}_{\mathbf{x}\sim\mathcal{P}}[\ell_d(\theta;\mathbf{x})]$ and $\widehat{L}_{\mathcal{D}}(\theta) := \frac{1}{n}\sum_{i=1}^n \ell_d(\theta;\mathbf{x}_i)$, respectively.

We consider an LLM-based AHD framework $\mathcal{F}$ that adaptively searches over a heuristic space $\Theta$ and may maintain auxiliary knowledge states in a space $\mathcal{K}$, where each $K\in\mathcal{K}$ represents an intermediate knowledge state such as a design principle, motif, rule, or textual insight. At round $t$, the framework has access to a search history $H^{t-1}$ and constructs a context for the LLM from that history. The LLM, parameterized by $\phi$, induces conditional sampling distributions over both heuristics and knowledge states, which we denote by $Q_{\phi,t}^{\Theta}(\cdot\mid\cdot)$ and $Q_{\phi,t}^{\mathcal{K}}(\cdot\mid\cdot)$, respectively.

At a high level, $\mathcal{F}$ iteratively proposes heuristic candidates, evaluates them on $\mathcal{D}$ through their empirical risks, and updates its search state accordingly. The search dynamics are therefore governed by the order in which the framework proposes heuristics, acquires empirical feedback, and updates auxiliary knowledge states. After $T$ rounds, the framework returns a candidate $\widehat{\theta}_{\mathcal{F},T}(\mathcal{D})$. The overall objective is to minimize its expected population risk:
\begin{equation}
\mathfrak{R}_n(\mathcal{F}) := \mathbb{E}_{\mathcal{D}\sim\mathcal{P}^n}\,\mathbb{E}_{\mathcal{F}(\cdot\mid\mathcal{D})}\!\left[L_{\mathcal{P}}\bigl(\widehat{\theta}_{\mathcal{F},T}(\mathcal{D})\bigr)\right],
\end{equation}
where the inner expectation is over the framework's internal randomness, including LLM sampling.

\subsection{Bottom-Up AHD}

Bottom-up AHD follows a code-first workflow: the framework first proposes a heuristic candidate, then evaluates it on $\mathcal{D}$, and finally abstracts knowledge from the resulting empirical feedback. We write $A^t=\alpha(\theta^t)$ for the heuristic artifact observed by the reflector, which may be the source code itself, a trace, or a summary presented to the LLM.

Formally, at round $t$ the framework samples and evaluates a candidate according to
\begin{equation}
\theta^t \sim Q_{\phi,t}^{\Theta}(\cdot\mid H^{t-1}), \qquad \widehat{L}^t = \widehat{L}_{\mathcal D}(\theta^t).
\end{equation}
It then constructs an intermediate context $C^t = \Psi_{\mathrm{BU}}(H^{t-1},A^t,\widehat{L}^t)$ and generates a knowledge state
\begin{equation}
K^t \sim Q_{\phi,t}^{\mathcal K}(\cdot\mid C^t,A^t,\widehat{L}^t).
\end{equation}
The round terminates with a history update $H^t = \Xi_{\mathrm{BU}}(H^{t-1},A^t,\widehat{L}^t,K^t)$. This factorization emphasizes retrospective abstraction: the framework explores $\Theta$ directly, obtains empirical evidence through evaluation, and then uses the observed artifact together with its measured quality to synthesize a higher-level design insight.

\subsection{Top-Down AHD}

Top-down AHD follows a principle-first workflow: the framework first proposes a knowledge state, then instantiates it as a heuristic candidate, and finally uses empirical evaluation to validate or refine that knowledge. Formally, at round $t$ the process starts with
\begin{equation}
K^t \sim Q_{\phi,t}^{\mathcal K}(\cdot\mid H^{t-1}).
\end{equation}
The framework then constructs an intermediate context $C^t = \Psi_{\mathrm{TD}}(H^{t-1},K^t)$ and realizes a heuristic candidate according to
\begin{equation}
\theta^t \sim Q_{\phi,t}^{\Theta}(\cdot\mid C^t,K^t), \qquad \widehat{L}^t = \widehat{L}_{\mathcal D}(\theta^t).
\end{equation}
The history is subsequently updated as $H^t = \Xi_{\mathrm{TD}}(H^{t-1},K^t,A^t,\widehat{L}^t)$. This factorization emphasizes hypothesis-driven search: rather than searching directly over all heuristics in $\Theta$, the framework first moves in the more abstract space $\mathcal K$ and then realizes a sampled knowledge state as executable code, whose empirical performance is used to assess the plausibility of that principle.


From a Bayesian perspective, top-down search induces a latent-variable model at each round. Writing $C:=H^{t-1}$, let $q_t(K\mid C)$ denote the proposal law over knowledge states and $q_t(\theta\mid K,C)$ the conditional synthesis law over heuristics. The reflector does not reason over $\theta^t$ as an extensional object; instead, it observes the artifact $A^t=\alpha(\theta^t)$ together with the empirical feedback $\widehat{L}^t=\widehat{L}_{\mathcal D}(\theta^t)$. Under this view, bottom-up refinement is naturally interpreted as approximate posterior updating:
\begin{equation}
q_t(K\mid A^t,\widehat{L}^t,C) \propto q_t(\widehat{L}^t\mid A^t,K,C)\,q_t(A^t\mid K,C)\,q_t(K\mid C).
\label{eq:bayes}
\end{equation}
Here, $q_t(K\mid C)$ is the prior over latent design principles, $q_t(A^t\mid K,C)$ captures how well $K$ explains the observed artifact, and $q_t(\widehat{L}^t\mid A^t,K,C)$ captures how well $K$ predicts the observed empirical quality given that artifact. Thus, the role of empirical feedback is not merely to rank candidates across rounds, but also to refine the posterior belief over latent design principles within each round. A detailed derivation is provided in Appendix~\ref{app:equation_6}.

This Bayesian view suggests that the advantage of top-down search depends on whether the knowledge variable $K$ provides a faithful yet compressive coordinate system for heuristic quality. To formalize this trade-off, introduce an abstraction map $\kappa:\Theta\to\mathcal K$ and a measurable realization map $g:\mathcal K\to\Theta$ satisfying $\kappa(g(K))=K$ for all $K\in\mathcal K$. The composite $g\circ\kappa$ replaces each heuristic by a canonical representative of its knowledge class. We quantify the population-level loss of this compression by
\begin{equation}
\delta_{\mathcal P}(g,\kappa) := \sup_{\theta\in\Theta}\Bigl(L_{\mathcal P}\bigl(g(\kappa(\theta))\bigr)-L_{\mathcal P}(\theta)\Bigr).
\end{equation}
The quantity $\delta_{\mathcal P}(g,\kappa)$ is small when heuristics that share the same knowledge state also have similar population quality.

For clarity, we state the exact terminal-search version. Let $\widehat K$ denote the terminal knowledge state returned by top-down search and let $\widehat\theta_{\mathrm{TD}}:=g(\widehat K)$ be its realized heuristic. Assume that
\begin{equation}
\widehat K \in \arg\min_{K\in\mathcal K}\widehat L_{\mathcal D}\bigl(g(K)\bigr),
\qquad
\widehat\theta_{\mathrm{BU}} \in \arg\min_{\theta\in\Theta}\widehat L_{\mathcal D}(\theta).
\end{equation}

\begin{proposition}[Distortion--compression trade-off between top-down and bottom-up search]
Assume that $0\le \ell_d(\theta;\mathbf{x})\le 1$ almost surely for every $\theta\in\Theta$ and $\mathbf{x}\sim\mathcal P$, and let $L_{\mathcal P}^\star:=\inf_{\theta\in\Theta}L_{\mathcal P}(\theta)$. Then
\begin{equation}
\mathbb E\bigl[L_{\mathcal P}(\widehat\theta_{\mathrm{TD}})\bigr]-L_{\mathcal P}^\star \le \delta_{\mathcal P}(g,\kappa)+\sqrt{\frac{\mathbb{I}(\widehat K;\mathcal D)}{2n}}, \qquad \mathbb E\bigl[L_{\mathcal P}(\widehat\theta_{\mathrm{BU}})\bigr]-L_{\mathcal P}^\star \le \sqrt{\frac{\mathbb{I}(\widehat\theta_{\mathrm{BU}};\mathcal D)}{2n}}.
\end{equation}
Here $\mathbb{I}(\cdot;\cdot)$ denotes mutual information. The expectations are taken over the draw of $\mathcal D$ and the internal randomness of the framework. Proof and discussion are provided in Appendix~\ref{app:proof_prop_1}.
\label{eq:prop_1}
\end{proposition}

Proposition~1 isolates the trade-off between the two paradigms. Top-down search pays a representation cost, captured by $\delta_{\mathcal P}(g,\kappa)$, because collapsing many heuristics into a single knowledge state may discard fine-grained implementation details. In return, it enjoys a potentially smaller adaptive generalization cost, since the information term is controlled at the level of the terminal knowledge state $\widehat K$ rather than the full heuristic output. Bottom-up search, by contrast, incurs no representation distortion because it operates directly in $\Theta$, but its adaptive complexity is governed by $\mathbb{I}(\widehat\theta_{\mathrm{BU}};\mathcal D)$.

Consequently, top-down AHD is favored when the knowledge representation is both coherent and compressive: coherent, in the sense that heuristics sharing the same knowledge state have similar population quality and hence $\delta_{\mathcal P}(g,\kappa)$ is small; and compressive, in the sense that the terminal knowledge state carries substantially less information about the sampled dataset than a directly selected heuristic, i.e., $\mathbb{I}(\widehat K;\mathcal D)\ll \mathbb{I}(\widehat\theta_{\mathrm{BU}};\mathcal D)$. Conversely, bottom-up AHD is preferable when heuristic quality depends strongly on low-level details that are not faithfully captured by the abstraction map $\kappa$, in which case the distortion term may dominate the benefit of compression.

\subsection{From Bottom-Up to Top-Down}

\begin{figure}[t]
  \centering%
  \includegraphics[width=\textwidth]{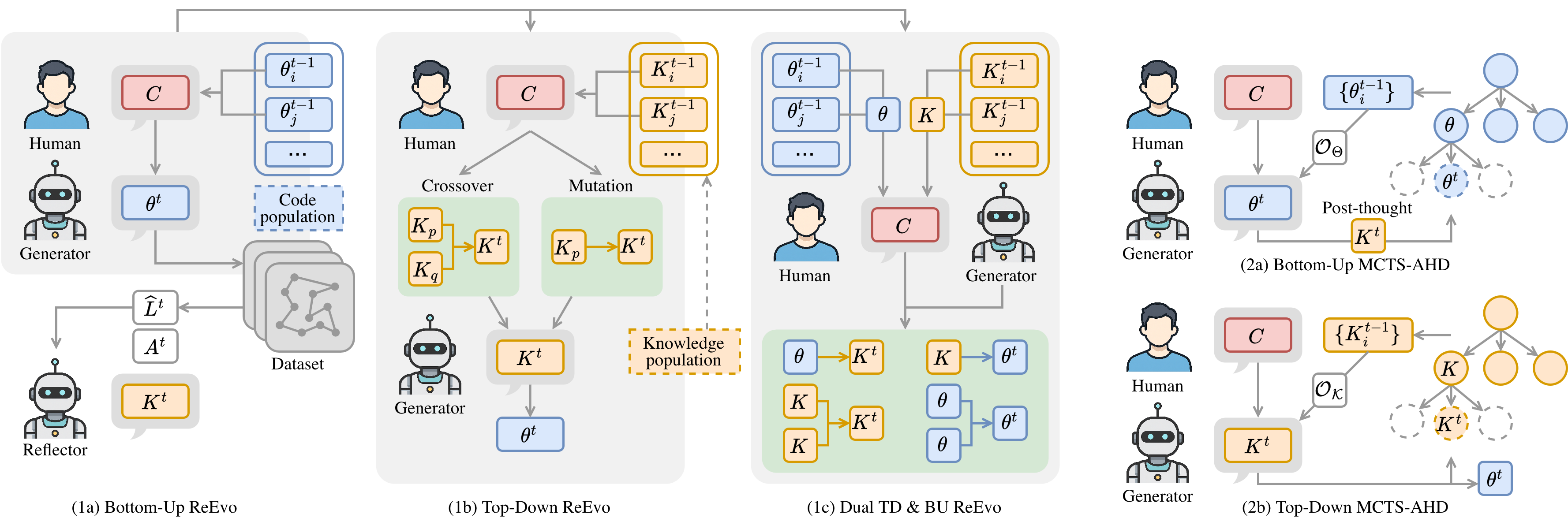}
  \caption{Comparison of bottom-up and top-down paradigms for AHD.
(1a) ReEvo BU evolves executable programs and distills knowledge from execution feedback.
(1b) ReEvo TD directly evolves a population of knowledge units via genetic operators, which are then instantiated into code.
(1c) Dual formulation jointly evolves code and knowledge with bidirectional interaction.
(2a) MCTS-AHD BU searches in the program space with post-hoc abstraction, whereas 
(2b) MCTS-AHD TD performs search directly in the knowledge space and derives executable programs from selected knowledge.}
  \label{fig:main}
\end{figure}

\paragraph{Existing AHD as Bottom-Up Search.}
The distinction above allows us to reinterpret existing AHD frameworks under a common lens. Frameworks such as FunSearch~\cite{romera2024mathematical}, ReEvo~\cite{ye2024reevo}, HSEvo~\cite{dat2025hsevo}, MCTS-AHD~\cite{zheng2025monte}, MOTIF~\cite{nguyen2026motif}, and HiFo-Prompt~\cite{chentongchen2026hifoprompt} differ in their search operators and control flow, yet they share the same fundamental priority: the primary search object is the heuristic code itself. In other words, these methods first move in the code space $\Theta$, then evaluate the resulting candidate on $\mathcal D$, and only afterward update auxiliary context or knowledge from the observed behavior. At a high level, the contrast between the two search orders can be summarized as
\begin{equation}
\boxed{\text{BU:}\; \theta^t \longrightarrow \widehat{L}^t \longrightarrow K^t \longrightarrow \theta^{t+1}}
\qquad
\boxed{\text{TD:}\; K^t \longrightarrow \theta^t \longrightarrow \widehat{L}^t \longrightarrow K^{t+1}}
\label{eq:bu_td_chain}
\end{equation}
In bottom-up AHD, knowledge $K^t$ is induced \emph{after} evaluating a realized heuristic, whereas in top-down AHD, knowledge $K^t$ is itself the state being evolved, and heuristic code is synthesized only to test whether that knowledge leads to good performance. Figure~\ref{fig:main} illustrates this contrast.

\paragraph{Top-Down Population-Based Search.}
A top-down population-based framework reverses the bottom-up priority by treating knowledge, rather than code, as the population state. Let $\mathcal P_K^t=\{K_i^t\}_{i=1}^{M}$ denote the knowledge population at generation $t$. One generation can be summarized as
\begin{equation}
\mathcal P_K^t
\xrightarrow{\text{instantiate}}
\{\theta_i^t\}_{i=1}^{M}
\xrightarrow{\text{evaluate on }\mathcal D}
\{\widehat L_{\mathcal D}(\theta_i^t)\}_{i=1}^{M}
\xrightarrow{\text{evolve in }\mathcal K}
\mathcal P_K^{t+1}.
\label{eq:td_pop_chain}
\end{equation}
Thus, evaluation is still carried out through executable code, but selection and variation are performed over knowledge states. For example, a top-down version of ReEvo~\cite{ye2024reevo} applies crossover (CX) and mutation (MT) to knowledge states rather than code candidates:
\begin{equation}
(K_p^t,K_q^t,\widehat L_{\mathcal D}(\theta_p^t),\widehat L_{\mathcal D}(\theta_q^t))
\mapsto K_{\mathrm{CX}}^{t+1},
\qquad
(K_p^t,\widehat L_{\mathcal D}(\theta_p^t))
\mapsto K_{\mathrm{MT}}^{t+1}.
\label{eq:td_pop_ops}
\end{equation}
The short-term and long-term reflection mechanisms of ReEvo are lifted in the same way: instead of reflecting on code edits directly, they summarize why certain knowledge states lead to better or worse instantiated heuristics, and the resulting reflections guide subsequent updates in $\mathcal K$.

\paragraph{Dual Top-Down and Bottom-Up Population-Based Search.}
The population-based setting can also be extended to a dual search process that maintains two populations in parallel: a knowledge population $\mathcal P_{\mathcal K}^{t}=\{K_i^t\}_{i=1}^{M_{\mathcal K}}$ and a code population $\mathcal P_{\Theta}^{t}=\{\theta_j^t\}_{j=1}^{M_{\Theta}}$. The two populations expose different search coordinates for the same objective. A knowledge individual $K_i^t$ is evaluated through its realization $g(K_i^t)$, with score $\widehat L_{\mathcal D}(g(K_i^t))$, while a code individual $\theta_j^t$ is evaluated by $\widehat L_{\mathcal D}(\theta_j^t)$.

Each iteration consists of a knowledge-side phase and a code-side phase. On the knowledge side, crossover combines two knowledge states after short-term reflection on their relative performance, while distillation (DS) converts high-performing executable behavior into a new knowledge state:
\begin{equation}
\begin{aligned}
(K_p^t,K_q^t,\widehat L_{\mathcal D}(g(K_p^t)),\widehat L_{\mathcal D}(g(K_q^t)))
&\xrightarrow{\mathrm{ST}_{\mathcal K}}
r_{\mathcal K}^t
\xrightarrow{\mathrm{CX}_{\mathcal K}}
K_{\mathrm{CX}}^{t+1},\\
(K^t,\theta^t,\widehat L_{\mathcal D}(g(K^t)),\widehat L_{\mathcal D}(\theta^t),R_{\mathrm{LT}}^t)
&\xrightarrow{\mathrm{DS}}
K_{\mathrm{DS}}^{t+1}.
\end{aligned}
\label{eq:dual_knowledge_phase}
\end{equation}
Here, $r_{\mathcal K}^t$ is a short-term reflection explaining why one knowledge state is better than another, and $R_{\mathrm{LT}}^t$ denotes the accumulated long-term reflection shared across both populations.

On the code side, crossover remains a bottom-up operator over executable heuristics, while grounding (GR) uses a strong knowledge state to generate a new code candidate:
\begin{equation}
\begin{aligned}
(\theta_p^t,\theta_q^t,\widehat L_{\mathcal D}(\theta_p^t),\widehat L_{\mathcal D}(\theta_q^t))
&\xrightarrow{\mathrm{ST}_{\Theta}}
r_{\Theta}^t
\xrightarrow{\mathrm{CX}_{\Theta}}
\theta_{\mathrm{CX}}^{t+1},\\
(K^t,\theta^t,\widehat L_{\mathcal D}(g(K^t)),\widehat L_{\mathcal D}(\theta^t),R_{\mathrm{LT}}^t)
&\xrightarrow{\mathrm{GR}}
\theta_{\mathrm{GR}}^{t+1}.
\end{aligned}
\label{eq:dual_code_phase}
\end{equation}
The short-term reflections $r_{\mathcal K}^t$ and $r_{\Theta}^t$ are then merged into the next long-term reflection $R_{\mathrm{LT}}^{t+1}$, which provides shared guidance for both DS and GR in later iterations.

\paragraph{Top-Down Tree-Based Search.}
The same shift applies to tree-based AHD. Let $\mathcal T^t$ denote the search tree at iteration $t$, where each node stores a knowledge state and its associated tree statistics. If $v_t$ is the selected leaf with knowledge state $K_{v_t}$, then top-down expansion first proposes child knowledge states, realizes them as code, and backs up their evaluated scores:
\begin{equation}
K_{v_t}
\xrightarrow{\text{expand}}
\{K_{t,j}\}_{j=1}^{m}
\xrightarrow{\text{instantiate}}
\{\theta_{t,j}\}_{j=1}^{m}
\xrightarrow{\text{evaluate}}
\{\widehat L_{\mathcal D}(\theta_{t,j})\}_{j=1}^{m}
\xrightarrow{\text{backup}}
\mathcal T^{t+1}.
\label{eq:td_tree_chain}
\end{equation}
This yields a top-down version of MCTS-AHD~\cite{zheng2025monte}: node values estimate the utility of knowledge states, expansion creates new principles rather than direct code variants, and the backed-up rewards come from the executable heuristics synthesized from those principles.

\section{Experiments}
\label{main:experiments}

We provide the definitions of all benchmark problems and algorithmic backbones in Appendices~\ref{app:problems} and~\ref{app:algorithms}, respectively. Implementation details, additional results, and qualitative analyses are reported in Appendices~\ref{app:setup},~\ref{app:add_exp}, and~\ref{app:qualitative_anal}. Overall, our experiments span more than 15 problems and 7 algorithms.


\subsection{Top-Down vs. Bottom-Up AHD}

\begin{figure}[t]
  \centering
  \begin{minipage}[c]{0.67\textwidth}
    \centering
    \includegraphics[width=\linewidth]{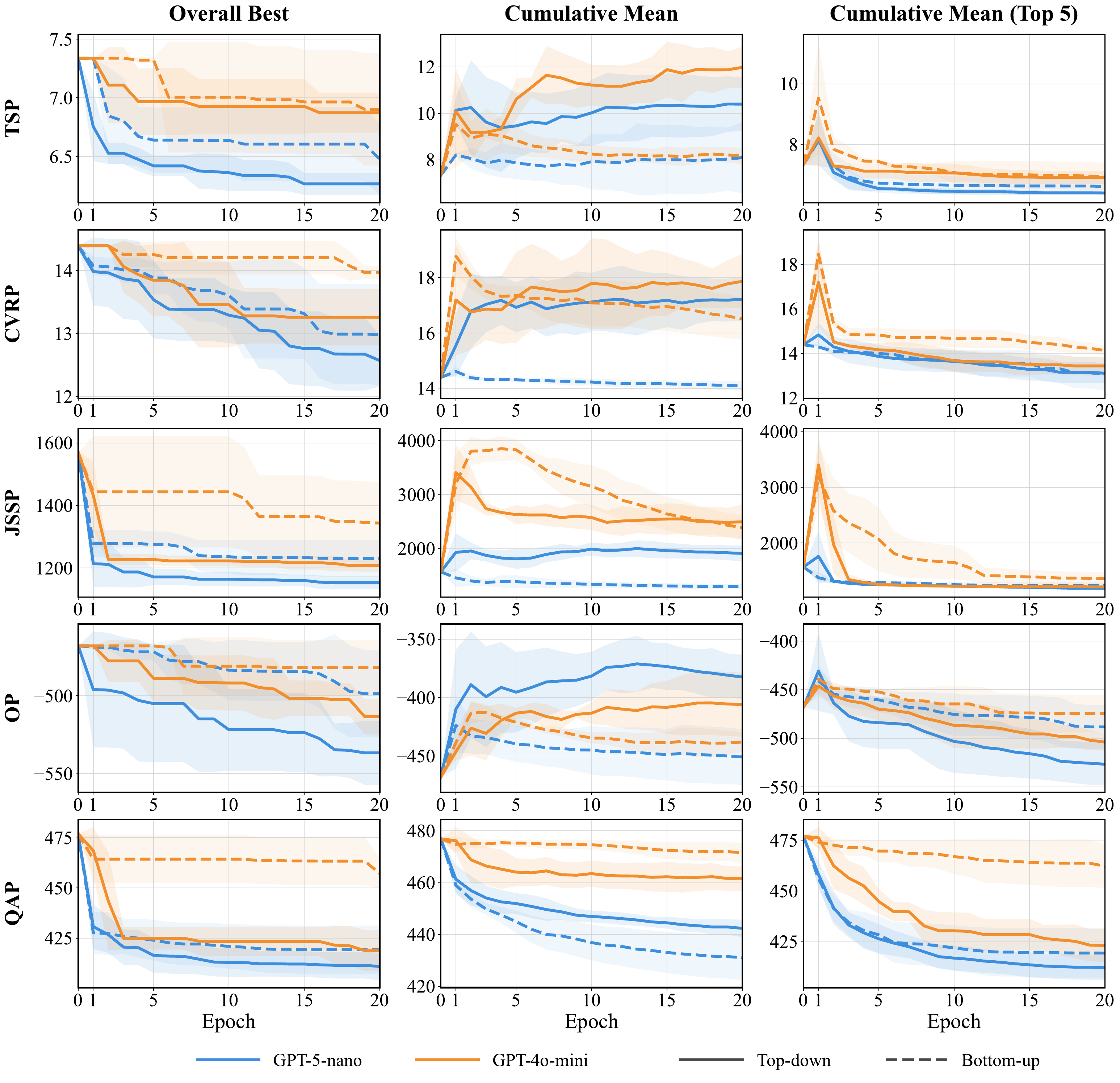}
  \end{minipage}
  \hfill
  \begin{minipage}[c]{0.32\textwidth}%
\begin{tcolorbox}[colback=white,colframe=black,boxrule=0.5pt,arc=0pt,boxsep=1pt,left=3pt,right=3pt,top=3pt,bottom=3pt]
\textbf{FunSearch BU}\par
\medskip{\scriptsize
Initialize a top-$k$ code archive $\mathcal{E}_{\theta}^{0}$.\\
For $t=1,\ldots,T$:\\
1. Propose: $\mathcal E_\theta^{t-1},K^{t-1}\rightarrow \{\theta^t_i\}$\\
2. Evaluate: $\{\theta^t_i\} \rightarrow \{\widehat{L}^t_i\}$\\
3. Update: $\mathcal E_\theta^{t-1},\{\theta^t_i,\widehat{L}_i^t\} \rightarrow \mathcal E_\theta^t$\\
4. Reflect: $\mathcal E_\theta^t \rightarrow K^t$}
\end{tcolorbox}
\begin{tcolorbox}[colback=white,colframe=black,boxrule=0.5pt,arc=0pt,boxsep=1pt,left=3pt,right=3pt,top=3pt,bottom=3pt]
\textbf{FunSearch TD}\par
\medskip{\scriptsize
Initialize a top-$k$ knowledge archive $\mathcal E_K^0$.\\
For $t=1,\ldots,T$:\\
1. Propose: $\mathcal E_K^{t-1} \rightarrow \{K^t_i\}$\\
2. Implement: $\{K^t_i\} \rightarrow \{\theta^t_i\}$\\
3. Evaluate: $\{\theta^t_i\} \rightarrow \{\widehat{L}^t_i\}$\\
4. Update: $\mathcal E_K^{t-1},\{K^t_i,\theta^t_i,\widehat{L}_i^t\}\rightarrow \mathcal E_K^t$}
\end{tcolorbox}
  \end{minipage}
\caption{\textbf{Left:} Training trajectories of constructive heuristics on five CO tasks (mean over 5 runs). Lower is better. 
\textbf{Right:} Bottom-up (code-centric) and top-down (knowledge-centric) search pipelines.}

  \label{fig:fig_1}
\end{figure}
We first compare top-down and bottom-up AHD in the simplest implementation to isolate the effect of changing the primary search object. As shown in Figure~\ref{fig:fig_1}, top-down search achieves better best-so-far objectives than bottom-up search in most cases for both \texttt{GPT-5-nano} and \texttt{GPT-4o-mini}. In some settings, especially with \texttt{GPT-5-nano}, its cumulative mean is also higher, suggesting broader exploration and more diverse code behaviors, since objective values reflect the behavior induced by generated programs. Bottom-up search appears more concentrated in those cases, but this narrower exploration does not necessarily yield stronger elites, as reflected by the top-5 cumulative mean.

This advantage with \texttt{GPT-5-nano} also suggests that top-down AHD is not merely ``reasoning before coding'' inside a stronger model. Rather, it provides a distinct search structure in which knowledge is explicitly maintained and evolved across iterations.

\subsection{Cross-Domain Generalization}

\begin{table}[t]
\centering
\caption{Performance gap (\%) relative to the best result for each problem type and size (lower is better).
Cell shading highlights the top-2 methods within each column (darker indicates better rank); the average rank is computed over all 12 columns.
CPT denotes \emph{cross-problem transfer}: TSP ACO uses the TSP constructive heuristic as source, while all other settings use TSP ACO as source.}
\label{tab:tab_1}
\setlength{\tabcolsep}{5pt}
\renewcommand{\arraystretch}{1.2}
\resizebox{\linewidth}{!}{%
\begin{tabular}{|c|l|*{3}{C{1.3cm}|}*{3}{C{1.3cm}|}*{3}{C{1.3cm}|}*{3}{C{1.3cm}|}C{1.2cm}|}
\hline
\multirow{2}{*}{\textbf{Method}} & \multicolumn{1}{c|}{\multirow{2}{*}{\textbf{Variant}}}
  & \multicolumn{3}{c|}{\textbf{TSP}}
  & \multicolumn{3}{c|}{\textbf{CVRP}}
  & \multicolumn{3}{c|}{\textbf{OVRP}}
  & \multicolumn{3}{c|}{\textbf{LVRP}}
  & \textbf{Avg.} \\
\cline{3-14}
 & \multicolumn{1}{c|}{} & 100 & 200 & 500 & 100 & 200 & 500 & 100 & 200 & 500 & 100 & 200 & 500 & \textbf{Rank} \\
\hline
\multirow{2}{*}{---}
  & ACO     & 13.61 & 23.53 & 36.74 & 17.33 & 23.44 & 28.10 & 25.44 & 32.12 & 37.03 & 14.01 & 19.77 & 23.71 & --- \\
\cline{2-15}
  & DeepACO &  0.01 &  0.00 &  0.00 &  0.00 &  0.00 &  0.00 &  2.22 &  0.00 &  0.00 &  0.00 &  0.00 &  0.00 & --- \\
\hline
\hline
\multirow{4}{*}{EA}
  & BU       & 1.02 & 5.17 & 11.33 & 4.94 & 7.66 & 8.82 & 5.64 & 8.80 & 12.11 & 4.35 & 7.02 & 8.41 & 6.33 \\
  & TD       & \cellcolor{rankone}0.00 & \cellcolor{ranktwo}2.47 & \cellcolor{ranktwo}5.93 & 4.10 & \cellcolor{ranktwo}5.56 & \cellcolor{rankone}5.59 & 4.96 & \cellcolor{ranktwo}5.82 & \cellcolor{ranktwo}6.70 & \cellcolor{ranktwo}2.36 & \cellcolor{ranktwo}3.43 & \cellcolor{ranktwo}3.87 & \cellcolor{rankone}2.08 \\
  & BU + CPT & 1.01 & 5.22 & 10.97 & 4.91 & 9.17 & 11.81 & 5.33 & 12.92 & 25.45 & 3.28 & 6.46 & 9.08 & 6.25 \\
  & TD + CPT & 1.46 & 6.03 & 12.50 & 4.75 & 6.79 & 7.53 & 5.61 & 7.52 & 8.46 & \cellcolor{rankone}2.18 & \cellcolor{rankone}3.32 & \cellcolor{rankone}3.75 & 4.67 \\
\hline
\multirow{4}{*}{MCTS}
  & BU       & 0.55 & 2.48 & 6.69 & 3.86 & 6.59 & 8.83 & 5.60 & 7.42 & 8.99 & 4.97 & 7.81 & 9.18 & 5.12 \\
  & TD       & \cellcolor{ranktwo}0.07 & \cellcolor{rankone}2.01 & \cellcolor{rankone}4.87 & \cellcolor{rankone}2.63 & \cellcolor{rankone}4.84 & \cellcolor{ranktwo}6.19 & \cellcolor{ranktwo}4.65 & 6.17 & 6.92 & 4.97 & 7.79 & 9.06 & 3.04 \\
  & BU + CPT & 0.48 & 3.39 & 7.77 & 6.08 & 10.43 & 13.06 & 5.93 & 8.35 & 10.88 & 2.71 & 5.78 & 7.95 & 5.67 \\
  & TD + CPT & 0.35 & 3.03 & 6.78 & \cellcolor{ranktwo}3.28 & 5.83 & 7.15 & \cellcolor{rankone}0.00 & \cellcolor{rankone}1.45 & \cellcolor{rankone}3.57 & 4.25 & 6.21 & 7.18 & \cellcolor{ranktwo}2.83 \\
\hline
\end{tabular}%
}
\end{table}

A useful abstraction should not only guide a single training trajectory, but also remain valuable beyond the trajectory in which it was discovered. This property has received limited empirical attention in prior AHD studies, where intermediate reflections are usually evaluated only through their immediate effect on the ongoing search. We therefore study whether the final artifacts produced by a completed search can be reused to guide new searches on related target domains.

\paragraph{Offline Transfer.}
We reuse the output of a completed source-domain search as an external artifact for the target-domain search. Specifically, bottom-up transfer injects the final code artifact, while top-down transfer injects the final knowledge artifact. In both cases, the artifact is included in the prompt at every step, and the LLM is instructed to use it as inspiration when designing target candidates. Table~\ref{tab:tab_1} considers both cross-solver transfer, from constructive TSP to ACO TSP, and within-solver transfer, from TSP to VRP variants under the same ACO backbone. Transfer does not always improve over training from scratch, but top-down transfer adapts better in most cases. In some settings, such as TSP $\to$ OVRP/LVRP, it even gives the best overall results. We also provide an information-theoretic analysis in Appendix~\ref{app:transfer}.

\paragraph{Online Transfer.}
Table~\ref{tab:tab_2} evaluates online transfer with NCO on multi-distribution TSP, where the LLM searches for heuristic functions that generate bias indicators for the solver. We consider two settings: \emph{white-box}, where the LLM is informed of the distributional structure, and \emph{black-box}, where it only knows that the distribution has shifted and must discover the structure through search. We run single-task search for pure top-down and bottom-up variants, while EoH-S~\cite{liu2026eoh} and the dual TD\&BU variant use a multi-task setting with four times the single-task budget. Here, transfer occurs during search through interaction between the code population, which captures specialized implementations, and the knowledge population, which captures more general abstractions. ReEvo TD\&BU achieves the best overall performance, while pure top-down still consistently outperforms bottom-up.

\begin{table}[t]
\centering
\setlength{\tabcolsep}{5pt}
\renewcommand{\arraystretch}{1.2}
\caption{Performance gap (\%) relative to the LKH3 baseline across TSP distributions and problem sizes (lower is better).
Cell shading highlights the best method within each block.
$^\dagger$ denotes the multi-task setting, where a single method jointly learns four heuristics, one for each distribution.}
\label{tab:tab_2}
\resizebox{\linewidth}{!}{%
\begin{tabular}{|c|l|*{3}{C{1cm}|}*{3}{C{1cm}|}*{3}{C{1cm}|}*{3}{C{1cm}|}C{1.15cm}|}
\hline
\multirow{2}{*}{\textbf{Type}} & \multicolumn{1}{c|}{\multirow{2}{*}{\textbf{Variant}}}
  & \multicolumn{3}{c|}{\textbf{TSP Uniform}}
  & \multicolumn{3}{c|}{\textbf{TSP Clustered}}
  & \multicolumn{3}{c|}{\textbf{TSP Diagonal}}
  & \multicolumn{3}{c|}{\textbf{TSP Barbell}}
  & \textbf{Avg.} \\
\cline{3-14}
 & \multicolumn{1}{c|}{} & 100 & 200 & 500 & 100 & 200 & 500 & 100 & 200 & 500 & 100 & 200 & 500 & \textbf{Rank} \\
\hline
\hline
\multicolumn{15}{|c|}{\textbf{White-Box}} \\
\hline
\multirow{5}{*}{Node} & POMO & 1.85 & 4.65 & 16.26 & 3.49 & 7.99 & 20.34 & 2.64 & 9.56 & 21.16 & 9.15 & 13.65 & 28.18 & 4.96 \\
 & + ReEvo BU & \cellcolor{rank1}\textbf{1.58} & 4.41 & 13.10 & 3.34 & 6.93 & 15.12 & 2.27 & 8.00 & 14.69 & \cellcolor{rank1}7.75 & \cellcolor{rank1}10.25 & 15.82 & 2.75 \\
 & + ReEvo TD & 1.71 & 4.20 & 12.14 & \cellcolor{rank1}3.28 & \cellcolor{rank1}\textbf{6.20} & \cellcolor{rank1}11.18 & 2.35 & \cellcolor{rank1}7.76 & \cellcolor{rank1}13.23 & 7.93 & 10.57 & 16.97 & 2.21 \\
 & + EoH-S$^\dagger$ & 1.85 & 4.12 & 12.70 & 3.48 & 6.80 & 13.11 & 2.29 & 8.09 & 14.74 & 8.13 & 10.32 & 15.61 & 3.04 \\
 & + ReEvo TD \& BU$^\dagger$ & 1.78 & \cellcolor{rank1}\textbf{3.96} & \cellcolor{rank1}11.02 & 3.28 & 6.72 & 14.80 & \cellcolor{rank1}2.26 & 7.97 & 14.48 & 8.28 & 10.46 & \cellcolor{rank1}15.58 & \cellcolor{rank1}2.04 \\
\hline
\multirow{5}{*}{Edge} & DIFUSCO + LS & 7.33 & 9.62 & 14.04 & 4.81 & 8.68 & 15.30 & 4.52 & 10.57 & 15.05 & 8.15 & 12.85 & 18.73 & 4.83 \\
 & + ReEvo BU & 6.18 & 7.49 & 10.83 & 4.60 & 7.18 & 11.20 & 4.57 & \cellcolor{rank1}7.07 & \cellcolor{rank1}8.47 & 7.67 & 9.96 & 12.19 & 2.67 \\
 & + ReEvo TD & 6.63 & 8.57 & 10.19 & 5.10 & 7.83 & 10.93 & \cellcolor{rank1}3.72 & 8.20 & 9.28 & \cellcolor{rank1}7.13 & \cellcolor{rank1}\textbf{9.50} & 12.38 & 2.67 \\
 & + EoH-S$^\dagger$ & 5.57 & 7.46 & 12.83 & 4.74 & 7.58 & 11.68 & 4.40 & 8.39 & 8.64 & 7.32 & 10.52 & 13.04 & 3.00 \\
 & + ReEvo TD \& BU$^\dagger$ & \cellcolor{rank1}5.20 & \cellcolor{rank1}7.27 & \cellcolor{rank1}9.22 & \cellcolor{rank1}4.40 & \cellcolor{rank1}6.61 & \cellcolor{rank1}10.24 & 4.19 & 9.83 & 14.53 & 7.54 & 9.81 & \cellcolor{rank1}11.85 & \cellcolor{rank1}1.83 \\
\hline
\hline
\multicolumn{15}{|c|}{\textbf{Black-Box}} \\
\hline
\multirow{4}{*}{Node} & + ReEvo BU & --- & --- & --- & \cellcolor{rank1}\textbf{3.22} & 6.90 & 16.57 & 2.29 & 8.18 & 15.05 & 7.99 & \cellcolor{rank1}10.40 & 17.09 & 2.94 \\
 & + ReEvo TD & --- & --- & --- & 3.48 & 6.64 & 12.98 & 2.29 & 8.07 & 14.72 & 7.93 & 11.08 & 19.28 & 2.72 \\
 & + EoH-S$^\dagger$ & \cellcolor{rank1}1.86 & \cellcolor{rank1}4.15 & 11.77 & 3.23 & 6.59 & 14.11 & 2.28 & 8.14 & 15.02 & 8.42 & 11.49 & 20.17 & 2.58 \\
 & + ReEvo TD \& BU$^\dagger$ & 1.94 & 4.19 & \cellcolor{rank1}11.67 & 3.41 & \cellcolor{rank1}6.53 & \cellcolor{rank1}12.83 & \cellcolor{rank1}2.24 & \cellcolor{rank1}7.82 & \cellcolor{rank1}13.71 & \cellcolor{rank1}7.89 & 10.71 & \cellcolor{rank1}{16.18} & \cellcolor{rank1}\textbf{1.42} \\
\hline
\multirow{4}{*}{Edge} & + ReEvo BU & --- & --- & --- & 4.54 & 7.56 & 11.60 & 3.33 & 8.58 & 12.58 & 7.57 & 10.05 & 11.79 & 3.00 \\
 & + ReEvo TD & --- & --- & --- & 4.85 & 7.14 & \cellcolor{rank1}\textbf{10.09} & 1.97 & 7.35 & 9.14 & \cellcolor{rank1}\textbf{6.95} & \cellcolor{rank1}9.59 & \cellcolor{rank1}\textbf{10.72} & \cellcolor{rank1}1.89 \\
 & + EoH-S$^\dagger$ & \cellcolor{rank1}5.21 & 7.37 & 8.97 & 4.71 & \cellcolor{rank1}6.84 & 10.41 & 4.91 & 8.75 & 9.26 & 7.06 & 10.98 & 18.02 & 2.58 \\
 & + ReEvo TD \& BU$^\dagger$ & 6.13 & \cellcolor{rank1}6.80 & \cellcolor{rank1}\textbf{8.07} & \cellcolor{rank1}4.40 & 7.11 & 10.24 & \cellcolor{rank1}\textbf{0.27} & \cellcolor{rank1}\textbf{4.54} & \cellcolor{rank1}\textbf{6.49} & 7.78 & 11.99 & 18.26 & 2.00 \\
\hline
\end{tabular}%
}
\end{table}

\begin{figure}[h]
  \centering
  \begin{minipage}[t]{0.58\linewidth}
    \vspace{0pt}
    \centering
    \includegraphics[width=\linewidth]{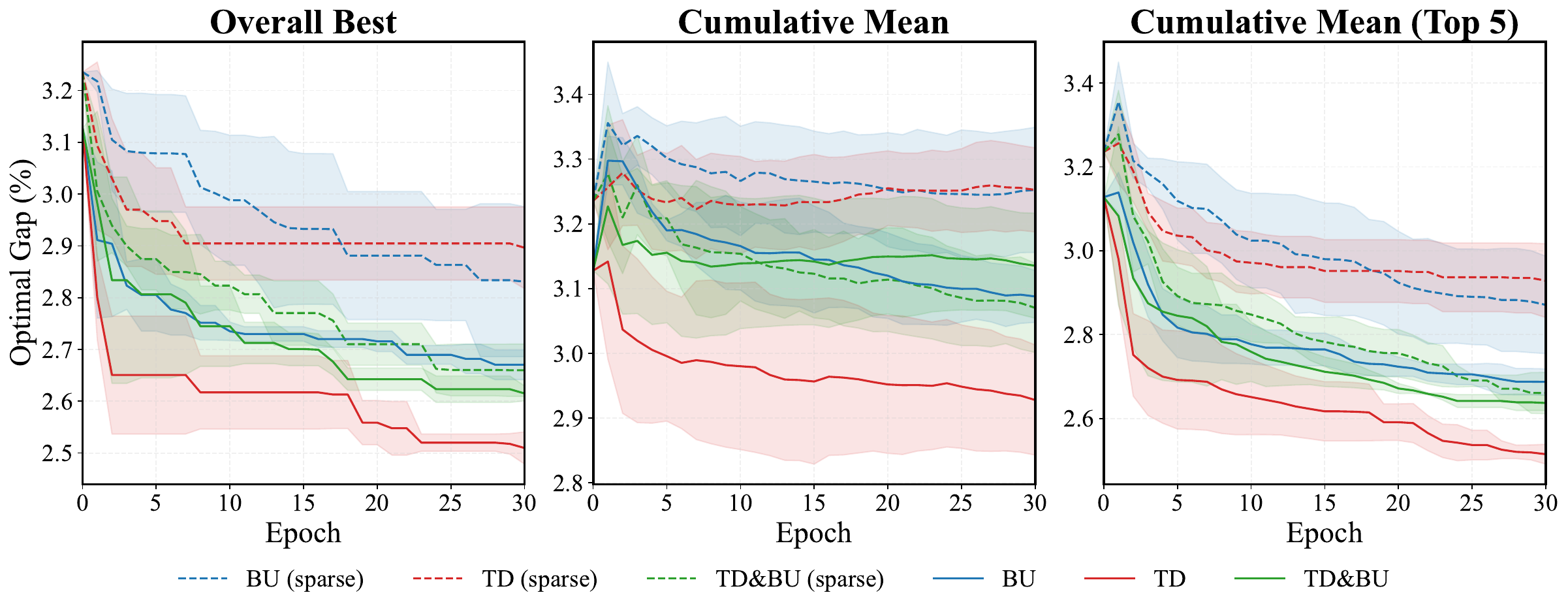}
    \captionof{figure}{Optimality-gap trajectories during training for TD, BU, and TD\&BU on 10 instances TSP200. \emph{Sparse} uses half as many evaluations, while unevaluated code/knowledge candidates still guide the search process.}
    \label{fig:fig_3}
  \end{minipage}\hfill%
  \begin{minipage}[t]{0.40\linewidth}
    \vspace{0pt}
    \centering
    \setlength{\tabcolsep}{5pt}
    \renewcommand{\arraystretch}{1.2}
    \captionof{table}{Test optimality gap (\%) on TSPLib for GLS heuristic functions discovered by different frameworks; lower is better. Instances are grouped by size.}
    \label{tab:tab_3}
    \resizebox{\linewidth}{!}{%
    \begin{tabular}{|l|C{2cm}|C{2cm}|C{2cm}|}
    \hline
    \textbf{Method} & \textbf{$[100,200)$} & \textbf{$[200,500)$} & \textbf{$[500,1000)$} \\
    \hline
    \textbf{Hyperparameters} & 6/25/4 & 15/300/200 & 30/1200/1000 \\
    \hline
    EoH & 2.4255 & 1.7359 & 1.2621 \\
    HSEvo & 2.4431 & 1.6530 & 1.4237 \\
    HiFo & 2.3211 & 1.6229 & 1.3456 \\
    ReEvo BU & 2.2321 & 1.9492 & 1.4294 \\
    ReEvo TD & 2.2443 & 1.4822 & 0.9559 \\
    ReEvo TD \& BU & 2.2908 & 1.6083 & 1.2418 \\
    MCTS-AHD BU & 2.4626 & 1.6780 & 1.3251 \\
    MCTS-AHD TD & 2.2820 & 1.7555 & 1.1134 \\
    \hline
    ReEvo BU (sparse) & 2.4221 & 1.8463 & 1.4093 \\
    ReEvo TD (sparse) & \textbf{2.1714} & \textbf{1.2455} & \textbf{0.8676} \\
    ReEvo TD \& BU (sparse) & 2.3143 & 1.5696 & 1.2926 \\
    \hline
    \end{tabular}%
    }
  \end{minipage}%
\end{figure}

\subsection{Sparse Evaluation}

In AHD, the main cost usually lies not in sampling heuristics from the LLM, but in evaluating generated code on a training set. This cost is especially high when each candidate must be tested on many CO instances, when the solver is slow, or when larger training sets are used to improve generalization. We therefore study a sparse-evaluation setting in which only a subset of generated candidates are empirically evaluated, while the rest are retained as unevaluated artifacts for subsequent search. In our experiments, the unevaluated ratio is $0.5$, halving the number of empirical evaluations relative to full evaluation.

Figure~\ref{fig:fig_3} shows that sparse evaluation weakens training-time progress. Across top-down, bottom-up, and dual variants, training trajectories under sparse evaluation are consistently worse than under full evaluation. This is expected: fewer evaluations provide a weaker optimization signal during search.

However, Table~\ref{tab:tab_3} shows that ReEvo TD under sparse evaluation achieves the best test performance among contemporary AHD methods despite using only half of the evaluation budget. This suggests that sparse evaluation does more than save compute: it may also change the statistical character of search by reducing overfitting to the finite training set while still exploiting structural information from unevaluated artifacts. We analyze this effect in more detail in Appendix~\ref{app:sparse_theory}.

This observation is consistent with the Bayesian view in Equation~\ref{eq:bayes}. Under sparse evaluation, an unevaluated candidate lacks the empirical feedback $\widehat{L}^t$, so the refinement over knowledge drops the evaluative likelihood and becomes
\begin{equation}
q_t(K\mid A^t,C)
\propto
q_t(A^t\mid K,C)\,q_t(K\mid C).
\label{eq:sparse_bayes}
\end{equation}
Thus, sparse evaluation does not discard unevaluated candidates; it changes their role from performance-based evidence to structure-only evidence. This particularly favors top-down search, whose primary object is the latent knowledge state $K$: unevaluated artifacts can still support or refine design principles through $q_t(A^t\mid K,C)$. In contrast, bottom-up refinement relies more directly on empirical feedback to explain why a heuristic is good or bad. This explains why ReEvo TD can remain effective, and even generalize better, with only sparse candidate evaluation.

\begin{table}[t]
\centering
\setlength{\tabcolsep}{5pt}
\renewcommand{\arraystretch}{1.2}
\caption{Results on tasks beyond CO, including SCO, SR, and PE. We report the performance gap (\%) for SCO, NMSE for SR, and the optimal gap (\%) for PE; lower is better for all metrics.}
\label{tab:sco_sr_bo}
\resizebox{\linewidth}{!}{%
\begin{tabular}{|l|C{1.6cm}|C{1.6cm}|C{1.6cm}|C{1.6cm}|C{1.6cm}|C{1.6cm}|C{1.6cm}|C{1.6cm}|C{1.2cm}|}
\hline
\multicolumn{1}{|c|}{\multirow{2}{*}{\textbf{Method}}}
  & \multicolumn{2}{c|}{\small\textbf{SCO1: CTS}}
  & \multicolumn{2}{c|}{\small\textbf{SCO2: FTR}}
  & \multicolumn{2}{c|}{\small\textbf{SCO3: OAS}}
  & \multicolumn{2}{c|}{\small\textbf{SCO4: WPF}}
  & \textbf{Avg.} \\
\cline{2-9}
  & Train & Test & Train & Test & Train & Test & Train & Test & \textbf{Rank} \\
\hline
Baseline    & 15.9636 & 16.6758 & 31.7629 & 33.0416 & 66.7779 & 55.9252 & 5.9036 & 5.5159 & 7.00 \\
Beam Search &  0.2082 &  0.0180 &  0.1089 &  \cellcolor{rank1}0.0000 &  1.9681 &  1.7258 & 5.4104 & 4.4148 & 4.38 \\
ReEvo BU    &  0.0544 &  0.0450 &  0.5398 &  0.7253 &  0.4073 &  0.5485 & 1.9906 & 1.5379 & 4.25 \\
ReEvo TD    &  0.0118 &  0.3019 &  0.1570 &  0.4041 &  0.1421 &  0.2594 & 1.6605 & 1.1447 & 3.12 \\
MCTS-AHD BU &  0.1035 &  0.3263 &  0.9195 &  0.8330 &  0.4243 &  0.6077 & 1.5004 & 1.1579 & 5.00 \\
MCTS-AHD TD &  \cellcolor{rank1}0.0000 &  0.0543 &  0.2621 &  0.2941 &  0.2990 &  0.3033 & \cellcolor{rank1}0.0000 & 0.5096 & 2.62 \\
ReEvo TD \& BU  &  0.0927 &  \cellcolor{rank1}0.0000 &  \cellcolor{rank1}0.0000 &  0.2279 &  \cellcolor{rank1}0.0000 &  \cellcolor{rank1}0.0000 & 0.2417 & \cellcolor{rank1}0.0000 & \cellcolor{rank1}1.62 \\
\hline
\end{tabular}%
}
\resizebox{\linewidth}{!}{%
\begin{tabular}{|l|C{1.6cm}|C{1.6cm}|C{1.6cm}|C{1.6cm}|C{1.6cm}|C{1.6cm}|C{1.6cm}|C{1.6cm}|C{1.2cm}|}
\hline
\multicolumn{1}{|c|}{\multirow{2}{*}{\textbf{Method}}}
  & \multicolumn{2}{c|}{\small\textbf{SR1: oscillation1}}
  & \multicolumn{2}{c|}{\small\textbf{SR2: oscillation2}}
  & \multicolumn{2}{c|}{\small\textbf{SR3: bactgrow}}
  & \multicolumn{2}{c|}{\small\textbf{SR4: stressstrain}}
  & \textbf{Avg.} \\
\cline{2-9}
  & ID & OOD & ID & OOD & ID & OOD & ID & OOD & \textbf{Rank} \\
\hline
LLM-SR    & 5.8289e-6 & 6.4234e-4 & 6.1132e-2 & 1.4298e-5 & 0.3321 & 0.7854 & 2.5635e-2 & 0.1324 & 4.38 \\
ReEvo BU    & 6.8259e-6 & 8.1374e-4 & 6.0956e-2  & 1.6365e-2  & 0.5140 & 1.0234 & 3.5795e-2 & 0.1031 & 5.00 \\
ReEvo TD    & 2.4107e-6 & 5.0879e-4 & 4.4905e-7  & 7.7363e-6  & 0.2960 & 0.6860 & 2.4853e-2 & \cellcolor{rankone}0.0024 & 2.75 \\
MCTS-AHD BU & 2.3592e-5 & 6.8650e-5 & 1.1300e-9  & 5.7365e-9  & 0.6311 & 0.8205 & 5.4489e-2 & 0.6145 & 4.50 \\
MCTS-AHD TD & 1.3280e-6 & \cellcolor{rankone}2.8437e-5 & 5.1693e-8  & 1.5996e-10 & 0.5211 & 0.7710 & 2.5151e-2 & 0.0853 & 2.63 \\
ReEvo TD \& BU  & \cellcolor{rankone}2.1078e-7 & 1.4555e-4 & \cellcolor{rankone}4.1422e-12 & \cellcolor{rankone}7.6179e-11 & \cellcolor{rankone}0.1024 & \cellcolor{rankone}0.3227 & \cellcolor{rankone}2.3892e-2 & 0.3227 & \cellcolor{rankone}1.75 \\
\hline
\end{tabular}%
}

\resizebox{\linewidth}{!}{%
\begin{tabular}{|l|C{1.6cm}|C{1.6cm}|C{1.6cm}|C{1.6cm}|C{1.6cm}|C{1.6cm}|C{1.6cm}|C{1.6cm}|C{1.2cm}|}
\hline
\multicolumn{1}{|c|}{\multirow{2}{*}{\textbf{Method}}}
  & \multicolumn{2}{c|}{\small\textbf{PE1: Alpha Amylase}}
  & \multicolumn{2}{c|}{\small\textbf{PE2: Hydrophobic Core}}
  & \multicolumn{2}{c|}{\small\textbf{PE3: Imine Reductase}}
  & \multicolumn{2}{c|}{\small\textbf{PE4: Rhodopsin}}
  & \textbf{Avg.} \\
\cline{2-9}
  & Train & Test & Train & Test & Train & Test & Train & Test & \textbf{Rank} \\
\hline
EI & 3.6571 & 6.9354 & 1.1914 & 2.6524 & 14.5710 & 30.8325 & 5.4527 & 12.1314 & 5.62 \\
UCB & 3.6033 & 3.9545 & 1.0315 & 7.3742 & 13.1581 & 43.0866 & 3.7037 & 14.4903 & 5.75 \\
ReEvo BU    & 1.3097 & 5.5851 & \cellcolor{rank1}0.0000 & 3.3189 & 8.1131 & 31.0003 & \cellcolor{rank1}1.0288 & 17.1862 & 4.00 \\
ReEvo TD    & \cellcolor{rank1}0.6606 & \cellcolor{rank1}2.2575 & \cellcolor{rank1}0.0000 & 0.1021 & 7.8712 & 28.1180 & 2.4691 & 16.7460 & \cellcolor{rank1}2.50 \\
MCTS-AHD BU & 2.2834 & 7.1935 & \cellcolor{rank1}0.0000 & 1.2213 & 9.9525 & 33.7088 & 3.2922 & 11.2890 & 4.50 \\
MCTS-AHD TD & 3.1039 & 3.4084 & \cellcolor{rank1}0.0000 & \cellcolor{rank1}0.0000 & 9.6215 & 32.4178 & 2.5620 & \cellcolor{rank1}10.4465 & 3.06 \\
ReEvo TD \& BU  & 1.2365 & 5.5269 & \cellcolor{rank1}0.0000 & \cellcolor{rank1}0.0000 & \cellcolor{rank1}5.6994 & \cellcolor{rank1}22.7130 & 2.7778 & 13.7321 & 2.56 \\
\hline
\end{tabular}%
}
\end{table}

\subsection{Evaluation on Tasks Beyond CO}
Beyond CO benchmarks, we test whether top-down search remains useful in distinct settings, with task definitions and solver backbones in Appendices~\ref{app:sco},~\ref{app:sr}, and~\ref{app:pe}. We consider four synthetic sequential combinatorial optimization (SCO) problems, which limit familiar LLM priors; four symbolic regression (SR) datasets, where scientific understanding can guide equation discovery; and four protein engineering (PE) datasets, where the goal is to design acquisition functions for Bayesian optimization over protein candidates in a structured black-box domain.

Table~\ref{tab:sco_sr_bo} shows that the advantage of top-down search is not restricted to conventional CO. On SCO, top-down variants generally outperform bottom-up ones, suggesting that explicit knowledge states help even when the task is synthetic and less familiar to the LLM. On SR and PE, where domain-level hypotheses are naturally meaningful, searching over interpretable knowledge can guide the generation of better executable formulas or acquisition functions. Across these settings, the dual code--knowledge variant often achieves the best average rank, indicating that bottom-up implementation feedback and top-down hypothesis evolution provide complementary signals.


\bibliographystyle{unsrtnat}
\bibliography{refs}


\newpage

\appendix

\begin{center}
\vspace*{0.25em}
\rule{\linewidth}{1pt}\\[0.5em]
{\LARGE\bfseries Appendix}\\[0.35em]
{\large\textit{Back to the Beginning of Heuristic Design:}\\[0.2em]Bridging Code and Knowledge with LLMs}
\rule{\linewidth}{1pt}
\end{center}

\appendixcontents
\vspace{0.75em}

\rule{\linewidth}{1pt}

\appsection{Q\&A}
\label{app:qa}

\begin{bbox}
What is the key novelty of this work?
\end{bbox}
The key novelty of this work is a \emph{knowledge-first view} of LLM-driven AHD. Instead of treating knowledge as auxiliary context distilled after code evaluation, we make it the primary search object: the framework proposes, evolves, transfers, and refines reusable design principles, while code serves as the executable realization used to test them. This shift opens a new conceptual direction for AHD and is supported by an \emph{information-theoretic perspective}, which explains the trade-off between the compression benefit of searching over knowledge and the distortion cost of abstracting away implementation details. Empirically, the paper shows that this view is useful not only in standard AHD comparisons, but also in less explored settings such as cross-domain transfer, dual code--knowledge evolution, tasks beyond CO, and especially \emph{sparse evaluation}, where unevaluated artifacts can still contribute structural evidence for knowledge refinement. Thus, the novelty lies in making knowledge a first-class search object and opening up new AHD settings centered on transfer, sparse evaluation, and generalization beyond standard CO benchmarks.

\begin{bbox}
Is there additional computational overhead compared to bottom-up methods?
\end{bbox}
No. The top-down variants do not introduce additional overhead compared to their bottom-up counterparts under our protocol. As detailed in Appendix~\ref{app:setup}, we keep the \emph{evaluation budget}, the number of \emph{LLM calls}, and the main implementation settings matched across paradigms, so the comparison isolates the effect of changing the search object rather than increasing computation.

\begin{bbox}
What is the main limitation of treating knowledge as a first-class search object?
\end{bbox}
The main limitation is the \emph{fidelity of abstraction}. Treating knowledge as a first-class search object is useful only when the knowledge state preserves the performance-relevant structure of the heuristic. If the abstraction is too coarse, it may discard low-level implementation details crucial for a specific solver, problem class, or evaluation regime; then different code realizations of the same knowledge can behave quite differently. In this case, the benefit of searching in a more compact knowledge space may be outweighed by the \emph{distortion} introduced by abstraction, so bottom-up code-level search can be preferable. We provide empirical \emph{failure cases} of top-down search in Appendix~\ref{app:failure}.

\begin{bbox}
What does top-down enable that bottom-up fundamentally cannot?
\end{bbox}
Top-down enables \emph{direct search over reusable principles} before committing to a particular implementation. Because knowledge states are generated and maintained as first-class objects, they can be selected, recombined, transferred, and refined even when their current code realizations are imperfect or only partially evaluated. Bottom-up search, in its native form, can only abstract knowledge after code has been generated and evaluated, so its knowledge is tied to the explored code trajectory and to the availability of empirical feedback. This makes top-down especially useful for transfer and sparse evaluation: a knowledge state can preserve problem-level or solver-level invariants across domains, and unevaluated artifacts can still contribute structural evidence for refining future hypotheses. If a bottom-up method is modified to explicitly maintain and evolve such knowledge states, it effectively moves toward the top-down or dual code--knowledge formulation.

\begin{bbox}
How does the method compare in terms of interpretability?
\end{bbox}
The top-down AHD is more interpretable because the search process exposes the intermediate \emph{knowledge states} that guide heuristic generation. In bottom-up AHD, interpretability usually comes after the fact: one inspects successful code or reads reflections distilled from evaluated programs. In top-down AHD, the rationale is explicit before implementation, so each candidate can be examined as a design hypothesis and then linked to the code and empirical performance it produces. This makes it easier to understand why a heuristic was proposed, how knowledge evolves across iterations, and which principles transfer across tasks. However, interpretability is not automatic: the generated knowledge may still be incomplete, overly abstract, or imperfectly realized in code, so the final behavior should be interpreted through both the knowledge state and its executable instantiation.

\begin{bbox}
Does the approach rely on the LLM implicitly understanding the task domain?
\end{bbox}
Yes, to some extent. Top-down search benefits from the LLM's \emph{domain understanding}, because the first object it proposes is a knowledge state rather than executable code; if this prior knowledge is accurate, it can guide the search toward reusable design principles, but if the task description is misleading or the LLM's prior is wrong, the resulting knowledge can bias the search in an unfavorable direction. This is why top-down is not expected to dominate in all settings: its advantage depends partly on whether the LLM can form meaningful abstractions for the target domain. The model-dependent behavior in Figure~\ref{fig:fig_1}, where a stronger model such as \texttt{GPT-5-nano} makes the top-down advantage clearer in the FunSearch setting, is consistent with this interpretation. We further study cases where the problem description induces incorrect or unhelpful knowledge in Appendix~\ref{app:failure}.

\begin{bbox}
Why is sparse evaluation important for AHD?
\end{bbox}
Sparse evaluation is important because \emph{empirical evaluation is the main scalable cost} in AHD. Once the LLM provider, prompting protocol, and number of generated candidates are fixed, the latency and cost of LLM calls are largely fixed by the experimental budget; similarly, the overhead of parsing, compiling, or executing the generated program interface is usually secondary. The cost that grows substantially is evaluating each candidate on training instances, especially when the downstream solver is slow or when a larger training set is used to improve generalization. Parallel and distributed execution can reduce wall-clock time, and we use such acceleration in our experiments, but it does not remove the underlying evaluation budget: it still requires more workers, more total compute, and more solver executions. Sparse evaluation therefore studies a practically important regime where AHD must learn from fewer evaluated candidates while still exploiting unevaluated artifacts as structural evidence. This setting is especially relevant for top-down search, because knowledge states can continue to be refined even when only part of the generated code is empirically evaluated.

\begin{bbox}
Are the improvements over bottom-up AHD only marginal?
\end{bbox}
No. The comparison is designed to be \emph{controlled}: top-down search changes the primary search object within the same AHD backbone, such as ReEvo BU vs. ReEvo TD and MCTS-AHD BU vs. MCTS-AHD TD, while keeping the evaluation budget, LLM calls, and implementation protocol matched. Thus, the gains should be read not merely as raw numerical improvements, but as evidence that changing from code-centric to knowledge-centric search can improve the same underlying framework. In several settings, including Tables~\ref{tab:tab_1},~\ref{tab:tab_2}, and~\ref{tab:tab_3}, the improvements are visibly non-marginal; and in mature CO benchmarks, even moderate reductions in optimality gap are meaningful because progress becomes increasingly difficult near strong baselines.

\begin{center}
Additional questions may be incorporated into this Q\&A in future revisions.
\end{center}

\newpage

\appsection{Related Works}
\label{app:related_works}

\appsubsection{LLMs for CO}

\paragraph{Combinatorial Optimization (CO).}
CO studies optimization problems whose feasible solutions are discrete structures such as subsets, permutations, assignments, routes, schedules, or graphs. Classical references such as \citet{papadimitriou1998combinatorial} and \citet{wolsey1999integer} formalize CO through algorithmic, polyhedral, and complexity-theoretic perspectives, while \citet{Karp1972} establishes the central role of NP-completeness for many canonical CO problems. Because exact optimization can be computationally prohibitive at scale, practical CO has long relied on approximation algorithms, local search, metaheuristics, and problem-specific heuristics to obtain high-quality solutions under limited computational budgets.

Machine learning has provided a complementary route for constructing CO solvers from data. Pointer Networks~\citep{NIPS2015_29921001} introduced an architecture for predicting output sequences whose elements point to positions in the input, making it naturally applicable to routing and ordering problems. Neural Combinatorial Optimization~\citep{bello*2017neural} trains neural policies with reinforcement learning to construct solutions for problems such as TSP and knapsack. The graph-embedding framework of \citet{khalil2017learning} learns greedy policies for graph optimization problems such as minimum vertex cover, maximum cut, and TSP, while the attention model of \citet{kool2018attention} substantially improves neural construction policies for routing problems. These methods replace hand-engineered decision rules with learned policies, but they typically require task-specific training, architectures, or distributions.

\paragraph{Automatic Heuristic Design (AHD).}
AHD aims to reduce the manual effort required to craft effective heuristics for hard optimization problems. Hyper-heuristics, surveyed by \citet{burke2013hyper}, search over heuristics or heuristic components rather than directly over solutions, thereby raising the level of search from instance-level decisions to algorithm-level design. Related automated algorithm configuration methods such as ParamILS~\citep{hutter2009paramils} and SMAC~\citep{10.1007/978-3-642-25566-3_40} optimize the parameters of existing solvers over a target distribution of problem instances. These lines of work share a common motivation: strong heuristic performance often depends on expert-designed algorithmic choices, and automating those choices can improve robustness, transferability, and development efficiency.

Recent LLM-based approaches extend automatic heuristic design by using language models to generate, modify, and evaluate executable algorithmic components. FunSearch~\citep{romera2024mathematical} pairs a pretrained LLM with an automated evaluator to evolve programs and obtain new constructions for mathematical and algorithmic problems. Evolution of Heuristics~\citep{liu2024evolution} uses LLMs and evolutionary computation to search over both natural-language heuristic ideas and executable code, applying the framework to CO benchmarks. ReEvo~\citep{ye2024reevo} formulates LLM-based heuristic generation as language hyper-heuristics and introduces reflective evolution to guide the search over heuristic space. These works show that LLMs can act not only as code generators, but also as operators for algorithmic variation, recombination, and refinement.

\appsubsection{LLMs for BO}

\paragraph{Bayesian Optimization (BO).}
BO is a sample-efficient framework for optimizing expensive black-box objectives, where gradients are unavailable and each evaluation may require costly simulation, training, or real-world experimentation. A standard BO loop maintains a probabilistic surrogate model, often a Gaussian process, over the unknown objective and selects new candidates by maximizing an acquisition function that trades off exploration and exploitation. Classical BO builds on Bayesian global optimization by \citet{Mockus1978} and efficient global optimization by \citet{JonesSW98}, while GP-UCB by \citet{srinivas2010gaussian} provides regret guarantees and the work of \citet{snoek2012practical} established BO as a practical tool for machine learning hyperparameter tuning. A broader review by \citet{7352306} summarizes BO as a general paradigm for sequential model-based optimization under limited evaluation budgets.

\paragraph{Language-Augmented Bayesian Optimization.}
Recent work has begun to combine BO with large language models, either by using LLMs to improve BO components or by using BO to optimize LLM-facing artifacts. LLAMBO~\citep{liu2024large} frames BO problems in natural language and uses LLMs for warm-starting, surrogate modeling, and candidate generation, showing benefits in hyperparameter optimization when observations are sparse. InstructZero~\citep{pmlr-v235-chen24e} uses BO to optimize a low-dimensional soft prompt that is decoded by an open-source LLM into natural-language instructions for a black-box LLM. In this formulation, BO does not directly search over discrete instructions, but instead searches over a continuous latent space whose points induce candidate prompts.

More recent methods further adapt BO to LLM-driven search and prompt selection. BOPRO~\citep{ICLR2025_a79237d6} uses Bayesian optimization to guide LLM generation by selecting previous generations for exploration or exploitation, applying the approach to word search, molecule optimization, and program-like hypothesis search. HbBoPs~\citep{schneider2025hyperbandbased} combines a structural-aware Gaussian-process surrogate with Hyperband-style multi-fidelity evaluation for black-box prompt selection, modeling instructions and few-shot exemplars as separate prompt components. Together, these works show that BO can interact with LLMs at several levels: as a controller for prompt search, as a surrogate-assisted mechanism for LLM-generated candidates, or as a model-based optimizer whose components are augmented by language models.

\appsubsection{LLMs for AD}

\paragraph{Algorithm Design (AD).} AD is the broader practice of creating, refining, or discovering procedures that solve a target task efficiently, including numerical solvers, reinforcement-learning reward functions, multi-stage LLM pipelines, and mathematical constructions. Whereas AHD focuses on heuristic components inside an existing solver, AD treats the whole algorithm, including control flow, subroutines, scoring rules, and even the evaluation metric, as the search target. Traditionally, AD has relied on human expertise, theoretical analysis, and trial-and-error experimentation, with only limited automation through methods such as algorithm configuration and operator selection. With large language models, recent work instead casts AD as a search problem over textual artifacts such as programs, prompts, or natural-language descriptions, where LLMs act as both proposal mechanisms and reasoning engines while automated evaluators provide feedback. This perspective, discussed by \citet{wu2024evolutionary}, \citet{surina2025algorithm}, and \citet{liu2026systematic}, unifies program synthesis, evolutionary code search, and prompt optimization under a common view of the LLM as a general-purpose algorithm designer.

\paragraph{LLM-Driven Algorithm and Prompt Discovery.} A representative line of work uses LLMs to propose, mutate, and refine programs against an automatic evaluator. AlphaEvolve~\citep{novikov2025alphaevolve} applies this idea to full algorithms, with results on matrix multiplication and data-center scheduling. Eureka~\citep{maeureka} evolves dense reward functions for robotics, LLaMEA~\citep{van2024llamea} generates black-box metaheuristics from scratch, and DeepEvolve~\citep{liu2025scientific} adds deep research and external knowledge retrieval. Similar evaluator-guided search also appears in LLM pipelines and prompts: DSPy~\citep{khattab2023dspy} compiles compound LLM systems, MIPROv2~\citep{opsahl2024optimizing} combines bootstrapped traces with Bayesian-style prompt search, and GEPA~\citep{agrawal2025gepa} uses reflective prompt evolution, outperforming reinforcement-learning-based adaptation while using up to 35-fold fewer rollouts. Together, these works show that LLM-driven, evaluator-guided search supports algorithm design at several levels, from executable programs to scientific algorithms and LLM systems themselves.

\appsubsection{LLMs for SD}


\paragraph{Symbolic Regression.} Similar to algorithm design, LLMs have been widely explored as a promising direction for symbolic regression (SR) \cite{shojaeellm, grayeli2024symbolic, guo2026coevo, xia2025srscientist, wang2025drsr}. LLM-SR \cite{shojaeellm} leverages LLMs to generate symbolic programs by integrating scientific prior knowledge with a multi-island evolutionary search process driven by data-based feedback, while refining the program parameters through numerical optimization. CoEvo \cite{guo2026coevo} enables open-ended discovery through the introduction of an external knowledge library. Meanwhile, LaSR \cite{grayeli2024symbolic} proposes a concept-learning approach for equation discovery, where abstract textual concepts extracted by LLMs are combined with evolutionary search (via PySR \cite{cranmer2023interpretable}) and LLM-guided exploration to evolve new hypotheses. DrSR \cite{wang2025drsr} introduces a closed-loop reasoning framework that improves SR performance through data-driven insight and reflective reasoning. SR-Scientist \cite{xia2025srscientist} proposes an agentic framework that incorporates tool-calling capabilities, such as data analysis modules, enabling LLMs to write and execute code for dataset analysis.

\paragraph{Scientific Discovery.} LLMs are opening up new directions in scientific discovery (SD). They have been applied to challenging mathematical problems \cite{trinh2024solving}, mathematical proof assistance \cite{collins2024evaluating}, scientific literature retrieval \cite{ajith2024litsearch} and code generation for analytical and computational tasks \cite{huang2024mlagentbench, tian2024scicode}. Recent work further shows that LLMs can propose novel research ideas, some of which are judged comparable to or more novel than ideas from human experts \cite{wang2024scimon, sican, baek2025researchagent}, although follow-up studies find that stronger ideation does not always translate into stronger empirical outcomes during execution \cite{si2025ideation}. Beyond ideation, LLMs have been used to accelerate chemistry research \cite{jablonka2024leveraging} and, when paired with evolutionary search, to discover new mathematical programs and algorithms \cite{romera2024mathematical}.


\appsection{Theoretical Foundations}
\label{app:theory}

\appsubsection{Bayesian Interpretation of Equation~\ref{eq:bayes}}
\label{app:equation_6}

Equation~\ref{eq:bayes} formalizes the relationship between the top-down and bottom-up views of AHD as two complementary directions of the same latent-variable process. At round $t$, let $C:=H^{t-1}$ denote the current search context. The top-down view treats the knowledge state $K^t$ as the primary latent variable, from which a heuristic $\theta^t$ is synthesized and then evaluated. In particular, the generative chain is
\[
K^t \longrightarrow \theta^t \longrightarrow (A^t,\widehat{L}^t),
\]
where $A^t=\alpha(\theta^t)$ denotes the heuristic artifact observed by the reflector, such as source code, a trace, or a summary, and $\widehat{L}^t=\widehat{L}_{\mathcal D}(\theta^t)$ is the empirical feedback obtained by evaluating $\theta^t$ on $\mathcal D$.

Conditional on $C$, the top-down mechanism specifies a prior $q_t(K\mid C)$ over knowledge states and a synthesis law $q_t(\theta\mid K,C)$ over heuristics. To model what the reflector actually observes, let $q_t(A,\widehat{L}\mid \theta,K,C)$ be the observation kernel induced by the heuristic artifact and its empirical evaluation. Marginalizing out $\theta$ yields the knowledge-conditioned observation model
\[
q_t(A,\widehat{L}\mid K,C) := \int_{\Theta} q_t(A,\widehat{L}\mid \theta,K,C)\,q_t(\theta\mid K,C)\,d\theta.
\]
The resulting joint law of $(K^t,A^t,\widehat{L}^t)$ conditional on $C$ is therefore
\[
q_t(K,A,\widehat{L}\mid C) = q_t(A,\widehat{L}\mid K,C)\,q_t(K\mid C).
\]

The bottom-up view proceeds in the reverse direction. Once the reflector observes $(A^t,\widehat{L}^t)$, it updates its belief over latent knowledge states according to Bayes' rule:
\[
q_t(K\mid A^t,\widehat{L}^t,C)
=
\frac{q_t(A^t,\widehat{L}^t\mid K,C)\,q_t(K\mid C)}{q_t(A^t,\widehat{L}^t\mid C)},
\]
where the normalizing constant is
\[
q_t(A^t,\widehat{L}^t\mid C)
=
\int_{\mathcal K} q_t(A^t,\widehat{L}^t\mid K,C)\,q_t(K\mid C)\,dK.
\]
Using the factorization $q_t(A^t,\widehat{L}^t\mid K,C)=q_t(\widehat{L}^t\mid A^t,K,C)\,q_t(A^t\mid K,C)$ gives Equation~\ref{eq:bayes}:
\[
q_t(K\mid A^t,\widehat{L}^t,C)
\propto
\underbrace{q_t(\widehat{L}^t\mid A^t,K,C)}_{\text{evaluative likelihood}}
\underbrace{q_t(A^t\mid K,C)}_{\text{structural likelihood}}
\underbrace{q_t(K\mid C)}_{\text{prior over knowledge}}.
\]

Each factor in Equation~\ref{eq:bayes} has a distinct meaning. The prior $q_t(K\mid C)$ captures which design principles are plausible before inspecting the current candidate. The structural likelihood $q_t(A^t\mid K,C)$ measures how well the observed artifact matches what one would expect from knowledge state $K$. The evaluative likelihood $q_t(\widehat{L}^t\mid A^t,K,C)$ measures how well that same knowledge state predicts the observed quality of the heuristic, after accounting for the artifact itself. Thus, $A^t$ answers the question of what the heuristic appears to be doing, while $\widehat{L}^t$ answers the question of how well that behavior performs on $\mathcal D$.

This distinction is important. If the reflector had direct access to the full extensional object $\theta^t$, then $\widehat{L}^t=\widehat{L}_{\mathcal D}(\theta^t)$ would be deterministic given $\theta^t$ and the evaluative term would become redundant within the current round. However, the reflector does not reason over $\theta^t$ in that strong mathematical sense; it reasons over the coarsened artifact $A^t$. Under this more realistic observation model, $\widehat{L}^t$ contributes genuine posterior information by disambiguating knowledge states that may explain the artifact similarly well but imply different expected performance.

This can be seen explicitly from the posterior odds ratio. For any two knowledge states $K_1,K_2\in\mathcal K$,
\[
\frac{q_t(K_1\mid A^t,\widehat{L}^t,C)}{q_t(K_2\mid A^t,\widehat{L}^t,C)}
=
\frac{q_t(\widehat{L}^t\mid A^t,K_1,C)}{q_t(\widehat{L}^t\mid A^t,K_2,C)}
\cdot
\frac{q_t(A^t\mid K_1,C)}{q_t(A^t\mid K_2,C)}
\cdot
\frac{q_t(K_1\mid C)}{q_t(K_2\mid C)}.
\]
Hence, when two knowledge states are similarly compatible with the observed artifact, the empirical feedback $\widehat{L}^t$ becomes the decisive Bayes factor. In this sense, bottom-up refinement should be understood not merely as heuristic ranking across rounds, but as within-round posterior updating over latent design principles.

Finally, Equation~\ref{eq:bayes} should be viewed as an idealized Bayesian target. In practice, a bottom-up LLM reflector does not compute the exact posterior; rather, it approximates this update through prompting and sampling. Nonetheless, the equation clarifies the conceptual role of the reflector: it infers which latent design principles are most consistent with both the observed structure of the heuristic and its measured empirical quality.

\appsubsection{Proof and Discussion of Proposition~\ref{eq:prop_1}}
\label{app:proof_prop_1}

Before proving Proposition~\ref{eq:prop_1}, we briefly justify the bounded-loss assumption. Suppose there exist constants $a_d<b_d$ such that
\[
a_d \le \ell_d(\theta;\mathbf{x}) \le b_d
\qquad
\text{for all }\theta\in\Theta \text{ and all }\mathbf{x}\in\operatorname{supp}(\mathcal P).
\]
Then one may normalize the per-instance loss by
\[
\bar{\ell}_d(\theta;\mathbf{x})
:=
\frac{\ell_d(\theta;\mathbf{x})-a_d}{b_d-a_d},
\]
so that $0\le \bar{\ell}_d(\theta;\mathbf{x})\le 1$. All risks, distortions, and excess-risk bounds are then rescaled by the constant factor $b_d-a_d$. Since our goal is to compare the statistical trade-off between top-down and bottom-up search rather than preserve the original units of the objective, assuming $0\le \ell_d(\theta;\mathbf{x})\le 1$ is without loss of generality up to a fixed change of scale. For fixed-size CO benchmarks with bounded instance support, such a uniform bound is natural.

We now prove Proposition~\ref{eq:prop_1}. The proof relies on a standard information-theoretic generalization bound specialized to bounded losses.

\paragraph{Lemma.}
Let $\mathcal D=(\mathbf{x}_1,\dots,\mathbf{x}_n)\sim\mathcal P^n$, and let $A=A(\mathcal D)$ be any data-dependent random variable taking values in $\Theta$. Assume that $0\le \ell_d(\theta;\mathbf{x})\le 1$ almost surely for every $\theta\in\Theta$ and $\mathbf{x}\sim\mathcal P$. Then
\[
\left|
\mathbb E\bigl[L_{\mathcal P}(A)-\widehat{L}_{\mathcal D}(A)\bigr]
\right|
\le
\sqrt{\frac{\mathbb{I}(A;\mathcal D)}{2n}}. 
\]

\begin{proof}
For a fixed $\theta\in\Theta$, define
\[
G(\mathcal D,\theta)
:=
L_{\mathcal P}(\theta)-\widehat{L}_{\mathcal D}(\theta)
=
\frac{1}{n}\sum_{i=1}^n Z_i(\theta),
\]
where
\[
Z_i(\theta)
:=
L_{\mathcal P}(\theta)-\ell_d(\theta;\mathbf{x}_i).
\]
Since $0\le \ell_d(\theta;\mathbf{x}_i)\le 1$, we have $-1\le Z_i(\theta)\le 1$, and by definition of $L_{\mathcal P}(\theta)$,
\[
\mathbb E[Z_i(\theta)]
=
L_{\mathcal P}(\theta)-\mathbb E[\ell_d(\theta;\mathbf{x}_i)]
=
0.
\]
Hoeffding's lemma therefore implies that, for every $\lambda\in\mathbb R$,
\[
\mathbb E\bigl[e^{\lambda Z_i(\theta)}\bigr]\le e^{\lambda^2/8}.
\]
Using independence of $\mathbf{x}_1,\dots,\mathbf{x}_n$, we obtain
\[
\mathbb E\bigl[e^{\lambda G(\mathcal D,\theta)}\bigr]
=
\prod_{i=1}^n \mathbb E\bigl[e^{(\lambda/n)Z_i(\theta)}\bigr]
\le
\prod_{i=1}^n e^{\lambda^2/(8n^2)}
=
e^{\lambda^2/(8n)}.
\]
Thus, for each fixed $\theta$, the centered generalization gap $G(\mathcal D,\theta)$ is $1/(4n)$-sub-Gaussian with respect to the randomness of $\mathcal D$.

Now let $P$ denote the joint law of $(A,\mathcal D)$ and let $Q:=P_A\otimes P_{\mathcal D}$ be the product of the marginals. By the Donsker--Varadhan variational formula, for every $\lambda>0$,
\[
\mathbb{I}(A;\mathcal D)
=
D(P\|Q)
\ge
\lambda\,\mathbb E_P\bigl[G(\mathcal D,A)\bigr]
-
\log \mathbb E_Q\bigl[e^{\lambda G(\mathcal D,A)}\bigr].
\]
Under $Q$, conditioning on $A=\theta$ and using the sub-Gaussian bound above gives
\[
\mathbb E_Q\bigl[e^{\lambda G(\mathcal D,A)}\bigr]
=
\mathbb E_{A}\Bigl[\mathbb E_{\mathcal D}\bigl[e^{\lambda G(\mathcal D,A)}\mid A\bigr]\Bigr]
\le
e^{\lambda^2/(8n)}.
\]
Hence
\[
\mathbb{I}(A;\mathcal D)
\ge
\lambda\,\mathbb E\bigl[G(\mathcal D,A)\bigr]
-
\frac{\lambda^2}{8n},
\]
which implies
\[
\mathbb E\bigl[G(\mathcal D,A)\bigr]
\le
\frac{\mathbb{I}(A;\mathcal D)}{\lambda}
+
\frac{\lambda}{8n}.
\]
Optimizing the right-hand side over $\lambda>0$ by choosing $\lambda=\sqrt{8n\,\mathbb{I}(A;\mathcal D)}$ yields
\[
\mathbb E\bigl[G(\mathcal D,A)\bigr]
\le
\sqrt{\frac{\mathbb{I}(A;\mathcal D)}{2n}}.
\]
Applying the same argument to $-G(\mathcal D,A)$ gives
\[
-\mathbb E\bigl[G(\mathcal D,A)\bigr]
\le
\sqrt{\frac{\mathbb{I}(A;\mathcal D)}{2n}}.
\]
Combining the two inequalities proves that
\[
\left|
\mathbb E\bigl[L_{\mathcal P}(A)-\widehat{L}_{\mathcal D}(A)\bigr]
\right|
\le
\sqrt{\frac{\mathbb{I}(A;\mathcal D)}{2n}}.
\]
\end{proof}

\begin{proof}[Proof of Proposition~\ref{eq:prop_1}]
Let
\[
L_{\mathcal P}^\star := \inf_{\theta\in\Theta} L_{\mathcal P}(\theta),
\qquad
\theta^\star \in \arg\min_{\theta\in\Theta} L_{\mathcal P}(\theta),
\qquad
K^\star := \kappa(\theta^\star).
\]

We first analyze the top-down output $\widehat\theta_{\mathrm{TD}}=g(\widehat K)$. Since $\widehat K\in\arg\min_{K\in\mathcal K}\widehat L_{\mathcal D}(g(K))$, the candidate $K^\star$ is feasible and therefore
\[
\widehat L_{\mathcal D}\bigl(g(\widehat K)\bigr)
\le
\widehat L_{\mathcal D}\bigl(g(K^\star)\bigr).
\]
Now decompose the excess population risk of the top-down output as
\[
L_{\mathcal P}\bigl(g(\widehat K)\bigr)-L_{\mathcal P}(\theta^\star)
=
\Bigl(L_{\mathcal P}\bigl(g(\widehat K)\bigr)-\widehat L_{\mathcal D}\bigl(g(\widehat K)\bigr)\Bigr)
+
\Bigl(\widehat L_{\mathcal D}\bigl(g(\widehat K)\bigr)-\widehat L_{\mathcal D}\bigl(g(K^\star)\bigr)\Bigr)
\]
\[
\hspace{2.6cm}
+
\Bigl(\widehat L_{\mathcal D}\bigl(g(K^\star)\bigr)-L_{\mathcal P}\bigl(g(K^\star)\bigr)\Bigr)
+
\Bigl(L_{\mathcal P}\bigl(g(K^\star)\bigr)-L_{\mathcal P}(\theta^\star)\Bigr).
\]
Taking expectations over both $\mathcal D$ and the internal randomness of the framework, we bound these four terms separately.

For the first term, apply the lemma with the data-dependent output $A=g(\widehat K)=\widehat\theta_{\mathrm{TD}}$:
\[
\mathbb E\Bigl[L_{\mathcal P}\bigl(g(\widehat K)\bigr)-\widehat L_{\mathcal D}\bigl(g(\widehat K)\bigr)\Bigr]
\le
\sqrt{\frac{\mathbb{I}(g(\widehat K);\mathcal D)}{2n}}.
\]
Since $g(\widehat K)$ is a deterministic function of $\widehat K$, the data processing inequality gives
\[
\mathbb{I}(g(\widehat K);\mathcal D)\le \mathbb{I}(\widehat K;\mathcal D),
\]
and hence
\[
\mathbb E\Bigl[L_{\mathcal P}\bigl(g(\widehat K)\bigr)-\widehat L_{\mathcal D}\bigl(g(\widehat K)\bigr)\Bigr]
\le
\sqrt{\frac{\mathbb{I}(\widehat K;\mathcal D)}{2n}}.
\]

The second term is non-positive because $\widehat K$ minimizes the empirical risk over the realized knowledge class:
\[
\widehat L_{\mathcal D}\bigl(g(\widehat K)\bigr)-\widehat L_{\mathcal D}\bigl(g(K^\star)\bigr)\le 0.
\]

For the third term, note that $g(K^\star)$ is a fixed heuristic independent of the sampled dataset. Therefore,
\[
\mathbb E\Bigl[\widehat L_{\mathcal D}\bigl(g(K^\star)\bigr)\Bigr]
=
L_{\mathcal P}\bigl(g(K^\star)\bigr),
\]
so its expectation is exactly zero:
\[
\mathbb E\Bigl[\widehat L_{\mathcal D}\bigl(g(K^\star)\bigr)-L_{\mathcal P}\bigl(g(K^\star)\bigr)\Bigr]=0.
\]

For the fourth term, by the definition of the distortion
\[
\delta_{\mathcal P}(g,\kappa)
:=
\sup_{\theta\in\Theta}
\Bigl(
L_{\mathcal P}\bigl(g(\kappa(\theta))\bigr)-L_{\mathcal P}(\theta)
\Bigr),
\]
we have
\[
L_{\mathcal P}\bigl(g(K^\star)\bigr)-L_{\mathcal P}(\theta^\star)
=
L_{\mathcal P}\bigl(g(\kappa(\theta^\star))\bigr)-L_{\mathcal P}(\theta^\star)
\le
\delta_{\mathcal P}(g,\kappa).
\]

Combining the four bounds yields
\[
\mathbb E\bigl[L_{\mathcal P}(\widehat\theta_{\mathrm{TD}})\bigr]-L_{\mathcal P}^\star
\le
\delta_{\mathcal P}(g,\kappa)
+
\sqrt{\frac{\mathbb{I}(\widehat K;\mathcal D)}{2n}}.
\]

We now turn to the bottom-up output $\widehat\theta_{\mathrm{BU}}$. Since $\widehat\theta_{\mathrm{BU}}\in\arg\min_{\theta\in\Theta}\widehat L_{\mathcal D}(\theta)$, we have
\[
\widehat L_{\mathcal D}(\widehat\theta_{\mathrm{BU}})
\le
\widehat L_{\mathcal D}(\theta^\star).
\]
Decompose its excess population risk as
\[
L_{\mathcal P}(\widehat\theta_{\mathrm{BU}})-L_{\mathcal P}(\theta^\star)
=
\Bigl(L_{\mathcal P}(\widehat\theta_{\mathrm{BU}})-\widehat L_{\mathcal D}(\widehat\theta_{\mathrm{BU}})\Bigr)
+
\Bigl(\widehat L_{\mathcal D}(\widehat\theta_{\mathrm{BU}})-\widehat L_{\mathcal D}(\theta^\star)\Bigr)
\]
\[
\hspace{4.1cm}
+
\Bigl(\widehat L_{\mathcal D}(\theta^\star)-L_{\mathcal P}(\theta^\star)\Bigr).
\]
Taking expectations, the first term is bounded by the lemma:
\[
\mathbb E\Bigl[L_{\mathcal P}(\widehat\theta_{\mathrm{BU}})-\widehat L_{\mathcal D}(\widehat\theta_{\mathrm{BU}})\Bigr]
\le
\sqrt{\frac{\mathbb{I}(\widehat\theta_{\mathrm{BU}};\mathcal D)}{2n}}.
\]
The second term is non-positive because $\widehat\theta_{\mathrm{BU}}$ minimizes the empirical risk over $\Theta$, and the third term has expectation zero because $\theta^\star$ is fixed. Hence
\[
\mathbb E\bigl[L_{\mathcal P}(\widehat\theta_{\mathrm{BU}})\bigr]-L_{\mathcal P}^\star
\le
\sqrt{\frac{\mathbb{I}(\widehat\theta_{\mathrm{BU}};\mathcal D)}{2n}}.
\]
This proves the proposition.
\end{proof}

Proposition~\ref{eq:prop_1} reveals the exact source of the trade-off between top-down and bottom-up search. The term $\delta_{\mathcal P}(g,\kappa)$ is a representation error: it measures how much quality can be lost when a heuristic is compressed into a knowledge state and then reconstructed through the canonical realization map $g$. If heuristics that share the same knowledge state have similar population quality, then this distortion is small and the compression is faithful. By contrast, if fine-grained implementation details within a knowledge class strongly affect performance, then $\delta_{\mathcal P}(g,\kappa)$ can be large.

The mutual-information terms quantify adaptive dependence on the sampled dataset. Since the terminal top-down output is mediated through the knowledge state $\widehat K$, its adaptive complexity is controlled by $\mathbb{I}(\widehat K;\mathcal D)$. Bottom-up search does not incur representation distortion, but it exposes the final heuristic itself to the sampled dataset, which leads to the larger comparison term $\mathbb{I}(\widehat\theta_{\mathrm{BU}};\mathcal D)$.

Taken together, the proposition explains when top-down AHD should outperform bottom-up AHD. Top-down is preferred when the knowledge representation is both coherent and compressive: coherent, so that $\delta_{\mathcal P}(g,\kappa)$ is small, and compressive, so that $\mathbb{I}(\widehat K;\mathcal D)\ll \mathbb{I}(\widehat\theta_{\mathrm{BU}};\mathcal D)$. Bottom-up is preferred when the abstraction map $\kappa$ is too coarse to preserve performance-critical implementation details, in which case the distortion term may dominate the advantage gained from compression.

As a final remark, the exact-search assumption is made only to keep the statement clean. If the terminal top-down and bottom-up procedures solve their empirical objectives only approximately, with optimization errors $\eta_{\mathrm{TD}}(\mathcal D)$ and $\eta_{\mathrm{BU}}(\mathcal D)$, then the same proof yields the more general bounds
\[
\mathbb E\bigl[L_{\mathcal P}(\widehat\theta_{\mathrm{TD}})\bigr]-L_{\mathcal P}^\star
\le
\delta_{\mathcal P}(g,\kappa)+\mathbb E[\eta_{\mathrm{TD}}(\mathcal D)]+\sqrt{\frac{\mathbb{I}(\widehat K;\mathcal D)}{2n}},
\]
and
\[
\mathbb E\bigl[L_{\mathcal P}(\widehat\theta_{\mathrm{BU}})\bigr]-L_{\mathcal P}^\star
\le
\mathbb E[\eta_{\mathrm{BU}}(\mathcal D)]+\sqrt{\frac{\mathbb{I}(\widehat\theta_{\mathrm{BU}};\mathcal D)}{2n}}.
\]

\appsubsection{Offline Transfer Across Domains}
\label{app:transfer}

We now consider an offline transfer setting from a source domain $s$ to a target domain $t$. Let
\[
s=(\mathcal X_s,\mathcal Y_s,f_s), \qquad t=(\mathcal X_t,\mathcal Y_t,f_t),
\]
with source and target datasets $\mathcal D_s\sim\mathcal P_s^{n_s}$ and $\mathcal D_t\sim\mathcal P_t^{n_t}$. For the target domain, we write
\[
L_{\mathcal P_t}(\theta) := \mathbb E_{\mathbf{x}\sim\mathcal P_t}[\ell_t(\theta;\mathbf{x})], \qquad \widehat L_{\mathcal D_t}(\theta) := \frac{1}{n_t}\sum_{i=1}^{n_t}\ell_t(\theta;\mathbf{x}_i).
\]

A key issue in offline transfer is that the source and target heuristic spaces, $\Theta_s$ and $\Theta_t$, need not share the same execution signature. This is especially relevant when transferring across related but non-identical graph optimization tasks, such as from TSP to CVRP. In such cases, the source heuristic $\widehat\theta_s\in\Theta_s$ is not generally the natural transfer object for the target side. Instead, the transfer object should be defined by the representation that the target-side adaptation procedure actually consumes.

Accordingly, we distinguish two offline transfer representations. For a source run produced by top-down AHD, the natural transfer object is the terminal knowledge state
\[
R_s^{\mathrm{TD}} := \widehat K_s \in \mathcal K.
\]
For a source run produced by bottom-up AHD, the natural transfer object is the evaluated heuristic artifact
\[
R_s^{\mathrm{BU}} := (A_s,\widehat L_s), \qquad A_s := \alpha_s(\widehat\theta_s), \qquad \widehat L_s := \widehat L_{\mathcal D_s}(\widehat\theta_s).
\]
This distinction reflects the fact that bottom-up reflection operates on the observed artifact together with its empirical feedback, rather than on the executable heuristic as an extensional object.

To support cross-domain transfer, we assume that the knowledge space $\mathcal K$ is shared across domains, while its realizations are domain-specific. Concretely, for each domain $d\in\{s,t\}$, let
\[
g_d:\mathcal K\to\Theta_d
\]
be a realization map that instantiates a knowledge state as a heuristic in domain $d$. Thus, the same knowledge state may be realized differently in TSP and CVRP, even though it captures shared structural concepts such as edge quality, local compactness, or future routing flexibility.

Given a transferred source representation, the target side adapts it using target data:
\[
\widehat\theta_{t\leftarrow s}^{\mathrm{TD}} = \mathcal A_t^{\mathrm{TD}}(R_s^{\mathrm{TD}},\mathcal D_t), \qquad \widehat\theta_{t\leftarrow s}^{\mathrm{BU}} = \mathcal A_t^{\mathrm{BU}}(R_s^{\mathrm{BU}},\mathcal D_t).
\]
We use these two outputs to compare knowledge transfer and artifact-based transfer.

We interpret Proposition~\ref{prop:transfer} at the meta-transfer level. Specifically, we fix a source domain $s$ and draw a target domain $T\sim \Pi_s$ from a distribution over domains related to $s$. For each target domain $t$, let $\theta_t^\star\in\arg\min_{\theta\in\Theta_t}L_{\mathcal P_t}(\theta)$ be a measurable choice of target-optimal heuristic. Then $\theta_T^\star$ is a random variable induced by the sampled target domain $T$, although it becomes fixed once a particular target domain is conditioned upon.

\begin{proposition}[Compressed sufficient transfer representation]
\label{prop:transfer}
Suppose that there exists a measurable map $\psi$ such that
\[
R_s^{\mathrm{TD}} = \psi(R_s^{\mathrm{BU}})
\]
and that
\[
\theta_T^\star \perp\!\!\!\perp R_s^{\mathrm{BU}} \mid R_s^{\mathrm{TD}}.
\]
Then $R_s^{\mathrm{TD}}$ is a sufficient statistic of $R_s^{\mathrm{BU}}$ for target-relevant transfer information, in the sense that
\[
p(\theta_T^\star\mid R_s^{\mathrm{BU}}) = p(\theta_T^\star\mid R_s^{\mathrm{TD}})
\quad \text{a.s.}
\]
Moreover, the source dependence of the transferred representation is no larger under top-down transfer:
\[
\mathbb I(R_s^{\mathrm{TD}};\mathcal D_s) \le \mathbb I(R_s^{\mathrm{BU}};\mathcal D_s).
\]
\end{proposition}

\begin{proof}
Since $R_s^{\mathrm{TD}}=\psi(R_s^{\mathrm{BU}})$ is a measurable function of $R_s^{\mathrm{BU}}$, conditioning on $R_s^{\mathrm{BU}}$ also determines $R_s^{\mathrm{TD}}$. Therefore,
\[
p(\theta_T^\star\mid R_s^{\mathrm{BU}})
=
p(\theta_T^\star\mid R_s^{\mathrm{BU}},R_s^{\mathrm{TD}}).
\]
Using the conditional independence assumption $\theta_T^\star \perp\!\!\!\perp R_s^{\mathrm{BU}} \mid R_s^{\mathrm{TD}}$, we obtain
\[
p(\theta_T^\star\mid R_s^{\mathrm{BU}},R_s^{\mathrm{TD}})
=
p(\theta_T^\star\mid R_s^{\mathrm{TD}}),
\]
which proves the sufficiency claim.

For the mutual-information inequality, note again that $R_s^{\mathrm{TD}}$ is a deterministic function of $R_s^{\mathrm{BU}}$. Hence, by the data processing inequality,
\[
\mathbb I(R_s^{\mathrm{TD}};\mathcal D_s)\le \mathbb I(R_s^{\mathrm{BU}};\mathcal D_s).
\]
This proves the proposition.
\end{proof}

This is an idealized comparison in which both the top-down and bottom-up source runs are assumed to return their best terminal representations. We assume $R_s^{\mathrm{TD}}=\psi(R_s^{\mathrm{BU}})$ because, at the representation level, a knowledge state can be viewed as an abstraction of an evaluated code artifact and its feedback; the assumption therefore captures the best-case setting where the top-down representation preserves exactly the target-relevant abstraction extractable from the bottom-up artifact.

Proposition~\ref{prop:transfer} formalizes the sense in which knowledge can be more transferable than code at the meta-transfer level. The point is not that $R_s^{\mathrm{TD}}$ contains more information than $R_s^{\mathrm{BU}}$; indeed, the proposition shows the opposite in mutual-information terms. Rather, $R_s^{\mathrm{TD}}$ can be preferable when it is a compressed representation that preserves the source information relevant to the sampled target optimum while discarding source-specific nuisance details.

The statistical benefit of such a representation appears on the target side through a conditional mutual-information bound.

\paragraph{Lemma.}
Assume that $0\le \ell_t(\theta;\mathbf{x})\le 1$ for all $\theta\in\Theta_t$ and all $\mathbf{x}\in\operatorname{supp}(\mathcal P_t)$. Let $R_s$ be any source-side transfer representation satisfying $R_s \perp\!\!\!\perp \mathcal D_t$, and let $\widehat\theta_{t\leftarrow s}=\mathcal A_t(R_s,\mathcal D_t)$ be the target-side output adapted from $R_s$. Then
\[
\left|
\mathbb E\bigl[L_{\mathcal P_t}(\widehat\theta_{t\leftarrow s})-\widehat L_{\mathcal D_t}(\widehat\theta_{t\leftarrow s})\bigr]
\right|
\le
\sqrt{\frac{\mathbb I(\widehat\theta_{t\leftarrow s};\mathcal D_t\mid R_s)}{2n_t}}.
\]

\begin{proof}
Condition on a fixed value $R_s=r$. Since $r$ is fixed, the output $\widehat\theta_{t\leftarrow s}=\mathcal A_t(r,\mathcal D_t)$ is simply a data-dependent hypothesis on the target dataset. By the same bounded-loss mutual-information argument used in the proof of Proposition~\ref{eq:prop_1}, we have
\[
\left|
\mathbb E\bigl[L_{\mathcal P_t}(\widehat\theta_{t\leftarrow s})-\widehat L_{\mathcal D_t}(\widehat\theta_{t\leftarrow s}) \mid R_s=r\bigr]
\right|
\le
\sqrt{\frac{\mathbb I(\widehat\theta_{t\leftarrow s};\mathcal D_t\mid R_s=r)}{2n_t}}.
\]
Taking expectations over $R_s$ and applying Jensen's inequality gives
\[
\left|
\mathbb E\bigl[L_{\mathcal P_t}(\widehat\theta_{t\leftarrow s})-\widehat L_{\mathcal D_t}(\widehat\theta_{t\leftarrow s})\bigr]
\right|
\le
\mathbb E\left[\sqrt{\frac{\mathbb I(\widehat\theta_{t\leftarrow s};\mathcal D_t\mid R_s)}{2n_t}}\right]
\le
\sqrt{\frac{\mathbb I(\widehat\theta_{t\leftarrow s};\mathcal D_t\mid R_s)}{2n_t}}.
\]
This proves the lemma.
\end{proof}

Applying the lemma to the two offline transfer representations yields
\[
\left|
\mathbb E\bigl[L_{\mathcal P_t}(\widehat\theta_{t\leftarrow s}^{\mathrm{TD}})-\widehat L_{\mathcal D_t}(\widehat\theta_{t\leftarrow s}^{\mathrm{TD}})\bigr]
\right|
\le
\sqrt{\frac{\mathbb I(\widehat\theta_{t\leftarrow s}^{\mathrm{TD}};\mathcal D_t\mid R_s^{\mathrm{TD}})}{2n_t}},
\]
\[
\left|
\mathbb E\bigl[L_{\mathcal P_t}(\widehat\theta_{t\leftarrow s}^{\mathrm{BU}})-\widehat L_{\mathcal D_t}(\widehat\theta_{t\leftarrow s}^{\mathrm{BU}})\bigr]
\right|
\le
\sqrt{\frac{\mathbb I(\widehat\theta_{t\leftarrow s}^{\mathrm{BU}};\mathcal D_t\mid R_s^{\mathrm{BU}})}{2n_t}}.
\]
These inequalities show that the effectiveness of offline transfer is governed by how much additional target-side information is still needed after the source representation has been provided.

Taken together, the proposition and lemma suggest the following interpretation. Top-down transfer pays a realization cost: the transferred object $R_s^{\mathrm{TD}}=\widehat K_s$ is abstract and must be instantiated in the target domain through $\mathcal A_t^{\mathrm{TD}}$ or, in the simplest case, through $g_t$. If the knowledge state omits performance-critical low-level details, this abstraction may introduce bias. Bottom-up transfer instead carries the richer evaluated artifact $R_s^{\mathrm{BU}}=(A_s,\widehat L_s)$, which preserves more low-level information but also retains more source-specific detail and therefore may require heavier adaptation on the target side.

Under the sufficient-statistic condition of Proposition~\ref{prop:transfer}, $R_s^{\mathrm{TD}}$ preserves all source information that is relevant to the target optimum while containing no more source dependence than $R_s^{\mathrm{BU}}$. In this regime, knowledge transfer is more general than artifact-based transfer in an information-theoretic sense: it is a compressed representation that preserves target-relevant invariants. If, in addition,
\[
\mathbb I(\widehat\theta_{t\leftarrow s}^{\mathrm{TD}};\mathcal D_t\mid R_s^{\mathrm{TD}})
\ll
\mathbb I(\widehat\theta_{t\leftarrow s}^{\mathrm{BU}};\mathcal D_t\mid R_s^{\mathrm{BU}}),
\]
then top-down transfer enjoys a strictly smaller target-side generalization burden and is therefore expected to be more sample-efficient.

Conversely, bottom-up transfer may be preferable when target performance depends on low-level implementation details that are present in $R_s^{\mathrm{BU}}$ but are not captured by $R_s^{\mathrm{TD}}$. This is precisely the setting in which knowledge compression ceases to be faithful, and the abstraction bias of top-down transfer can outweigh its information-theoretic benefit.

\appsubsection{Sparse Evaluation as Information Regularization}
\label{app:sparse_theory}

Sparse evaluation can be interpreted as a form of information regularization. Recall the Bayesian refinement rule in Equation~\ref{eq:bayes}: when a candidate is evaluated, the posterior over latent knowledge uses both the artifact $A^t$ and the empirical feedback $\widehat L^t$. When the candidate is not evaluated, $\widehat L^t$ is unobserved and the posterior marginalizes over the missing feedback:
\[
q_t(K\mid A^t,C)
=
\int q_t(K,\widehat L\mid A^t,C)\,d\widehat L.
\]
Using Equation~\ref{eq:bayes}, this gives
\begin{equation}
q_t(K\mid A^t,C)
\propto
q_t(A^t\mid K,C)\,q_t(K\mid C),
\end{equation}
because $\int q_t(\widehat L\mid A^t,K,C)\,d\widehat L=1$. Thus, unevaluated candidates are not discarded: they still contribute structural evidence through $q_t(A^t\mid K,C)$, but they do not contribute evaluative evidence through $q_t(\widehat L^t\mid A^t,K,C)$.

This missing evaluative likelihood has a statistical consequence. Each observed empirical loss $\widehat L^t=\widehat L_{\mathcal D}(\theta^t)$ reveals information about the finite training set $\mathcal D$. Revealing fewer such losses can reduce the adaptive dependence of the final output on $\mathcal D$, and hence reduce overfitting to the training set. We formalize this effect below.

Let $M^t\in\{0,1\}$ denote whether the candidate at round $t$ is evaluated. Let $\rho_t:=\mathbb P(M^t=1)$ be the evaluation rate, and let $H_\rho^T$ denote the final search history under sparse evaluation. Let $\widehat K_\rho$ be the terminal knowledge state returned by top-down search and let $\widehat\theta_{\mathrm{TD},\rho}:=g(\widehat K_\rho)$ be the final realized heuristic. Define the sparse-evaluation optimization residual
\[
\eta_{\mathrm{TD}}^{(\rho)}(\mathcal D)
:=
\widehat L_{\mathcal D}\bigl(g(\widehat K_\rho)\bigr)
-
\inf_{K\in\mathcal K}\widehat L_{\mathcal D}\bigl(g(K)\bigr).
\]

\begin{proposition}[Sparse evaluation controls adaptive information]
Assume that $0\le \ell_d(\theta;\mathbf{x})\le 1$ for all $\theta\in\Theta$ and $\mathbf{x}\in\operatorname{supp}(\mathcal P)$. Suppose that, under sparse evaluation, new information about $\mathcal D$ enters the search history only through evaluated empirical losses. Moreover, assume that each evaluated loss contributes at most $\gamma_t$ nats of conditional information about $\mathcal D$, i.e.,
\[
\mathbb{I}\bigl(\widehat L^t;\mathcal D\mid H_\rho^{t-1},A^t,M^t=1\bigr)\le \gamma_t.
\]
Then
\[
\mathbb E\bigl[L_{\mathcal P}(\widehat\theta_{\mathrm{TD},\rho})\bigr]-L_{\mathcal P}^\star
\le
\delta_{\mathcal P}(g,\kappa)
+
\mathbb E\bigl[\eta_{\mathrm{TD}}^{(\rho)}(\mathcal D)\bigr]
+
\sqrt{\frac{\sum_{t=1}^T \rho_t\gamma_t}{2n}},
\]
where $L_{\mathcal P}^\star:=\inf_{\theta\in\Theta}L_{\mathcal P}(\theta)$.
\end{proposition}

\begin{proof}
The proof follows the distortion--compression argument of Proposition~\ref{eq:prop_1}, with an additional bound on the information term. Since $\widehat\theta_{\mathrm{TD},\rho}=g(\widehat K_\rho)$, the same decomposition gives
\[
\mathbb E\bigl[L_{\mathcal P}(\widehat\theta_{\mathrm{TD},\rho})\bigr]-L_{\mathcal P}^\star
\le
\delta_{\mathcal P}(g,\kappa)
+
\mathbb E\bigl[\eta_{\mathrm{TD}}^{(\rho)}(\mathcal D)\bigr]
+
\sqrt{\frac{\mathbb{I}(\widehat K_\rho;\mathcal D)}{2n}}.
\]
It remains to control $\mathbb{I}(\widehat K_\rho;\mathcal D)$.

Since $\widehat K_\rho$ is a function of the final sparse-evaluation history $H_\rho^T$, the data processing inequality gives
\[
\mathbb{I}(\widehat K_\rho;\mathcal D)\le \mathbb{I}(H_\rho^T;\mathcal D).
\]
By the chain rule for mutual information,
\[
\mathbb{I}(H_\rho^T;\mathcal D)
\le
\sum_{t=1}^T
\mathbb{I}\bigl(M^t,A^t,\widetilde L^t;\mathcal D\mid H_\rho^{t-1}\bigr),
\]
where $\widetilde L^t=\widehat L^t$ if $M^t=1$ and $\widetilde L^t=\bot$ otherwise. Under the assumption that new information about $\mathcal D$ enters only through evaluated empirical losses, the terms involving $M^t$ and $A^t$ do not add information about $\mathcal D$ beyond the current history. Hence
\[
\mathbb{I}\bigl(M^t,A^t,\widetilde L^t;\mathcal D\mid H_\rho^{t-1}\bigr)
=
\mathbb{I}\bigl(\widetilde L^t;\mathcal D\mid H_\rho^{t-1},A^t,M^t\bigr).
\]
When $M^t=0$, $\widetilde L^t=\bot$ is deterministic and contributes zero information. Therefore,
\[
\mathbb{I}\bigl(\widetilde L^t;\mathcal D\mid H_\rho^{t-1},A^t,M^t\bigr)
=
\rho_t\,\mathbb{I}\bigl(\widehat L^t;\mathcal D\mid H_\rho^{t-1},A^t,M^t=1\bigr)
\le
\rho_t\gamma_t.
\]
Summing over $t$ yields
\[
\mathbb{I}(\widehat K_\rho;\mathcal D)\le \sum_{t=1}^T\rho_t\gamma_t.
\]
Substituting this into the previous excess-risk bound proves the result.
\end{proof}

The proposition makes explicit the trade-off introduced by sparse evaluation. Reducing the evaluation rate lowers the adaptive information term from the full-evaluation scale $\sum_t\gamma_t$ to $\sum_t\rho_t\gamma_t$, thereby acting as a regularizer against overfitting to $\mathcal D$. This benefit is not free: fewer evaluated candidates can increase the optimization residual $\eta_{\mathrm{TD}}^{(\rho)}(\mathcal D)$. Sparse evaluation is therefore beneficial when the reduction in adaptive information outweighs the increase in this residual.

This also clarifies why top-down search can be especially effective under sparse evaluation. Even when $\widehat L^t$ is missing, Equation~\eqref{eq:sparse_bayes} shows that unevaluated artifacts still provide structural evidence about the latent knowledge state through $q_t(A^t\mid K,C)$. Thus, top-down search can continue refining and recombining knowledge using unevaluated candidates, while relying on a smaller number of evaluated candidates to calibrate performance. Bottom-up search, by contrast, depends more directly on empirical feedback to explain why a heuristic is good or bad; when evaluations are sparse, many bottom-up updates become structure-only and lose the evaluative signal needed to distinguish genuinely useful heuristics from merely plausible ones.

\appsection{Benchmark Problems}
\label{app:problems}

\appsubsection{Canonical CO Problems}

\paragraph{Traveling Salesman Problem (TSP).}
Let $G=(V,E)$ be a complete directed graph on a set of cities $V=\{1,\dots,n\}$, and let $d_{ij}\ge 0$ denote the travel cost from city $i$ to city $j$. The TSP seeks a minimum-cost Hamiltonian tour visiting each city exactly once and returning to the starting city. A standard formulation uses binary variables $x_{ij}\in\{0,1\}$ indicating whether the tour travels directly from $i$ to $j$:
\[
\min \sum_{i\in V}\sum_{j\in V,\; j\neq i} d_{ij}x_{ij}
\]
subject to
\[
\sum_{j\in V,\; j\neq i} x_{ij} = 1 \quad \forall i\in V,
\qquad
\sum_{i\in V,\; i\neq j} x_{ij} = 1 \quad \forall j\in V,
\]
and the subtour-elimination constraints
\[
\sum_{i\in S}\sum_{j\in S,\; j\neq i} x_{ij} \le |S|-1
\quad
\forall S\subset V,\; 2\le |S|\le n-1.
\]
The objective minimizes the total travel cost over all feasible tours.

In our experiments, we consider four TSP distributions, namely uniform, clustered, diagonal, and barbell, in order to test TSP solvers under qualitatively different geometric structures.

\emph{(i) Uniform TSP.} Cities are sampled independently from the uniform distribution on the unit square, i.e., $x_i \sim \mathrm{Unif}([0,1]^2)$. This distribution is difficult because it does not provide strong geometric regularities such as clusters or bottlenecks, so many candidate edges have similar local quality and the main challenge is global tour coordination.

\emph{(ii) Clustered TSP.} Cities are concentrated around a small number of spatial clusters; in our data, points are generated around six cluster centers arranged roughly on a circle. This distribution is difficult because the solver must simultaneously exploit short intra-cluster edges and choose good inter-cluster connections, and poor decisions about cluster ordering or entry/exit points can substantially increase the tour length.

\emph{(iii) Diagonal TSP.} Cities are sampled near the main diagonal of the unit square, so that most points lie close to a nearly one-dimensional manifold. This distribution is difficult because many cities become nearly collinear, which creates many near-ties between alternative local connections and makes the global visiting order highly sensitive to small geometric perturbations.

\emph{(iv) Barbell TSP.} Cities are distributed over two dense groups connected by a narrow bridge of intermediate points, yielding a barbell-shaped geometry. This distribution is difficult because it combines local clustering with a structural bottleneck: the solver must build efficient subtours inside each dense group while also handling the bridge correctly, since suboptimal transitions across the bottleneck can incur a large global penalty.

\begin{figure}[t]
  \centering
  \includegraphics[width=\linewidth]{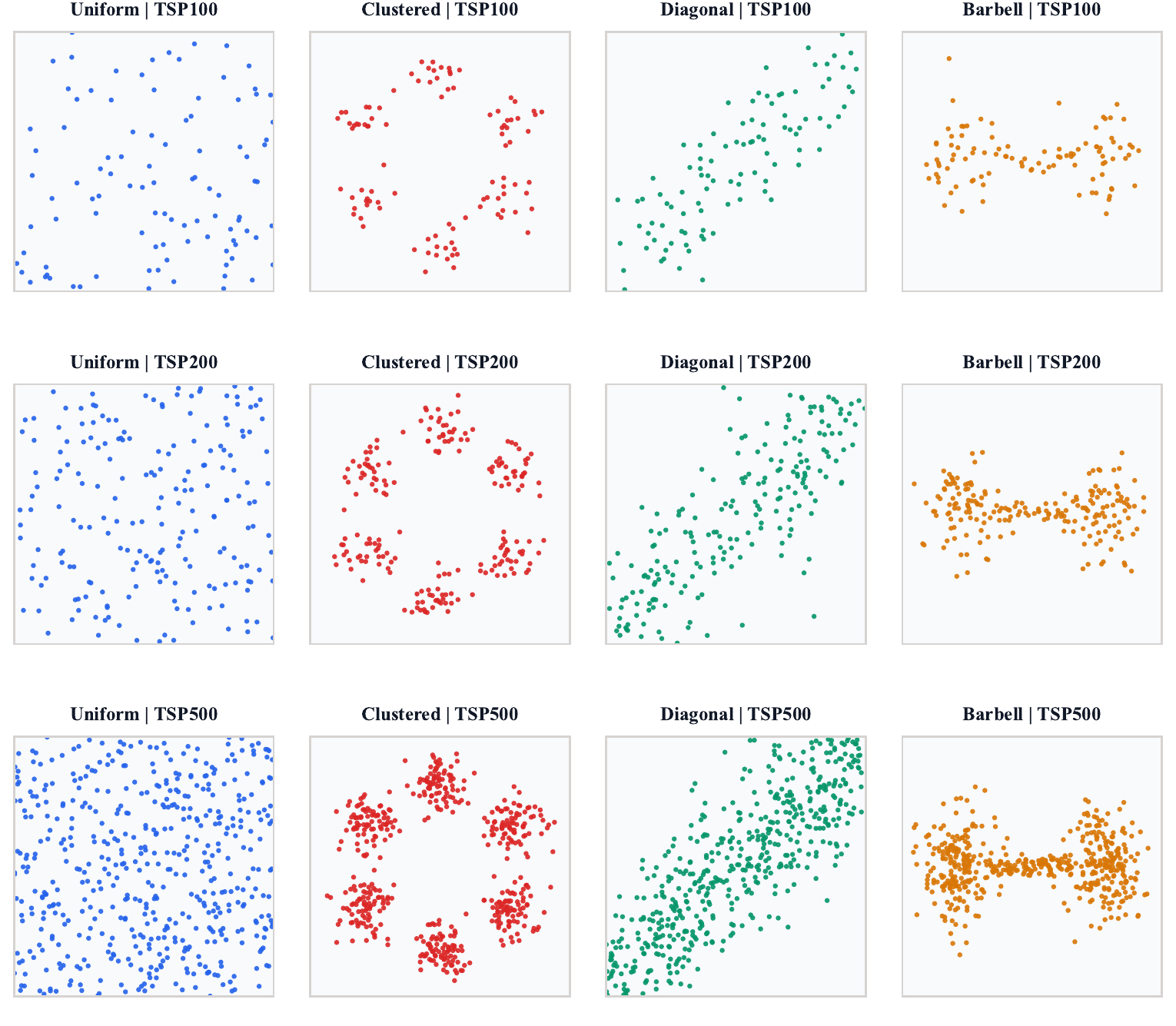}
  \caption{Example TSP instances from four spatial distributions at sizes 100, 200, and 500.}
  \label{fig:tsp_multi_dist}
\end{figure}

\paragraph{Capacitated Vehicle Routing Problem (CVRP).}
Let $V=\{0\}\cup N$ be the set of nodes, where node $0$ is the depot and $N=\{1,\dots,n\}$ is the set of customers. Each customer $i\in N$ has demand $q_i>0$, each vehicle has capacity $Q$, and $d_{ij}\ge 0$ denotes the travel cost from node $i$ to node $j$. The CVRP seeks a collection of depot-to-depot routes of minimum total cost such that every customer is visited exactly once and the total demand served by each route does not exceed $Q$. Using binary variables $x_{ij}\in\{0,1\}$ and load variables $u_i$, one common formulation is
\[
\min \sum_{i\in V}\sum_{j\in V,\; j\neq i} d_{ij}x_{ij}
\]
subject to
\[
\sum_{j\in V,\; j\neq i} x_{ij} = 1 \quad \forall i\in N,
\qquad
\sum_{i\in V,\; i\neq j} x_{ij} = 1 \quad \forall j\in N,
\]
\[
q_i \le u_i \le Q \quad \forall i\in N,
\]
\[
u_j \ge u_i + q_j - Q(1-x_{ij})
\quad
\forall i,j\in N,\; i\neq j.
\]
The depot degree constraints determine the number of used vehicles, and the load constraints eliminate infeasible subtours while enforcing vehicle capacities.

In addition, we also consider two variants of VRP, namely LVRP and OVRP.

\paragraph{Duration-Limited Vehicle Routing Problem (LVRP).}
LVRP extends CVRP by imposing, in addition to vehicle-capacity constraints, a maximum route length. Let $V=\{0\}\cup N$ be the set of nodes, where node $0$ is the depot and $N=\{1,\dots,n\}$ is the set of customers. Each customer $i\in N$ has demand $q_i>0$, each vehicle has capacity $Q$, $d_{ij}\ge 0$ denotes the travel cost from node $i$ to node $j$, and $L>0$ is the maximum allowable length of each route. Using binary variables $x_{ij}\in\{0,1\}$, load variables $u_i$, and arrival-distance variables $t_i$, one formulation is
\[
\min \sum_{i\in V}\sum_{j\in V,\; j\neq i} d_{ij}x_{ij}
\]
subject to
\[
\sum_{j\in V,\; j\neq i} x_{ij} = 1 \quad \forall i\in N,
\qquad
\sum_{i\in V,\; i\neq j} x_{ij} = 1 \quad \forall j\in N,
\]
\[
q_i \le u_i \le Q \quad \forall i\in N,
\]
\[
u_j \ge u_i + q_j - Q(1-x_{ij})
\quad
\forall i,j\in N,\; i\neq j,
\]
\[
d_{0i}x_{0i} \le t_i \le L - d_{i0}
\quad
\forall i\in N,
\]
\[
t_j \ge t_i + d_{ij} - L(1-x_{ij})
\quad
\forall i,j\in N,\; i\neq j.
\]
The objective minimizes the total travel cost, while the additional constraints ensure that every route has total length at most $L$, including the return to the depot.

\paragraph{Open Vehicle Routing Problem (OVRP).}
OVRP is another variant of CVRP in which each vehicle starts at the depot but does not need to return to the depot after serving its last customer. Let $V=\{0\}\cup N$ be defined as above. Using binary variables $x_{ij}\in\{0,1\}$, terminal indicators $z_i\in\{0,1\}$, and load variables $u_i$, one formulation is
\[
\min \sum_{i\in V}\sum_{j\in N,\; j\neq i} d_{ij}x_{ij}
\]
subject to
\[
\sum_{i\in V,\; i\neq j} x_{ij} = 1 \quad \forall j\in N,
\]
\[
\sum_{j\in N,\; j\neq i} x_{ij} + z_i = 1 \quad \forall i\in N,
\]
\[
\sum_{j\in N} x_{0j} = \sum_{i\in N} z_i,
\]
\[
q_i \le u_i \le Q \quad \forall i\in N,
\]
\[
u_j \ge u_i + q_j - Q(1-x_{ij})
\quad
\forall i,j\in N,\; i\neq j.
\]
Here, $z_i=1$ indicates that customer $i$ is the last customer of an open route. Compared with CVRP, OVRP preserves customer-coverage and capacity constraints, but routes terminate at customers rather than returning to the depot.

\paragraph{Orienteering Problem (OP).}
Let $V=\{0\}\cup N$ be a set of nodes, where node $0$ is the depot. Each customer node $i\in N$ yields a prize $r_i\ge 0$, and traveling from node $i$ to node $j$ incurs cost $d_{ij}\ge 0$. Given a travel budget $B$, the OP seeks a depot-to-depot walk that collects the maximum total prize while respecting the budget. With binary routing variables $x_{ij}\in\{0,1\}$ and visit indicators $y_i\in\{0,1\}$, a standard formulation is
\[
\max \sum_{i\in N} r_i y_i
\]
subject to
\[
\sum_{j\in V,\; j\neq 0} x_{0j} = 1,
\qquad
\sum_{i\in V,\; i\neq 0} x_{i0} = 1,
\]
\[
\sum_{j\in V,\; j\neq i} x_{ij} = y_i
\quad \forall i\in N,
\qquad
\sum_{j\in V,\; j\neq i} x_{ji} = y_i
\quad \forall i\in N,
\]
\[
\sum_{i\in V}\sum_{j\in V,\; j\neq i} d_{ij}x_{ij} \le B,
\]
together with subtour-elimination constraints. The objective maximizes the collected reward subject to route feasibility and the total travel budget.

\paragraph{Job Shop Scheduling Problem (JSSP).}
Let $\mathcal{J}=\{1,\dots,n\}$ be the set of jobs and $\mathcal{M}=\{1,\dots,m\}$ the set of machines. Each job $j\in\mathcal{J}$ consists of an ordered sequence of operations $(O_{j1},\dots,O_{j\ell_j})$. Operation $O_{jk}$ must be processed on machine $\mu_{jk}\in\mathcal{M}$ for duration $p_{jk}>0$. The JSSP seeks start times $S_{jk}\ge 0$ for all operations so as to minimize the makespan
\[
\min C_{\max}
\]
subject to the precedence constraints within each job,
\[
S_{j,k+1} \ge S_{jk} + p_{jk}
\quad
\forall j\in\mathcal{J},\; k=1,\dots,\ell_j-1,
\]
and the machine-capacity constraints: for any two operations $O_{jk}$ and $O_{j'k'}$ assigned to the same machine, exactly one must precede the other. Using binary variables $z_{jk,j'k'}\in\{0,1\}$ and a sufficiently large constant $M$, this can be written as
\[
S_{j'k'} \ge S_{jk} + p_{jk} - M(1-z_{jk,j'k'}),
\]
\[
S_{jk} \ge S_{j'k'} + p_{j'k'} - Mz_{jk,j'k'}
\]
whenever $\mu_{jk}=\mu_{j'k'}$. Finally,
\[
C_{\max} \ge S_{jk} + p_{jk}
\quad
\forall j\in\mathcal{J},\; k=1,\dots,\ell_j.
\]
The objective minimizes the completion time of the last finished operation.

\paragraph{Quadratic Assignment Problem (QAP).}
Let $\mathcal{F}=\{1,\dots,n\}$ be a set of facilities and $\mathcal{L}=\{1,\dots,n\}$ a set of locations. For each pair of facilities $(a,b)$, let $f_{ab}\ge 0$ denote the flow between them, and for each pair of locations $(i,j)$, let $d_{ij}\ge 0$ denote the distance between locations $i$ and $j$. The QAP seeks a one-to-one assignment of facilities to locations minimizing the total flow-distance interaction cost. Using binary variables $x_{ai}\in\{0,1\}$, where $x_{ai}=1$ if facility $a$ is assigned to location $i$, the problem is
\[
\min \sum_{a\in\mathcal{F}}\sum_{b\in\mathcal{F}}\sum_{i\in\mathcal{L}}\sum_{j\in\mathcal{L}}
f_{ab}\, d_{ij}\, x_{ai}x_{bj}
\]
subject to
\[
\sum_{i\in\mathcal{L}} x_{ai} = 1 \quad \forall a\in\mathcal{F},
\qquad
\sum_{a\in\mathcal{F}} x_{ai} = 1 \quad \forall i\in\mathcal{L},
\]
\[
x_{ai}\in\{0,1\}
\quad
\forall a\in\mathcal{F},\; i\in\mathcal{L}.
\]
The quadratic objective captures the fact that assigning two strongly interacting facilities to distant locations incurs a large cost.

\appsubsection{Synthetic SCO Problems}
\label{app:sco}

\paragraph{Sequential Combinatorial Optimization.}
We consider sequential combinatorial optimization (SCO) problems over a finite item set $\mathcal{V}=\{1,2,\ldots,n\}$. Starting from an initial state $s_1$, the decision process unfolds over $T$ steps, where at each step $1\le t\le T$ the algorithm selects a feasible decision $x_t\in\mathcal{A}_t\subseteq\mathcal{V}$, with $\mathcal{A}_t$ denoting the available set of admissible decisions at step $t$. The environment then transitions according to $s_{t+1}=\Phi(x_t,s_t)$ and yields an immediate reward $r_t=\mathcal{R}(x_t,s_t)$. The goal is to construct a sequence of decisions $(x_1,\ldots,x_T)$ that maximizes the cumulative return $\sum_{t=1}^{T} r_t$, thereby capturing a broad class of routing, scheduling, and other structured decision-making problems in which each action changes the future feasible set and the downstream utility.

\paragraph{Conference Talk Scheduling (CTS).}
We consider a finite set of candidate talks $\mathcal{V}=\{1,\dots,n\}$ and a schedule of fixed length $T$. Each talk $i\in\mathcal{V}$ is associated with a quality score $q_i>0$ and a normalized topic embedding $u_i\in\mathbb{R}^d$. At step $t$, the decision maker selects one unscheduled talk $x_t\in\mathcal{A}_t$, where
\[
\mathcal{A}_t=\mathcal{V}\setminus\{x_1,\dots,x_{t-1}\},
\]
so each talk can be scheduled at most once. The state may be written as $s_t=(\mathcal{A}_t,u_{x_{t-1}})$, with the convention that $u_{x_0}=\mathbf{0}$. After selecting $x_t$, the environment removes this talk from the available set and updates the previous-topic state to $u_{x_t}$. The immediate reward is
\[
r_t=
\begin{cases}
q_{x_t}, & t=1,\\[2mm]
q_{x_t}-\alpha \max\{0,\langle u_{x_t},u_{x_{t-1}}\rangle\}, & t\ge 2,
\end{cases}
\]
where $\alpha>0$ penalizes excessive topical similarity between adjacent talks. The objective is therefore
\[
\max_{x_1,\dots,x_T}\ \sum_{t=1}^{T} r_t
\quad
\text{s.t.}\quad
x_t\in\mathcal{A}_t,\ \ x_t\neq x_{t'}\ \forall t\neq t'.
\]
This formulation captures the trade-off between selecting individually high-quality talks and maintaining topical diversity across consecutive slots.

\paragraph{Online Ad Sequencing with Fatigue (OAS).}
Let $\mathcal{V}=\{1,\dots,n\}$ denote the set of ads and let the platform fill a sequence of $T$ ad slots. Each ad $i\in\mathcal{V}$ has a base value $b_i>0$ and a fatigue rate $\rho_i\in[0,1)$. Repetitions are allowed, so the feasible set is constant over time, namely $\mathcal{A}_t=\mathcal{V}$ for all $t$. The state is the vector of cumulative display counts $s_t=c_t\in\mathbb{Z}_{\ge 0}^n$, where $c_{i,t}$ records how many times ad $i$ has been shown in the first $t-1$ slots. If ad $x_t$ is selected at step $t$, its realized reward is its fatigued value
\[
r_t=b_{x_t}(1-\rho_{x_t})^{c_{x_t,t}},
\]
and the state transitions as
\[
c_{i,t+1}=
\begin{cases}
c_{i,t}+1, & i=x_t,\\
c_{i,t}, & i\neq x_t.
\end{cases}
\]
Hence the optimization problem is
\[
\max_{x_1,\dots,x_T}\ \sum_{t=1}^{T} b_{x_t}(1-\rho_{x_t})^{c_{x_t,t}}
\quad
\text{s.t.}\quad
x_t\in\mathcal{V}\ \forall t.
\]
This objective favors ads with large intrinsic value while accounting for diminishing returns under repeated exposure.

\paragraph{Food-Truck Routing (FTR).}
Consider a set of candidate stops $\mathcal{V}=\{1,\dots,n\}$ over a planning horizon of $T$ decision steps. Each location $i\in\mathcal{V}$ has coordinates $\ell_i\in\mathbb{R}^2$ and a base popularity $p_i>0$. Revisits are allowed, so $\mathcal{A}_t=\mathcal{V}$ for every $t$. The state at step $t$ consists of the current truck position $z_t\in\mathbb{R}^2$ and the last-visit information for each location. Let $\delta_{i,t}$ denote the number of steps since location $i$ was last visited before step $t$, with $\delta_{i,t}=+\infty$ if $i$ has never been visited. When choosing the next stop $x_t\in\mathcal{V}$, the immediate reward is
\[
r_t
=
p_{x_t}\bigl(1-e^{-\lambda \delta_{x_t,t}}\bigr)
-
c\,\| \ell_{x_t}-z_t\|_2,
\]
where $\lambda>0$ controls demand recovery and $c>0$ is the travel-cost coefficient. The state then transitions according to $z_{t+1}=\ell_{x_t}$ together with the appropriate update of all $\delta_{i,t}$. The resulting problem is
\[
\max_{x_1,\dots,x_T}\ \sum_{t=1}^{T}
\left[
p_{x_t}\bigl(1-e^{-\lambda \delta_{x_t,t}}\bigr)
-
c\,\| \ell_{x_t}-z_t\|_2
\right]
\quad
\text{s.t.}\quad
x_t\in\mathcal{V}\ \forall t.
\]
This formulation balances immediate sales potential, spatial travel efficiency, and the benefit of allowing demand at previously visited locations to recover over time.

\paragraph{Warehouse Picking with Fatigue (WPF).}
Let $\mathcal{V}=\{1,\dots,n\}$ denote the set of orders in a warehouse. Each order $i\in\mathcal{V}$ yields value $v_i>0$ and requires a base picking time $w_i>0$. Unlike fixed-horizon settings, the process unfolds under a total time budget $B>0$. If $t-1$ orders have already been picked, then the current fatigue multiplier is
\[
m_t = 1+\mu (t-1),
\]
where $\mu>0$ is the linear fatigue-growth rate. The state can be represented as $s_t=(\mathcal{U}_t,b_t)$, where $\mathcal{U}_t\subseteq\mathcal{V}$ is the set of unpicked orders and $b_t$ is the remaining time budget. At step $t$, the feasible action set is
\[
\mathcal{A}_t
=
\left\{
i\in \mathcal{U}_t \;:\; w_i\, m_t \le b_t
\right\},
\]
namely the set of unpicked orders that are still affordable under the current fatigue level. Choosing order $x_t\in\mathcal{A}_t$ yields immediate reward
\[
r_t=v_{x_t},
\]
and updates the state via
\[
\mathcal{U}_{t+1}=\mathcal{U}_t\setminus\{x_t\},
\qquad
b_{t+1}=b_t-w_{x_t}m_t.
\]
The process terminates at the stopping time
\[
\tau=\min\{t\ge 1:\mathcal{A}_t=\emptyset\},
\]
and the objective is
\[
\max_{x_1,\dots,x_{\tau-1}}\ \sum_{t=1}^{\tau-1} v_{x_t}
\quad
\text{s.t.}\quad
x_t\in\mathcal{A}_t\ \forall t<\tau.
\]
This problem formalizes a sequential knapsack-like decision process in which the effective cost of later actions increases endogenously because of worker fatigue.

\appsubsection{SR Problems}
\label{app:sr}

\paragraph{Symbolic Regression (SR).}
SR seeks an explicit analytical expression that maps observed inputs to outputs from data. Given a dataset
\[
\mathcal{D}=\{(z_i,y_i)\}_{i=1}^m,
\]
where $z_i$ denotes the input variables and $y_i$ the target quantity, the goal is to identify both a symbolic functional form $f\in\mathcal{F}$ and its numerical parameters $\theta$ such that
\[
y \approx f(z;\theta).
\]
In this work, candidate expressions are evaluated using the normalized mean squared error (NMSE), defined as
\[
\mathrm{NMSE}(f,\theta)
=
\frac{1}{m\,\mathbb{V}[y]}
\sum_{i=1}^m \bigl(f(z_i;\theta)-y_i\bigr)^2,
\]
where $\mathbb{V}[y]$ denotes the empirical variance of the target values. Thus, SR can be viewed as the problem of finding an interpretable closed-form expression with low NMSE on the observed data. In this work, we follow the four SR benchmark problems used in LLM-SR~\citep{shojaee2025llmsr}.

\paragraph{Damped Nonlinear Oscillator.}
In the first SR task, the objective is to recover the governing equation of a damped nonlinear oscillator from observations of position and velocity. Let $x$ denote displacement and $v=\dot{x}$ denote velocity. The target is the acceleration $\dot{v}$, and the task is to identify a symbolic expression of the form
\[
\dot{v} = f(x,v;\theta).
\]
Given samples of the state $(x,v)$ and the corresponding acceleration, the goal is to recover a compact analytical law that captures the nonlinear restoring and damping effects of the oscillator.

\paragraph{Driven Damped Nonlinear Oscillator.}
The second SR task extends the previous one by introducing an explicitly time-dependent driving force. In this case, the acceleration depends not only on the state $(x,v)$ but also on time $t$, and the goal is to identify an equation of the form
\[
\dot{v} = f(t,x,v;\theta).
\]
Compared with the first oscillator task, this problem is more challenging because the governing law must simultaneously capture nonlinear dynamics and external temporal forcing.

\paragraph{Bacterial Growth Rate.}
In the third SR task, the objective is to discover an analytical expression for the growth dynamics of \emph{E. coli}. Let $b$ denote population density, $s$ substrate concentration, $T$ temperature, and $\mathrm{pH}$ the acidity level. The target is the growth rate $\dot{b}$, and the task is to identify an equation of the form
\[
\dot{b} = f(b,s,T,\mathrm{pH};\theta).
\]
This problem aims to recover an interpretable growth law that explains how population dynamics depend on both internal state variables and environmental conditions.

\paragraph{Stress--Strain Relation.}
In the fourth SR task, the objective is to infer an empirical constitutive law for material behavior under loading. Let $\varepsilon$ denote strain, $T$ temperature, and $\sigma$ stress. The task is to identify a symbolic relation of the form
\[
\sigma = f(\varepsilon,T;\theta).
\]
Unlike the other SR tasks, this problem is based on empirical material data rather than a known closed-form governing equation, so the objective is to discover an interpretable constitutive relation that accurately describes the observed stress--strain response.

\appsubsection{PE Problems}
\label{app:pe}

\paragraph{Protein Engineering (PE).}
PE aims to identify protein sequences with high functional fitness under a given experimental assay. Let $\mathcal{A}$ denote the amino-acid alphabet, and let $x \in \mathcal{A}^*$ be a protein sequence of fixed or variable length. Each sequence is associated with an experimentally measured property $F(x)$, such as enzymatic activity, stability, or absorption wavelength. In the offline sequential setting considered here, we are given an initial support set
\[
\mathcal{D}_0 = \{(x_i, y_i)\}_{i=1}^{n_0},
\qquad y_i = F(x_i),
\]
together with a finite candidate pool $\mathcal{C} = \{\tilde{x}_1,\dots,\tilde{x}_M\}$ and a query budget $B$. At each round $t=1,\dots,B$, the method selects an untested candidate
\[
x_t \in \mathcal{C} \setminus \{x_1,\dots,x_{t-1}\},
\]
observes its fitness $y_t = F(x_t)$, and augments the observed set accordingly. The goal is to discover high-fitness sequences within the limited budget. In our implementation, performance is measured by the average fitness of the top discovered sequences after $B$ rounds. The following four PE benchmarks are derived from FLIP2~\cite{Didi2026.02.23.707496}.

\begin{table}[h]
\centering
\caption{Dataset sizes for protein sequence design benchmarks.}
\label{tab:psd_dataset_sizes}
\setlength{\tabcolsep}{5pt}
\renewcommand{\arraystretch}{1.2}
\resizebox{\linewidth}{!}{%
\begin{tabular}{|l|p{7.8cm}|C{2.0cm}|C{2.0cm}|}
\hline
\textbf{Problem} & \textbf{Description} & \textbf{Training set} & \textbf{Test set} \\
\hline
\textbf{Alpha Amylase} & Protein sequence design for alpha-amylase activity optimization. & 3218 & 488 \\
\hline
\textbf{Imine Reductase} & Protein sequence design for imine reductase activity optimization. & 4408 & 4178 \\
\hline
\textbf{Hydrophobic Core} & Protein sequence design for improving hydrophobic-core fitness and stability-related properties. & 1450 & 23485 \\
\hline
\textbf{Rhodopsin} & Protein sequence design for rhodopsin functional property optimization. & 700 & 184 \\
\hline
\end{tabular}%
}
\end{table}

\appsubsection{DL Problems}

\paragraph{Discrete Location Problems (DLP).}
DLPs form a family of combinatorial optimization problems in which a decision maker must choose a subset of facilities from a finite set of candidate locations so as to optimize a service or diversity criterion. They are classical models in operations research and logistics, with applications including warehouse and distribution-center placement, emergency-service stationing, healthcare and school siting, telecommunication infrastructure planning, and retail-network design. In these settings, the core trade-off is to determine where limited facilities should be located so as to balance coverage, accessibility, response quality, and spatial diversity.

Let $N=\{1,\dots,n\}$ denote the set of demand points and candidate facility locations, let $d_{ij}\ge 0$ denote the distance from location $i$ to location $j$, and let $p$ be the number of facilities to be selected. Using binary variables $y_j\in\{0,1\}$ to indicate whether facility $j$ is opened, all DLP variants considered here impose the cardinality constraint
\[
\sum_{j\in N} y_j = p.
\]
The four problems differ in how they evaluate the selected facilities.

\paragraph{$p$-Median Problem.}
The $p$-median problem seeks to open $p$ facilities so as to minimize the total assignment cost from demand points to their nearest open facilities. Let $x_{ij}\in\{0,1\}$ indicate whether demand point $i$ is assigned to facility $j$. A standard formulation is
\[
\min \sum_{i\in N}\sum_{j\in N} d_{ij}x_{ij}
\]
subject to
\[
\sum_{j\in N} x_{ij} = 1 \quad \forall i\in N,
\]
\[
x_{ij} \le y_j \quad \forall i,j\in N,
\]
\[
\sum_{j\in N} y_j = p,
\]
\[
x_{ij}, y_j \in \{0,1\} \quad \forall i,j\in N.
\]
The objective minimizes the total distance between demand points and their assigned open facilities.

\paragraph{$p$-Center Problem.}
The $p$-center problem seeks to open $p$ facilities so as to minimize the maximum distance from any demand point to its nearest open facility. Using the same assignment variables $x_{ij}\in\{0,1\}$ and an auxiliary variable $R$ denoting the service radius, one formulation is
\[
\min R
\]
subject to
\[
\sum_{j\in N} x_{ij} = 1 \quad \forall i\in N,
\]
\[
x_{ij} \le y_j \quad \forall i,j\in N,
\]
\[
\sum_{j\in N} d_{ij}x_{ij} \le R \quad \forall i\in N,
\]
\[
\sum_{j\in N} y_j = p,
\]
\[
x_{ij}, y_j \in \{0,1\} \quad \forall i,j\in N.
\]
The objective minimizes the worst-case service distance over all demand points.

\paragraph{$p$-Cover Problem.}
The $p$-cover problem considered here is the maximal covering location problem. Each demand point $i\in N$ has weight $w_i\ge 0$, and a demand point is considered covered if at least one open facility lies within a prespecified covering radius $\rho>0$. Let $z_i\in\{0,1\}$ indicate whether demand point $i$ is covered. Then a standard formulation is
\[
\max \sum_{i\in N} w_i z_i
\]
subject to
\[
z_i \le \sum_{j\in N:\, d_{ij}\le \rho} y_j
\quad \forall i\in N,
\]
\[
\sum_{j\in N} y_j = p,
\]
\[
z_i, y_j \in \{0,1\} \quad \forall i,j\in N.
\]
The objective maximizes the total covered demand under a fixed number of facilities and a fixed service radius.

\paragraph{$p$-Dispersion Problem.}
The $p$-dispersion problem seeks to select $p$ locations that are as far apart from one another as possible. In the max-min formulation, the objective is to maximize the minimum pairwise distance among the selected locations. Let $D$ denote the minimum separation distance. A standard formulation is
\[
\max D
\]
subject to
\[
D \le d_{ij} + M(2-y_i-y_j)
\quad \forall i,j\in N,\; i<j,
\]
\[
\sum_{j\in N} y_j = p,
\]
\[
y_j \in \{0,1\} \quad \forall j\in N,
\]
where $M$ is a sufficiently large constant. Whenever both locations $i$ and $j$ are selected, the constraint reduces to $D\le d_{ij}$, so maximizing $D$ forces the chosen facilities to be well separated.

\appsection{Benchmark Algorithms}
\label{app:algorithms}

\appsubsection{Constructive Heuristic}

We employ constructive heuristics for five combinatorial optimization tasks: TSP, CVRP, OP, JSSP, and QAP. A constructive heuristic builds a solution incrementally from an empty or partial state by repeatedly selecting the next feasible decision according to a problem-specific rule. In our framework, the LLM is responsible for generating this decision rule, while the overall construction procedure remains fixed. This design allows us to evaluate the quality of the generated heuristic independently of other implementation details of the solver.

For routing problems, the generated rule determines the next node to visit from the current partial route. For JSSP, it determines the next operation to schedule from the set of ready operations. For QAP, it determines the next facility-location pair to assign given the current partial assignment. The corresponding Python function signatures are listed below.

Function signature for TSP:
\begin{python}
def select_next_city(
    current: int,
    start: int,
    unvisited: set,
    dist_mat: np.ndarray,
) -> int
\end{python}

Function signature for CVRP:
\begin{python}
def select_next_customer(
    current: int,
    depot: int,
    feasible_customers: list[int],
    dist_mat: np.ndarray,
    demands: np.ndarray,
    remaining_capacity: float,
    vehicle_capacity: float,
) -> int
\end{python}

Function signature for OP:
\begin{python}
def select_next_node(
    current: int,
    depot: int,
    feasible_nodes: list[int],
    dist_mat: np.ndarray,
    prizes: np.ndarray,
    remaining_budget: float,
) -> int
\end{python}

Function signature for JSSP:
\begin{python}
def select_next_operation(
    ready_operations: list[tuple[int, int]],
    processing_times: np.ndarray,
    machine_assignments: np.ndarray,
    machine_available: np.ndarray,
    job_available: np.ndarray,
) -> tuple[int, int]
\end{python}

Function signature for QAP:
\begin{python}
def select_next_assignment(
    unassigned_facilities: list[int],
    unassigned_locations: list[int],
    flow_mat: np.ndarray,
    dist_mat: np.ndarray,
    current_assignment: dict,
) -> tuple[int, int]
\end{python}

\begin{table}[t]
\centering
\caption{Training/testing data generation for constructive TSP, CVRP, OP, JSSP, QAP.}
\label{tab:const_data_gen}
\setlength{\tabcolsep}{4pt}
\renewcommand{\arraystretch}{1.2}
\resizebox{\linewidth}{!}{%
\begin{tabular}{|l|C{2.3cm}|C{1.5cm}|p{9.8cm}|}
\hline
\textbf{Problem} & \textbf{Training set} & \textbf{Test set} & \textbf{Data generation} \\
\hline
\textbf{TSP} &
$5$ instances, $50$ cities &
None &
City coordinates are sampled independently from the uniform distribution on $[0,1)^2$. \\
\hline
\textbf{CVRP} &
$5$ instances, $50$ customers, capacity $50$ &
None &
Node coordinates are sampled independently from $[0,1)^2$, with the depot fixed at the center $(0.5,0.5)$. Customer demands are drawn uniformly from the integers $[1,14]$, while the depot demand is set to $0$. \\
\hline
\textbf{OP} &
$5$ instances, $80$ nodes &
None &
Node coordinates are sampled independently from $[0,1)^2$, with the depot fixed at $(0.5,0.5)$. Node prizes are drawn uniformly from the integers $[1,29]$, with depot prize $0$. The travel budget is set to $35\%$ of the greedy nearest-neighbor tour length. \\
\hline
\textbf{JSSP} &
$5$ instances, $15$ jobs, $10$ machines &
None &
Processing times are drawn uniformly from the integers $[1,99]$. For each job, the machine order is generated as a random permutation of all $10$ machines, so each job visits every machine exactly once. \\
\hline
\textbf{QAP} &
$5$ instances, size $15 \times 15$ &
None &
Flow matrices are generated from integer entries in $[0,9]$, then symmetrized with zero diagonal. Distance matrices are computed from Euclidean distances between randomly sampled points in $[0,1)^2$. \\
\hline
\end{tabular}%
}
\end{table}

\begin{table}[t]
\centering
\caption{Training/testing data generation for constructive CTS, FTR, OAS, WPF.}
\label{tab:sco_const_data_gen}
\setlength{\tabcolsep}{4pt}
\renewcommand{\arraystretch}{1.2}
\resizebox{\linewidth}{!}{%
\begin{tabular}{|l|C{2cm}|C{2.3cm}|p{9.4cm}|}
\hline
\textbf{Problem} & \textbf{Training set} & \textbf{Test set} & \textbf{Data generation} \\
\hline
\textbf{CTS} &
$5$ instances, $100$ talks &
$100$ instances, $100$ talks &
Talk qualities are sampled from a log-normal distribution. Topic embeddings are generated from $5$ latent cluster centers in an $8$-dimensional space: cluster centers are drawn from a standard normal distribution and normalized, each talk is assigned to a random cluster, and its embedding is obtained by adding Gaussian noise followed by $\ell_2$ normalization. \\
\hline
\textbf{FTR} &
$5$ instances, $100$ locations &
$100$ instances, $100$ locations &
Location coordinates are sampled independently from the uniform distribution on $[0,1)^2$. Location popularities are sampled from a log-normal distribution. \\
\hline
\textbf{OAS} &
$5$ instances, $100$ ads &
$100$ instances, $100$ ads &
Base ad values are sampled from a log-normal distribution. Fatigue rates are generated from a Beta$(2,5)$ distribution and then adjusted by a standardized correlation term derived from the base values, inducing a mild positive dependence between ad value and fatigue before clipping to $[0.02, 0.9]$. \\
\hline
\textbf{WPF} &
$5$ instances, $100$ orders &
$100$ instances, $100$ orders &
Base picking times are sampled uniformly from $[0.5, 3.0]$. Order values are generated from an affine function of base picking time plus absolute Gaussian noise, then multiplied by an independent uniform scaling factor in $[0.6, 1.6]$, and finally clipped below at $0.1$. \\
\hline
\end{tabular}%
}
\end{table}

We further consider four constructive decision-making tasks in SCO, where the LLM again generates the local rule used to extend a partial solution or action sequence under a fixed evaluation procedure. The corresponding Python function signatures are listed below.

Function signature for CTS:
\begin{python}
def select_next_talk(
    available_talks: list[int],
    qualities: np.ndarray,
    topics: np.ndarray,
    previous_topic: np.ndarray,
) -> int
\end{python}

Function signature for FTR:
\begin{python}
def select_next_stop(
    locations: np.ndarray,
    popularities: np.ndarray,
    steps_since_last_visit: np.ndarray,
    last_location: int,
) -> int
\end{python}

Function signature for OAS:
\begin{python}
def select_next_ad(
    base_values: np.ndarray,
    fatigue_rates: np.ndarray,
    fatigue_levels: np.ndarray,
    remaining_slots: int,
) -> int
\end{python}

Function signature for WPF:
\begin{python}
def select_next_order(
    available_orders: list[int],
    values: np.ndarray,
    base_times: np.ndarray,
    effective_budget: float,
) -> int
\end{python}

\appsubsection{Ant Colony Optimization}
\paragraph{Ant Colony Optimization (ACO).}
We employ ACO for TSP, CVRP, LVRP, and OVRP. ACO is a population-based metaheuristic in which a set of ants repeatedly constructs solutions by sampling feasible solution components according to pheromone trails and heuristic desirability, after which the pheromone values are updated using the constructed solutions. Let
\[
\mathbf{D} = (d_{ij}) \in \mathbb{R}^{n\times n}
\]
denote the distance matrix,
\[
\mathbf{T} = (\tau_{ij}) \in \mathbb{R}_{\ge 0}^{n\times n}
\]
the pheromone matrix, and
\[
\mathbf{H} = (\eta_{ij}) \in \mathbb{R}_{\ge 0}^{n\times n}
\]
the heuristic desirability matrix. In a standard ACO transition rule, the probability of selecting edge $(i,j)$ is proportional to
\[
\tau_{ij}^{\alpha}\eta_{ij}^{\beta},
\]
where $\alpha$ and $\beta$ control the relative influence of pheromone and heuristic information. Thus, pheromone trails are learned online while solving each instance, whereas heuristic measures are typically predefined and remain fixed during the search.

Our framework follows this general ACO setup, but replaces manually designed heuristics with an LLM-generated heuristic function that maps the instance data, represented primarily through $\mathbf{D}$ together with any task-specific side information, to a heuristic desirability matrix:
\[
\mathbf{H} = h(\mathbf{D}, \ldots).
\]
The generated matrix $\mathbf{H}$ therefore determines how promising each candidate move or edge is before pheromone adaptation takes place, while the remainder of the ACO procedure remains unchanged.

This design is conceptually related to DeepACO~\cite{ye2023deepaco}, which improves existing ACO algorithms by learning stronger heuristic measures across instances and using them to guide solution construction. However, unlike DeepACO, our method does not rely on a trained neural heuristic learner or its additional enhancement components; rather, it uses the LLM directly to generate the heuristic function for each problem.

A generic ACO procedure can be summarized as follows:
\begin{algorithm}[h]
\caption{Generic ant colony optimization procedure}
\label{alg:aco_generic}
\small
\begin{algorithmic}
\Require Problem instance, pheromone matrix initialization, heuristic matrix generator, number of ants, number of iterations
\Ensure Best solution found
\State Compute heuristic desirability matrix
\State Initialize pheromone matrix
\For{each iteration}
    \State Combine pheromone and heuristic into transition scores
    \For{each ant}
        \State Construct a feasible solution by sampling from the transition scores
        \State Evaluate the solution cost
    \EndFor
    \State Evaporate pheromone
    \State Deposit pheromone using the constructed solutions
\EndFor
\State \Return best solution found
\end{algorithmic}
\end{algorithm}

In our benchmark, the LLM-generated heuristic function is used in the step that computes the heuristic desirability matrix. Its role is to transform the raw instance data, such as distances, demands, or route constraints, into edge-level scores that guide the stochastic construction process of the ants.

\begin{table}[t]
\centering
\caption{Training/testing data generation for ACO TSP, CVRP, OVRP, LVRP.}
\label{tab:aco_data_gen}
\setlength{\tabcolsep}{4pt}
\renewcommand{\arraystretch}{1.2}
\resizebox{\linewidth}{!}{%
\begin{tabular}{|l|C{3.2cm}|C{3.5cm}|p{9.4cm}|}
\hline
\textbf{Problem} & \textbf{Training set} & \textbf{Test set} & \textbf{Data generation} \\
\hline
\textbf{TSP} &
$5$ instances, $50$ cities &
$100$ instances for each $n \in \{50,100,200,500\}$ &
City coordinates are sampled independently from the uniform distribution on $[0,1)^2$. The training and test sets follow the same distribution, with the test set evaluated at four problem sizes. \\
\hline
\textbf{CVRP} &
$5$ instances, $50$ customers, capacity $50$ &
$100$ instances for each $n \in \{50,100,200,500\}$ &
Node coordinates are sampled independently from $[0,1)^2$, with the depot fixed at the center $(0.5,0.5)$. Customer demands are drawn uniformly from the integers $[1,14]$, while the depot demand is set to $0$. Test-time vehicle capacities are size-dependent, namely $50$, $80$, $120$, and $200$ for $n=50,100,200,500$, respectively. \\
\hline
\textbf{OVRP} &
$5$ instances, $50$ customers, capacity $50$ &
$100$ instances for each $n \in \{50,100,200,500\}$ &
The instance distribution is identical to that of CVRP: node coordinates are sampled uniformly from $[0,1)^2$, the depot is fixed at $(0.5,0.5)$, and customer demands are drawn uniformly from the integers $[1,14]$. Test-time capacities again depend on problem size: $50$, $80$, $120$, and $200$. \\
\hline
\textbf{LVRP} &
$5$ instances, $50$ customers, capacity $50$ &
$100$ instances for each $n \in \{50,100,200,500\}$ &
Instances are generated as in CVRP, with node coordinates sampled from $[0,1)^2$, the depot fixed at $(0.5,0.5)$, and customer demands drawn uniformly from the integers $[1,14]$. In addition, each instance is assigned a route-duration limit equal to $40\%$ of the greedy nearest-neighbor tour length. Test-time vehicle capacities are $50$, $80$, $120$, and $200$ for $n=50,100,200,500$, respectively. \\
\hline
\end{tabular}%
}
\end{table}

Function signature for TSP:
\begin{python}
def compute_heuristic_matrix(
    dist_mat: np.ndarray,
) -> np.ndarray
\end{python}

Function signature for CVRP:
\begin{python}
def compute_heuristic_matrix(
    dist_mat: np.ndarray,
    demands: np.ndarray,
    vehicle_capacity: float,
) -> np.ndarray
\end{python}

Function signature for OVRP:
\begin{python}
def compute_heuristic_matrix(
    dist_mat: np.ndarray,
    demands: np.ndarray,
    vehicle_capacity: float,
) -> np.ndarray
\end{python}

Function signature for LVRP:
\begin{python}
def compute_heuristic_matrix(
    dist_mat: np.ndarray,
    demands: np.ndarray,
    vehicle_capacity: float,
    max_duration: float,
) -> np.ndarray
\end{python}

\appsubsection{Neural Combinatorial Optimization}

\paragraph{Neural Combinatorial Optimization (NCO).}
NCO aims to solve combinatorial optimization problems using neural networks that either construct solutions directly or guide downstream search procedures. In this work, we consider two representative NCO solvers for Euclidean TSP, namely POMO~\cite{kwon2020pomo} and DIFUSCO~\cite{sun2023difusco}, and study how the LLM can enhance them through instance-specific edge biases. Let $V=\{1,\dots,n\}$ be the set of cities, let $x_i \in \mathbb{R}^2$ denote the coordinate of city $i$, and let
\[
\mathbf{D} = (d_{ij}) \in \mathbb{R}^{n\times n},
\qquad
d_{ij} = \|x_i-x_j\|_2,
\]
be the distance matrix. In both cases, the pretrained neural solver is kept fixed, while the LLM generates an auxiliary function that maps $\mathbf{D}$ to an edge-bias matrix
\[
\mathbf{B} = h(\mathbf{D}) \in \mathbb{R}^{n\times n}.
\]
This matrix is then injected into the inference procedure of the neural solver.

POMO is an autoregressive RL-based constructive solver that exploits the symmetry of routing solutions through multiple parallel rollouts from different starting nodes. If $\tau^{(k)}=(a_1^{(k)},\dots,a_n^{(k)})$ denotes the $k$-th rollout, then POMO optimizes a policy $\pi_\theta$ using a shared-baseline REINFORCE objective of the form
\[
\nabla_\theta J(\theta)
\approx
\frac{1}{K}\sum_{k=1}^K
\bigl(R(\tau^{(k)})-\bar{R}\bigr)
\nabla_\theta \log p_\theta(\tau^{(k)} \mid \mathbf{D}),
\qquad
\bar{R}=\frac{1}{K}\sum_{k=1}^K R(\tau^{(k)}),
\]
where $K$ is the number of parallel rollouts. In the original node-based decoder, if the current city at step $t$ is $i$, the model produces logits $z_t(i,j)$ over feasible next cities $j$. Our method augments these logits with the LLM-generated edge bias:
\[
\tilde{z}_t(i,j) = z_t(i,j) + B_{ij}.
\]
The transition probabilities are then computed from the modified logits over the feasible set $\mathcal{F}_t$,
\[
\pi(a_t=j \mid a_{<t}, \mathbf{D})
=
\frac{\exp(\tilde{z}_t(i,j))}
{\sum_{\ell \in \mathcal{F}_t} \exp(\tilde{z}_t(i,\ell))}.
\]
Thus, the LLM does not replace the pretrained POMO policy; instead, it reshapes the decoder's preference over outgoing edges in an instance-specific manner. The final solution is chosen as the best tour among the parallel POMO rollouts.

DIFUSCO, in contrast, is a non-autoregressive graph-based diffusion solver. For TSP, a tour is represented by a binary symmetric edge-indicator matrix
\[
\mathbf{X} \in \{0,1\}^{n\times n},
\]
where $X_{ij}=1$ indicates that edge $(i,j)$ belongs to the tour. DIFUSCO learns a denoising model that gradually recovers a clean tour-edge structure from noisy edge variables, and at inference time it outputs an edge heatmap
\[
\mathbf{H}_\theta(\mathbf{D}) \in [0,1]^{n\times n},
\]
whose entries indicate how likely each edge is to belong to a high-quality tour. In our framework, we keep the pretrained diffusion model fixed and combine its heatmap with the LLM-generated bias:
\[
\mathbf{S} = \mathbf{H}_\theta(\mathbf{D}) + \mathbf{B}.
\]
The resulting score matrix $\mathbf{S}$ is then decoded into a feasible tour by a greedy edge-based decoder, followed by local search refinement. Therefore, the role of the LLM is to provide an additional structured prior over edges, complementing the diffusion model's learned heatmap rather than replacing it.

In summary, both enhancements use the same high-level idea of generating an instance-specific edge prior, but they intervene at different stages of inference: in POMO, the bias modifies autoregressive transition logits during sequential decoding, whereas in DIFUSCO, it modifies the final edge heatmap before decoding and local search.

\begin{table}[t]
\centering
\caption{Pretraining settings for NCO TSP node-based (POMO) and edge-based (DIFUSCO) solvers.}
\label{tab:nco_tsp_pretrain}
\setlength{\tabcolsep}{4pt}
\renewcommand{\arraystretch}{1.2}
\resizebox{\linewidth}{!}{%
\begin{tabular}{|l|C{2.5cm}|C{3.8cm}|C{2cm}|C{2.5cm}|p{8.8cm}|}
\hline
\textbf{Method} & \textbf{Representation} & \textbf{Pretraining data} & \textbf{Problem size} & \textbf{Training regime} & \textbf{Main settings} \\
\hline
\textbf{POMO} &
Node-based &
Online random Euclidean TSP instances sampled uniformly from $[0,1)^2$ &
$100$ &
RL pretraining &
Transformer-style POMO model with embedding dimension $128$, $6$ encoder layers, $8$ attention heads, QKV dimension $16$, feed-forward dimension $512$, and logit clipping $10$. Trained with Adam (learning rate $10^{-4}$, weight decay $10^{-6}$), batch size $128$, $10{,}000$ episodes per epoch, and $1000$ epochs; checkpoints are saved every $100$ epochs. \\
\hline
\textbf{DIFUSCO} &
Edge-based &
Online random Euclidean TSP instances sampled uniformly from $[0,1)^2$; each instance is solved optimally with \texttt{elkai} to obtain target tour edges &
$100$ &
Supervised diffusion pretraining &
Discrete diffusion model with an anisotropic GNN backbone using $8$ layers, hidden dimension $128$, and diffusion horizon $T=1000$. Trained with Adam (learning rate $2\times10^{-4}$, weight decay $10^{-4}$), gradient clipping $1.0$, batch size $16$, $10{,}000$ episodes per epoch, and $1000$ epochs; cosine learning-rate decay and checkpointing every $10$ epochs are used. \\
\hline
\end{tabular}%
}
\end{table}

\begin{table}[t]
\centering
\caption{Training/testing data generation for NCO TSP multi-distribution.}
\label{tab:nco_tsp_multidist}
\setlength{\tabcolsep}{4pt}
\renewcommand{\arraystretch}{1.2}
\resizebox{\linewidth}{!}{%
\begin{tabular}{|l|C{2.8cm}|C{3.2cm}|p{8.7cm}|}
\hline
\textbf{Distribution} & \textbf{Training set} & \textbf{Test set} & \textbf{Data generation} \\
\hline
\textbf{Uniform} &
$5$ instances, $100$ cities &
$100$ instances for each $n \in \{100,200,500\}$ &
City coordinates are sampled independently from the uniform distribution on $[0,1)^2$. \\
\hline
\textbf{Clustered} &
$5$ instances, $100$ cities &
$100$ instances for each $n \in \{100,200,500\}$ &
Cities are generated from $6$ Gaussian clusters whose centers are arranged on a circle of radius $0.32$ around $(0.5,0.5)$. Cluster perturbations use standard deviation $0.055$, and coordinates are clipped to $[0,1]^2$. \\
\hline
\textbf{Diagonal} &
$5$ instances, $100$ cities &
$100$ instances for each $n \in \{100,200,500\}$ &
Cities are sampled around the main diagonal of the unit square. A latent position is drawn uniformly on the diagonal, then perturbed along the perpendicular direction by Gaussian noise with standard deviation $0.13$; points outside $[0,1]^2$ are rejected and resampled. \\
\hline
\textbf{Barbell} &
$5$ instances, $100$ cities &
$100$ instances for each $n \in \{100,200,500\}$ &
Instances contain two dense endpoint clusters and a narrow bridge. Specifically, $40\%$ of cities are sampled around $(0.22,0.5)$, $40\%$ around $(0.78,0.5)$, both with Gaussian spread $(0.07, 0.11)$, and the remaining $20\%$ are placed along a bridge between $x=0.34$ and $x=0.66$ with small Gaussian perturbations. Coordinates are clipped to $[0,1]^2$. \\
\hline
\end{tabular}%
}
\end{table}

Function signature for POMO-enhanced TSP:
\begin{python}
def compute_edge_bias(
    distance_matrix: torch.Tensor,
) -> torch.Tensor
\end{python}

Function signature for DIFUSCO-enhanced TSP:
\begin{python}
def compute_edge_bias(
    distance_matrix: torch.Tensor,
) -> torch.Tensor
\end{python}

\appsubsection{Guided Local Search}

\paragraph{Guided Local Search (GLS).}
We employ GLS for TSP. GLS is a metaheuristic that augments local search with adaptive penalties in order to escape poor local optima. Let
\[
\mathbf{D} = (d_{ij}) \in \mathbb{R}^{n\times n}
\]
be the distance matrix, and let a tour be represented by a permutation $\pi$ of the cities. Standard local search seeks to minimize the tour length
\[
f(\pi) = \sum_{t=1}^{n} d_{\pi_t,\pi_{t+1}},
\]
with $\pi_{n+1}=\pi_1$. GLS introduces an edge-penalty matrix
\[
\mathbf{P} = (P_{ij}) \in \mathbb{R}_{\ge 0}^{n\times n},
\]
and instead performs search on the augmented objective
\[
f_{\mathrm{aug}}(\pi)
=
f(\pi)
+
\lambda \sum_{(i,j)\in E(\pi)} P_{ij},
\]
where $E(\pi)$ is the set of tour edges and $\lambda > 0$ is a penalty weight. During the search, selected edges in the current local optimum are penalized, thereby modifying the landscape and encouraging the local search procedure to move toward different regions of the solution space.

In our benchmark, the LLM generates a nonnegative guide matrix
\[
\mathbf{G} = h(\mathbf{D}) \in \mathbb{R}_{\ge 0}^{n\times n},
\]
which determines how strongly each edge should be prioritized for penalization. Given the current tour, the utility of an edge $(i,j)$ is computed as
\[
u_{ij} = \frac{G_{ij}}{1 + P_{ij}}.
\]
Edges with large utility are penalized more aggressively. Hence, the generated guide matrix does not directly define the tour; rather, it controls which structures in the current local optimum are most likely to be broken during perturbation.

A generic GLS procedure can be summarized as follows:
\begin{algorithm}[h]
\caption{Generic guided local search procedure}
\label{alg:gls_generic}
\small
\begin{algorithmic}
\Require Problem instance, penalty-guide generator, local search routine, perturbation budget, iteration limit
\Ensure Best solution found
\State Compute penalty-guide matrix
\State Initialize penalty matrix
\State Construct an initial solution
\State Improve the solution by local search
\For{each iteration}
    \State Evaluate edge utilities using the guide and current penalties
    \State Penalize selected edges in the current solution
    \State Update the augmented objective
    \State Re-optimize the perturbed solution by local search
    \State Update the best solution found
\EndFor
\State \Return best solution found
\end{algorithmic}
\end{algorithm}

In our setting, the LLM-generated guide matrix is computed once per instance and then kept fixed throughout the GLS run. Its role is to provide an instance-specific prior over which edges should be discouraged when the search becomes trapped in a local optimum.

Function signature for TSP:
\begin{python}
def compute_penalty_guide(
    dist_mat: np.ndarray,
) -> np.ndarray
\end{python}

\begin{table}[t]
\centering
\caption{Training and test instances for GLS TSP.}
\label{tab:gls_tsp_data}
\setlength{\tabcolsep}{4pt}
\renewcommand{\arraystretch}{1.2}
\resizebox{\linewidth}{!}{%
\begin{tabular}{|l|C{2.4cm}|C{1.8cm}|p{12.8cm}|}
\hline
\textbf{Split} & \textbf{Instance range} & \textbf{\# Instances} & \textbf{Description} \\
\hline
\textbf{Training} & $n=200$ & $10$ & Random Euclidean TSP instances with city coordinates sampled independently from the uniform distribution on $[0,1)^2$. \\
\hline
\textbf{Test (TSPLib)} & $100 \leq n < 200$ & $25$ & \texttt{kroA100}, \texttt{kroB100}, \texttt{kroC100}, \texttt{kroD100}, \texttt{kroE100}, \texttt{rd100}, \texttt{eil101}, \texttt{lin105}, \texttt{pr107}, \texttt{gr120}, \texttt{pr124}, \texttt{bier127}, \texttt{ch130}, \texttt{pr136}, \texttt{gr137}, \texttt{pr144}, \texttt{ch150}, \texttt{kroA150}, \texttt{kroB150}, \texttt{pr152}, \texttt{u159}, \texttt{si175}, \texttt{brg180}, \texttt{rat195}, \texttt{d198}. \\
\hline
\textbf{Test (TSPLib)} & $200 \leq n < 500$ & $19$ & \texttt{kroA200}, \texttt{kroB200}, \texttt{gr202}, \texttt{ts225}, \texttt{tsp225}, \texttt{pr226}, \texttt{gr229}, \texttt{gil262}, \texttt{pr264}, \texttt{a280}, \texttt{pr299}, \texttt{lin318}, \texttt{linhp318}, \texttt{rd400}, \texttt{fl417}, \texttt{gr431}, \texttt{pr439}, \texttt{pcb442}, \texttt{d493}. \\
\hline
\textbf{Test (TSPLib)} & $500 \leq n < 1000$ & $11$ & \texttt{att532}, \texttt{ali535}, \texttt{si535}, \texttt{pa561}, \texttt{u574}, \texttt{rat575}, \texttt{p654}, \texttt{d657}, \texttt{gr666}, \texttt{u724}, \texttt{rat783}. \\
\hline
\end{tabular}%
}
\end{table}

\appsubsection{Greedy Randomized Adaptive Search Procedure}

\paragraph{Greedy Randomized Adaptive Search Procedure (GRASP).}
We employ GRASP for four DLP tasks: $p$-median, $p$-center, $p$-cover, and $p$-dispersion. GRASP is a multi-start metaheuristic that alternates between a greedy-randomized construction phase and a local search phase. Let
\[
\mathbf{D} = (d_{ij}) \in \mathbb{R}^{n\times n}
\]
denote the distance matrix, and let
\[
S \subseteq N, \qquad |S|=p,
\]
denote the set of selected facilities. In each GRASP iteration, a candidate solution is first constructed incrementally by selecting elements from a restricted candidate list (RCL), and is then refined by local search, typically through 1-swap moves. The best solution over all iterations is returned.

In our benchmark, the LLM generates a static guidance function that shapes the construction phase of GRASP. Depending on the problem, this guidance takes the form of either a node-score vector
\[
\mathbf{s} = h(\mathbf{D},\ldots) \in \mathbb{R}^{n}
\]
or a pairwise guide matrix
\[
\mathbf{G} = h(\mathbf{D}) \in \mathbb{R}^{n\times n}.
\]
For $p$-median, $p$-center, and $p$-cover, the generated node scores provide a prior ranking over candidate facilities. For $p$-dispersion, the generated matrix provides a prior over desirable pairs of selected locations. During construction, the solver combines this static guidance with dynamic information from the current partial solution to build an RCL, from which one candidate is sampled at random.

A generic GRASP procedure can be summarized as follows:
\begin{algorithm}[h]
\caption{Generic GRASP procedure}
\label{alg:grasp_generic}
\small
\begin{algorithmic}
\Require Problem instance, guidance generator, number of iterations, local search routine
\Ensure Best solution found
\State Compute static guidance
\For{each iteration}
    \State Initialize an empty partial solution
    \While{the solution is incomplete}
        \State Build a restricted candidate list using greedy scores
        \State Randomly select one candidate from the restricted candidate list
        \State Extend the partial solution
    \EndWhile
    \State Improve the constructed solution by local search
    \State Update the best solution found
\EndFor
\State \Return best solution found
\end{algorithmic}
\end{algorithm}

In our setting, the LLM-generated guidance is computed once per instance and then kept fixed throughout the GRASP run. Its role is not to replace the adaptive behavior of GRASP, but to provide an instance-specific prior that biases the greedy-randomized construction toward more promising facilities or facility pairs.

Function signature for $p$-Center:
\begin{python}
def compute_node_scores(
    dist_mat: np.ndarray,
) -> np.ndarray
\end{python}

Function signature for $p$-Cover:
\begin{python}
def compute_node_scores(
    dist_mat: np.ndarray,
    demands: np.ndarray,
    radius: float,
) -> np.ndarray
\end{python}

Function signature for $p$-Dispersion:
\begin{python}
def compute_guide_matrix(
    dist_mat: np.ndarray,
) -> np.ndarray
\end{python}

Function signature for $p$-Median:
\begin{python}
def compute_node_scores(
    dist_mat: np.ndarray,
) -> np.ndarray
\end{python}

\begin{table}[t]
\centering
\caption{Training/testing data generation for GRASP $p$-Center, $p$-Cover, $p$-Dispersion, $p$-Median.}
\label{tab:grasp_p_data_gen}
\setlength{\tabcolsep}{4pt}
\renewcommand{\arraystretch}{1.2}
\resizebox{\linewidth}{!}{%
\begin{tabular}{|l|C{2.2cm}|C{3.2cm}|p{9.5cm}|}
\hline
\textbf{Problem} & \textbf{Training set} & \textbf{Test set} & \textbf{Data generation} \\
\hline
\textbf{$p$-Center} &
$5$ instances, $100$ nodes, $p=10$ &
$100$ instances for each $n \in \{100,200,500\}$ &
Node coordinates are sampled independently from the uniform distribution on $[0,1)^2$. The number of selected facilities is fixed to $p=10$ for training, and set to $p=\max(10,\lfloor n/20 \rfloor)$ for testing. \\
\hline
\textbf{$p$-Cover} &
$5$ instances, $100$ nodes, $p=10$ &
$100$ instances for each $n \in \{100,200,500\}$ &
Node coordinates are sampled independently from $[0,1)^2$. Node demands are sampled as $0.5 + U(0,1)$, yielding values in $[0.5,1.5)$. The coverage radius is set to $1.8/\sqrt{n}$. The number of selected facilities is fixed to $p=10$ for training, and set to $p=\max(10,\lfloor n/20 \rfloor)$ for testing. \\
\hline
\textbf{$p$-Dispersion} &
$5$ instances, $100$ nodes, $p=10$ &
$100$ instances for each $n \in \{100,200,500\}$ &
Node coordinates are sampled independently from the uniform distribution on $[0,1)^2$. The subset size is fixed to $p=10$ for training, and set to $p=\max(10,\lfloor n/10 \rfloor)$ for testing. \\
\hline
\textbf{$p$-Median} &
$5$ instances, $100$ nodes, $p=10$ &
$100$ instances for each $n \in \{100,200,500\}$ &
Node coordinates are sampled independently from the uniform distribution on $[0,1)^2$. The number of selected facilities is fixed to $p=10$ for training, and set to $p=\max(10,\lfloor n/20 \rfloor)$ for testing. \\
\hline
\end{tabular}%
}
\end{table}

\appsubsection{SR Algorithm}
For the SR benchmarks described above, the LLM is responsible only for proposing the symbolic structure of the target equation. After the LLM generates a functional form $f$, its numerical parameters $\boldsymbol{\theta}$ are fitted on the observed data using the BFGS algorithm, and the resulting expression is evaluated by the NMSE metric defined above. Hence, the search space consists of symbolic expressions, while constant estimation is handled by a fixed downstream optimizer.

Function signature for the damped nonlinear oscillator:
\begin{python}
def equation(
    x: np.ndarray,
    v: np.ndarray,
    params: np.ndarray,
) -> np.ndarray
\end{python}

Function signature for the driven damped nonlinear oscillator:
\begin{python}
def equation(
    t: np.ndarray,
    x: np.ndarray,
    v: np.ndarray,
    params: np.ndarray,
) -> np.ndarray
\end{python}

Function signature for bacterial growth rate:
\begin{python}
def equation(
    b: np.ndarray,
    s: np.ndarray,
    temp: np.ndarray,
    pH: np.ndarray,
    params: np.ndarray,
) -> np.ndarray
\end{python}

Function signature for the stress--strain relation:
\begin{python}
def equation(
    strain: np.ndarray,
    temp: np.ndarray,
    params: np.ndarray,
) -> np.ndarray
\end{python}

\begin{table}[t]
\centering
\caption{Evaluation settings for SR benchmarks.}
\label{tab:sr_eval_settings}
\setlength{\tabcolsep}{4pt}
\renewcommand{\arraystretch}{1.2}
\resizebox{\linewidth}{!}{%
\begin{tabular}{|l|p{4.4cm}|p{3.8cm}|C{1.5cm}|C{1.6cm}|C{1.8cm}|}
\hline
\textbf{Problem} & \textbf{Inputs} & \textbf{Prediction target} & \textbf{Train} & \textbf{Test-ID} & \textbf{Test-OOD} \\
\hline
\textbf{Bacterial Growth} &
$b$ (population density), $s$ (substrate), temperature, pH &
Growth rate $\mathrm{d}B/\mathrm{d}t$ &
7501 & 2501 & 15001 \\
\hline
\textbf{Oscillator-1} &
$x$ (position), $v$ (velocity) &
Acceleration $\mathrm{d}v/\mathrm{d}t$ &
10001 & 10001 & 10001 \\
\hline
\textbf{Oscillator-2} &
$t$ (time), $x$ (position), $v$ (velocity) &
Acceleration $\mathrm{d}v/\mathrm{d}t$ &
10001 & 10001 & 10001 \\
\hline
\textbf{Stress-Strain} &
Strain, temperature &
Stress &
2162 & 1443 & 739 \\
\hline
\end{tabular}%
}
\end{table}

\appsubsection{PE Algorithm}

For the PE benchmarks described above, we instantiate the search procedure as sequential offline Bayesian optimization over the given candidate pool. Let $\mathcal{C}$, $\mathcal{D}_0$, and $B$ denote the benchmark-specific candidate set, initial support set, and query budget. At round $t=1,\dots,B$, the algorithm selects one previously untested sequence
\[
x_t \in \mathcal{C} \setminus \{x_1,\dots,x_{t-1}\},
\]
observes its fitness $y_t$, and updates the observed set
\[
\mathcal{D}_t = \mathcal{D}_{t-1} \cup \{(x_t,y_t)\}.
\]

In standard BO, one would construct an acquisition function from a surrogate posterior mean and uncertainty estimate. Our framework follows the same principle, but delegates the design of the acquisition rule to the LLM. More precisely, for each candidate $x \in \mathcal{C}\setminus \mathcal{D}_{t-1}$, a lightweight surrogate built from the observed set produces a predicted mean $\mu_t(x)$ and uncertainty estimate $\sigma_t(x)$. Let
\[
y_t^\star = \max_{(x,y)\in \mathcal{D}_{t-1}} y
\]
denote the current incumbent. We then define the standardized improvement score
\[
z_t(x) = \frac{\mu_t(x)-y_t^\star}{\max(\sigma_t(x),\varepsilon)},
\]
where $\varepsilon>0$ is a small constant for numerical stability. In addition, each candidate is assigned a plausibility score $\rho_t(x)$, which measures its compatibility with the support set, and a progress variable
\[
p_t = \frac{t-1}{B},
\]
which indicates how much of the budget has already been consumed.

The LLM generates an acquisition function
\[
a:\mathbb{R}^4 \to \mathbb{R},
\qquad
(z,\sigma,\rho,p) \mapsto a(z,\sigma,\rho,p),
\]
and candidate selection is performed by
\[
x_t
=
\arg\max_{x \in \mathcal{C}\setminus \mathcal{D}_{t-1}}
a\bigl(z_t(x),\sigma_t(x),\rho_t(x),p_t\bigr).
\]
Hence, the LLM does not directly predict protein fitness. Instead, it specifies how exploitation, uncertainty, biological plausibility, and search progress should be balanced inside the BO loop.

The surrogate signals are constructed differently for different PE settings, but the BO interface remains the same. For fixed-length sequence tasks such as alpha amylase and IRED, similarity is computed directly in sequence space, and the surrogate statistics are derived from weighted nearest neighbors among previously observed sequences. For variable-length tasks such as hydrophobic-core design and rhodopsin, sequences are first mapped to feature vectors before the same type of nearest-neighbor surrogate is applied. In both cases, the LLM receives only the four candidate-wise signals $(z,\sigma,\rho,p)$ and produces the acquisition scores used for sequential selection.

Performance is evaluated after the full budget is exhausted. Let $\{y_1,\dots,y_B\}$ be the discovered fitness values, and let $K$ denote the number of top discoveries retained for evaluation. The achieved score is
\[
\mathrm{TopKAvg}
=
\frac{1}{K}
\sum_{k=1}^{K} y_{(k)},
\]
where $y_{(1)} \ge \cdots \ge y_{(K)}$ are the top-$K$ discovered values. In our implementation, $B=32$ and $K=8$, and results are reported through the percentage optimality gap to the oracle top-$8$ mean for the given benchmark split.

Shared function signature for all PE benchmarks:
\begin{python}
def acquisition(
    z: np.ndarray,
    sigma: np.ndarray,
    plausibility: np.ndarray,
    progress: np.ndarray,
) -> np.ndarray
\end{python}

An EI-style baseline can be written as
\[
a_{\mathrm{EI}}(z,\sigma,\rho,p)
=
\rho \, \sigma \, \phi(z)
+
\rho \, \sigma \, z \Phi(z),
\]
where $\phi$ and $\Phi$ denote the standard normal density and cumulative distribution functions, respectively. This is the usual expected-improvement expression rewritten in terms of the standardized improvement signal $z$ and the uncertainty scale $\sigma$, with the plausibility score $\rho$ acting as a multiplicative compatibility gate.

A UCB-style baseline can be written as
\[
a_{\mathrm{UCB}}(z,\sigma,\rho,p)
=
\rho \bigl(\mu(z,\sigma) + \beta(p)\sigma \bigr),
\qquad
\mu(z,\sigma)=f_t^{\star}+z\sigma,
\]
where $f_t^{\star}$ is the incumbent value at BO step $t$, and $\beta(p)$ is a progress-dependent exploration weight. Under the PE interface, this is equivalently
\[
a_{\mathrm{UCB}}(z,\sigma,\rho,p)
=
\rho \bigl(z + \beta(p)\bigr)\sigma,
\]
up to the additive constant $\rho f_t^{\star}$, which does not affect the ranking of candidates at a fixed BO step. In practice, $\beta(p)$ may be chosen to decrease with progress, so that exploration is emphasized early in the budget and gradually reduced as the search proceeds.

\newpage

\appsection{Experimental Setup}
\label{app:setup}

\appsubsection{Implementation Details}

\begin{tcolorbox}[
  enhanced,
  colback=white,
  colframe=gray!55,
  boxrule=0.6pt,
  arc=5pt,
  left=8pt,
  right=8pt,
  top=6pt,
  bottom=6pt
]
\textbf{We plan to publicly release the source code, experimental settings, and all datasets used in this work.} They are omitted from the preprint to avoid disclosing identifying information and to respect confidentiality constraints associated with the current manuscript.
\end{tcolorbox}

\paragraph{\emph{Top-Down Population-Based Search}.}
We implement the population-based top-down variant by modifying the ReEvo~\citep{ye2024reevo} scaffold while keeping the outer evolutionary protocol fixed. The bottom-up side follows ReEvo's code-first workflow, whereas our top-down side replaces the primary population state from executable code $\theta \in \Theta$ to knowledge $K \in \mathcal{K}$, consistent with the methodology in Section~\ref{main:method}. We also write $A=\alpha(\theta)$ for the heuristic artifact observed by the LLM; in our implementation, $A$ is the generated source-code artifact associated with the executable heuristic $\theta$. In both paradigms, empirical fitness is obtained only by executing code and computing $\widehat{L}_{\mathcal{D}}(\theta)$.

\newcommand{\BUCX}{Q_{\phi,t}^{\Theta}(\cdot\mid H^{t-1},A_i^-,A_i^+,r_i^t)}
\newcommand{\BUMT}{Q_{\phi,t}^{\Theta}(\cdot\mid H^{t,\mathrm{CX}},\theta_t^\star,R_t)}
\newcommand{\TDCX}{Q_{\phi,t}^{\mathcal{K},\Theta}(\cdot\mid H^{t-1},K_i^-,K_i^+,r_i^t)}
\newcommand{\TDMT}{Q_{\phi,t}^{\mathcal{K},\Theta}(\cdot\mid H^{t,\mathrm{CX}},K_t^\star,R_t)}

\begin{figure}[h]
\centering

\begin{minipage}[t]{0.49\linewidth}
\captionsetup{type=algorithm}
\rule{\linewidth}{1pt}\vspace{-5pt}
\captionof{algorithm}{Bottom-up population-based AHD}
\vspace{-5pt}\rule{\linewidth}{1pt}\vspace{5pt}
\label{alg:ea_bu_impl}
\resizebox{\linewidth}{!}{%
\begin{minipage}{1.12\linewidth}
\begin{algorithmic}
\State $\mathcal{P}_{\Theta}^{0}\gets\textsc{Init}_{\Theta}(M_0,H^0)$
\State $\mathcal{P}_{\Theta}^{0}\gets\textsc{Eval}_{\mathcal{D}}(\mathcal{P}_{\Theta}^{0})$
\For{$t=1,\ldots,T$}
    \State $\mathcal{B}^t\gets\textsc{Pair}(\mathcal{P}_{\Theta}^{t-1},M)$
    \State $r_i^t\gets\Psi_{\mathrm{BU}}(A_i^-,A_i^+,\widehat{L}_i^-,\widehat{L}_i^+)$
    \State $\theta_{i,\mathrm{CX}}^t\sim\BUCX$
    \State $\widehat{L}_{i,\mathrm{CX}}^t\gets\widehat{L}_{\mathcal{D}}(\theta_{i,\mathrm{CX}}^t)$
    \State $\mathcal{P}_{\Theta}^{t,\mathrm{CX}}\gets\textsc{Replace}(\mathcal{C}_{\Theta}^{t,\mathrm{CX}})$
    \State $R_t\gets\textsc{LT}(R_{t-1},\{r_i^t\}_{i=1}^{M})$
    \State $\theta_{j,\mathrm{MT}}^t\sim\BUMT$
    \State $\widehat{L}_{j,\mathrm{MT}}^t\gets\widehat{L}_{\mathcal{D}}(\theta_{j,\mathrm{MT}}^t)$
    \State $\mathcal{P}_{\Theta}^{t}\gets\textsc{Extend}(\mathcal{P}_{\Theta}^{t,\mathrm{CX}},\mathcal{C}_{\Theta}^{t,\mathrm{MT}})$
    \State $H^t\gets\textsc{Update}_{\mathrm{BU}}(H^{t-1},\mathcal{P}_{\Theta}^{t},R_t)$
\EndFor
\State \Return $\widehat{\theta}_{\mathrm{BU}}\gets\textsc{Best}(\mathcal{P}_{\Theta}^{T})$
\end{algorithmic}
\end{minipage}}
\end{minipage}
\hfill
\begin{minipage}[t]{0.49\linewidth}
\captionsetup{type=algorithm}
\rule{\linewidth}{1pt}\vspace{-5pt}
\captionof{algorithm}{Top-down population-based AHD}
\vspace{-5pt}\rule{\linewidth}{1pt}\vspace{5pt}
\label{alg:ea_td_impl}
\resizebox{\linewidth}{!}{%
\begin{minipage}{1.12\linewidth}
\begin{algorithmic}
\State $\mathcal{P}_{\mathcal{K}}^{0}\gets\textsc{Init}_{\mathcal{K}}(M_0,H^0)$
\State $\mathcal{P}_{\mathcal{K}}^{0}\gets\textsc{RealizeEval}_{\mathcal{D}}(\mathcal{P}_{\mathcal{K}}^{0})$
\For{$t=1,\ldots,T$}
    \State $\mathcal{B}^t\gets\textsc{Pair}(\mathcal{P}_{\mathcal{K}}^{t-1},M)$
    \State $r_i^t\gets\Psi_{\mathrm{TD}}(K_i^-,K_i^+,\widehat{L}_i^-,\widehat{L}_i^+)$
    \State $(K_{i,\mathrm{CX}}^t,\theta_{i,\mathrm{CX}}^t)\sim\TDCX$
    \State $\widehat{L}_{i,\mathrm{CX}}^t\gets\widehat{L}_{\mathcal{D}}(\theta_{i,\mathrm{CX}}^t)$
    \State $\mathcal{P}_{\mathcal{K}}^{t,\mathrm{CX}}\gets\textsc{Replace}(\mathcal{C}_{\mathcal{K}}^{t,\mathrm{CX}})$
    \State $R_t\gets\textsc{LT}(R_{t-1},\{r_i^t\}_{i=1}^{M})$
    \State $(K_{j,\mathrm{MT}}^t,\theta_{j,\mathrm{MT}}^t)\sim\TDMT$
    \State $\widehat{L}_{j,\mathrm{MT}}^t\gets\widehat{L}_{\mathcal{D}}(\theta_{j,\mathrm{MT}}^t)$
    \State $\mathcal{P}_{\mathcal{K}}^{t}\gets\textsc{Extend}(\mathcal{P}_{\mathcal{K}}^{t,\mathrm{CX}},\mathcal{C}_{\mathcal{K}}^{t,\mathrm{MT}})$
    \State $H^t\gets\textsc{Update}_{\mathrm{TD}}(H^{t-1},\mathcal{P}_{\mathcal{K}}^{t},R_t)$
\EndFor
\State \Return $(\widehat{K}_{\mathrm{TD}},\widehat{\theta}_{\mathrm{TD}})\gets\textsc{Best}(\mathcal{P}_{\mathcal{K}}^{T})$
\end{algorithmic}
\end{minipage}}
\end{minipage}

\end{figure}
\vspace{-15pt}
\begin{figure}[h]
\centering
\begin{minipage}[t]{0.49\linewidth}
\rule{\linewidth}{1pt}
\end{minipage}
\hfill
\begin{minipage}[t]{0.49\linewidth}
\rule{\linewidth}{1pt}
\end{minipage}
\end{figure}
\vspace{-10 pt}
\paragraph{Top-Down State.}
The proposed variant stores each individual as
\[
    z=(K,\theta,\widehat{L}),\qquad 
    \widehat{L}=\widehat{L}_{\mathcal{D}}(\theta),
\]
where $K$ is the searched knowledge state, $\theta$ is its executable realization, and $A$ is the corresponding source-code artifact. Selection is based on $\widehat{L}$, but variation is applied primarily at the knowledge level. In principle, top-down search requires a realization step $K\mapsto\theta$. To avoid adding extra LLM calls for this step, we use structured output and configure each generation call to return two fields, \texttt{knowledge} and \texttt{code}, in a single response: the model first produces $K$, then produces $\theta$ as an implementation of that same knowledge state. Thus, crossover produces $(K_{i,\mathrm{CX}}^t,\theta_{i,\mathrm{CX}}^t)$ in one call, and mutation similarly refines the elitist knowledge $K_t^\star$ and realizes it as code in one call.

\paragraph{Matched Budget.}
For population size $M$ and mutation rate $\mu$, the number of offspring generated by the crossover (CX) operator is fixed to $M$ (i.e., crossover rate equals 1), while the number of offspring generated by the mutation (MT) operator is defined as $N=\max(1,\lfloor \mu M\rfloor)$. The top-down variant uses the same number of calls as ReEvo in each iteration:
\[
    M \;(\text{short-term reflection})+
    M \;(\text{CX})+
    1 \;(\text{long-term reflection})+
    N \;(\text{MT})
    =2M+1+N.
\]
Under full evaluation, both sides also evaluate $M+N$ generated programs per iteration. Thus, the comparison changes the search coordinate from $\Theta$ to $\mathcal{K}$ without increasing the LLM-call or empirical-evaluation budget.

\paragraph{Evaluation.}
Every generated $\theta$ is written as a Python module with the benchmark-specific function signature and evaluated by the same script. Invalid programs are assigned $\widehat{L}=+\infty$ and removed from the valid population. The active population is replaced after crossover and extended after mutation, while the elitist candidate is tracked separately. For top-down runs, we save both $\widehat{K}_{\mathrm{TD}}$ and $\widehat{\theta}_{\mathrm{TD}}$, so the final implementation can be traced back to the generating knowledge state.

\paragraph{\emph{Dual Top-Down and Bottom-Up Population-Based Search}.}
We further implement a dual variant that maintains two coupled populations,
\[
    \mathcal{P}_{\mathcal{K}}^t=\{(K_i^t,\theta_i^t,\widehat{L}_i^t)\},
    \qquad
    \mathcal{P}_{\Theta}^t=\{(\theta_j^t,\widehat{L}_j^t)\}.
\]
The knowledge population is top-down-native: variation is applied to $K$, and each generated knowledge state is realized as code before evaluation. The code population is bottom-up-native: variation is applied directly to executable artifacts. The two populations interact through two cross-pollination operators: distillation $\mathrm{DS}$ transfers information from strong code into the knowledge population, while grounding $\mathrm{GR}$ uses strong knowledge to generate new executable code.

\newcommand{\DKCX}{Q_{\phi,t}^{\mathcal{K},\Theta}(\cdot\mid H^{t-1},K_i^-,K_i^+,r_{K,i}^t)}
\newcommand{\DKDS}{Q_{\phi,t}^{\mathcal{K},\Theta}(\cdot\mid H^{t-1},K_t^\star,\theta_t^\star,R_{t-1})}
\newcommand{\DTCX}{Q_{\phi,t}^{\Theta}(\cdot\mid H^{t-1},A_i^-,A_i^+,r_{\Theta,i}^t)}
\newcommand{\DTGR}{Q_{\phi,t}^{\Theta}(\cdot\mid H^{t-1},K_t^\star,\theta_t^\star,R_{t-1})}

\begin{algorithm}[h]
\caption{Dual top-down and bottom-up population search}
\label{alg:dual_ea_impl}
\small
\resizebox{\linewidth}{!}{%
\begin{minipage}{1.03\linewidth}
\begin{algorithmic}
\Require $T$, $M_{\mathcal{K}}^0$, $M_{\Theta}^0$, $N_{\mathcal{K}}^{\mathrm{CX}}$, $N_{\mathcal{K}}^{\mathrm{DS}}$, $N_{\Theta}^{\mathrm{CX}}$, $N_{\Theta}^{\mathrm{GR}}$
\State $\mathcal{P}_{\mathcal{K}}^0\gets\textsc{RealizeEval}_{\mathcal{D}}(\textsc{Init}_{\mathcal{K}}(M_{\mathcal{K}}^0,H^0))$
\State $\mathcal{P}_{\Theta}^0\gets\textsc{Eval}_{\mathcal{D}}(\textsc{Init}_{\Theta}(M_{\Theta}^0,H^0)\cup\{\theta(K_0^\star)\})$
\State $R_0\gets\textsc{InitReflection}(H^0)$
\For{$t=1,\ldots,T$}
    \State \textit{// knowledge-side update}
    \State \hspace{\algorithmicindent}$\mathcal{B}_{\mathcal{K}}^t\gets\textsc{Pair}(\mathcal{P}_{\mathcal{K}}^{t-1},N_{\mathcal{K}}^{\mathrm{CX}})$
    \State \hspace{\algorithmicindent}$\mathcal{C}_{\mathcal{K}}^{t,\mathrm{CX}}\gets\mathrm{CX}_{\mathcal{K}}(\mathcal{B}_{\mathcal{K}}^t,H^{t-1})$
    \State \hspace{\algorithmicindent}$\mathcal{C}_{\mathcal{K}}^{t,\mathrm{DS}}\gets\mathrm{DS}(K_{t-1}^{\star},\theta_{t-1}^{\star},R_{t-1},H^{t-1})$
    \State \hspace{\algorithmicindent}$\mathcal{P}_{\mathcal{K}}^t\gets\textsc{Extend}(\mathcal{P}_{\mathcal{K}}^{t-1},\textsc{RealizeEval}_{\mathcal{D}}(\mathcal{C}_{\mathcal{K}}^{t,\mathrm{CX}}\cup\mathcal{C}_{\mathcal{K}}^{t,\mathrm{DS}}))$
    \State \textit{// code-side update}
    \State \hspace{\algorithmicindent}$\mathcal{B}_{\Theta}^t\gets\textsc{Pair}(\mathcal{P}_{\Theta}^{t-1},N_{\Theta}^{\mathrm{CX}})$
    \State \hspace{\algorithmicindent}$\mathcal{C}_{\Theta}^{t,\mathrm{CX}}\gets\mathrm{CX}_{\Theta}(\mathcal{B}_{\Theta}^t,H^{t-1})$
    \State \hspace{\algorithmicindent}$\mathcal{C}_{\Theta}^{t,\mathrm{GR}}\gets\mathrm{GR}(K_t^{\star},\theta_t^{\star},R_{t-1},H^{t-1})$
    \State \hspace{\algorithmicindent}$\mathcal{P}_{\Theta}^t\gets\textsc{Extend}(\mathcal{P}_{\Theta}^{t-1},\textsc{Eval}_{\mathcal{D}}(\mathcal{C}_{\Theta}^{t,\mathrm{CX}}\cup\mathcal{C}_{\Theta}^{t,\mathrm{GR}}))$
    \State \textit{// shared reflection and bookkeeping}
    \State \hspace{\algorithmicindent}$R_t\gets\textsc{LT}(R_{t-1},\mathcal{B}_{\mathcal{K}}^t,\mathcal{B}_{\Theta}^t,\mathcal{P}_{\mathcal{K}}^t,\mathcal{P}_{\Theta}^t)$
    \State \hspace{\algorithmicindent}$H^t\gets\textsc{Update}(H^{t-1},\mathcal{P}_{\mathcal{K}}^t,\mathcal{P}_{\Theta}^t,R_t)$
\EndFor
\State \Return $\widehat{\theta}\gets\textsc{Best}(\mathcal{P}_{\mathcal{K}}^T\cup\mathcal{P}_{\Theta}^T)$
\end{algorithmic}
\end{minipage}}
\end{algorithm}

\paragraph{Cross-pollination operators.}
The dual variant uses four operators per iteration. The knowledge-side crossover $\mathrm{CX}_{\mathcal{K}}$ combines two knowledge individuals $(K_i^-,K_i^+)$ and returns a new pair $(K_{\mathrm{CX}},\theta_{\mathrm{CX}})$. The distillation operator $\mathrm{DS}$ transfers information from the code population to the knowledge population by using the best code-side artifact $\theta_t^\star$ together with the best knowledge-side state $K_t^\star$ to produce a new knowledge state and its reference implementation. Symmetrically, code-side crossover $\mathrm{CX}_{\Theta}$ follows the ReEvo-style code operator, while grounding $\mathrm{GR}$ maps strong knowledge back into executable code. In our implementation, $\mathrm{CX}_{\mathcal{K}}$ and $\mathrm{DS}$ also use structured output, so each LLM call returns both \texttt{knowledge} and \texttt{code}; no extra realization call is added for $K\mapsto\theta$.

\paragraph{Budget control.}
Let $N_{\mathcal{K}}^{\mathrm{CX}}$, $N_{\mathcal{K}}^{\mathrm{DS}}$, $N_{\Theta}^{\mathrm{CX}}$, and $N_{\Theta}^{\mathrm{GR}}$ denote the number of offspring generated by the four operators. The per-iteration generation and evaluation budget is
\[
    N_{\mathcal{K}}^{\mathrm{CX}}
    +N_{\mathcal{K}}^{\mathrm{DS}}
    +N_{\Theta}^{\mathrm{CX}}
    +N_{\Theta}^{\mathrm{GR}}.
\]
In the experiments, these values are chosen so that the total number of evaluated programs is comparable to the single-population BU and TD baselines. Thus, the dual variant studies whether code-level and knowledge-level search provide complementary signals under a matched evaluation budget, rather than benefiting from an uncontrolled increase in solver calls.

\paragraph{\emph{Top-Down Tree-Based Search}.}
For the tree-based setting, we implement the bottom-up baseline following MCTS-AHD~\citep{zheng2025monte} and construct a top-down counterpart by changing the node semantics. In the bottom-up tree, each non-root node stores an executable heuristic and a post-hoc textual description,
\[
    v=(\theta_v,D_v,\widehat{L}_v,N_v,Q_v),
    \qquad A_v=\alpha(\theta_v),
\]
where $d_v$ is generated after code evaluation. In the top-down tree, each node instead stores a knowledge state and its realized executable heuristic,
\[
    v=(K_v,\theta_v,\widehat{L}_v,N_v,Q_v),
    \qquad A_v=\alpha(\theta_v).
\]
Both variants use the same MCTS selection, progressive widening, expansion operators, backpropagation rule, elite archive size, and evaluation budget; they differ only in whether a newly expanded node is generated as code-first or knowledge-first.

\newcommand{\MCTSCODE}{Q_{\phi,t}^{\Theta}(\cdot\mid H^{t-1},v,o,\mathcal{E})}
\newcommand{\MCTSDESC}{Q_{\phi,t}^{\mathcal{D}_{\mathrm{text}}}(\cdot\mid A_u,\widehat{L}_u,H^{t-1})}
\newcommand{\MCTSKNOW}{Q_{\phi,t}^{\mathcal{K}}(\cdot\mid H^{t-1},v,o,\mathcal{E})}
\newcommand{\MCTSIMPL}{Q_{\phi,t}^{\Theta}(\cdot\mid K_u,H^{t-1},v,o,\mathcal{E})}

\begin{figure}[h]
\centering

\begin{minipage}[t]{0.49\linewidth}
\captionsetup{type=algorithm}
\rule{\linewidth}{1pt}\vspace{-5pt}
\captionof{algorithm}{Bottom-up tree-based AHD}
\vspace{-5pt}\rule{\linewidth}{1pt}\vspace{5pt}
\label{alg:mcts_bu_impl}
\resizebox{\linewidth}{!}{%
\begin{minipage}{1.12\linewidth}
\begin{algorithmic}
\State $(\mathcal{T}^0,\mathcal{E}^0)\gets\textsc{InitTree}_{\Theta}(N_I,H^0)$
\While{$B<B_{\max}$}
    \State $\rho\gets 1-B/B_{\max}$
    \State $(v,\mathcal{W})\gets\textsc{UCTSelect}(\mathcal{T}^{t-1},\rho)$
    \State $\mathcal{U}\gets\textsc{ProgressiveWiden}_{\Theta}(\mathcal{W},\mathcal{E}^{t-1})$
    \State $\mathcal{O}\gets\{\mathrm{E2}\}\cup\{\mathrm{M1},\mathrm{M2}\}^{\times k}\cup\{\mathrm{S1}\}$
    \State $\mathcal{U}\gets\mathcal{U}\cup\textsc{Expand}_{\Theta}(v,\mathcal{O},\mathcal{E}^{t-1})$
    \For{$(p,o)\in\mathcal{U}$}
        \State $\theta_u\sim\MCTSCODE$
        \State $\widehat{L}_u\gets\widehat{L}_{\mathcal{D}}(\theta_u)$
        \State $D_u\sim\MCTSDESC$
        \State $\mathcal{T}^{t}\gets\textsc{BackUp}(\mathcal{T}^{t-1},p,\theta_u,D_u,\widehat{L}_u)$
    \EndFor
    \State $\mathcal{E}^{t}\gets\textsc{TopK}(\mathcal{E}^{t-1}\cup\mathcal{U})$
\EndWhile
\State \Return $\widehat{\theta}_{\mathrm{BU}}\gets\textsc{Best}(\mathcal{E}^{T})$
\end{algorithmic}
\end{minipage}}
\end{minipage}
\hfill
\begin{minipage}[t]{0.49\linewidth}
\captionsetup{type=algorithm}
\rule{\linewidth}{1pt}\vspace{-5pt}
\captionof{algorithm}{Top-down tree-based AHD}
\vspace{-5pt}\rule{\linewidth}{1pt}\vspace{5pt}
\label{alg:mcts_td_impl}
\resizebox{\linewidth}{!}{%
\begin{minipage}{1.12\linewidth}
\begin{algorithmic}
\State $(\mathcal{T}^0,\mathcal{E}^0)\gets\textsc{InitTree}_{\mathcal{K}}(N_I,H^0)$
\While{$B<B_{\max}$}
    \State $\rho\gets 1-B/B_{\max}$
    \State $(v,\mathcal{W})\gets\textsc{UCTSelect}(\mathcal{T}^{t-1},\rho)$
    \State $\mathcal{U}\gets\textsc{ProgressiveWiden}_{\mathcal{K}}(\mathcal{W},\mathcal{E}^{t-1})$
    \State $\mathcal{O}\gets\{\mathrm{E2}\}\cup\{\mathrm{M1},\mathrm{M2}\}^{\times k}\cup\{\mathrm{S1}\}$
    \State $\mathcal{U}\gets\mathcal{U}\cup\textsc{Expand}_{\mathcal{K}}(v,\mathcal{O},\mathcal{E}^{t-1})$
    \For{$(p,o)\in\mathcal{U}$}
        \State $K_u\sim\MCTSKNOW$
        \State $\theta_u\sim\MCTSIMPL$
        \State $\widehat{L}_u\gets\widehat{L}_{\mathcal{D}}(\theta_u)$
        \State $\mathcal{T}^{t}\gets\textsc{BackUp}(\mathcal{T}^{t-1},p,K_u,\theta_u,\widehat{L}_u)$
    \EndFor
    \State $\mathcal{E}^{t}\gets\textsc{TopK}(\mathcal{E}^{t-1}\cup\mathcal{U})$
\EndWhile
\State \Return $(\widehat{K}_{\mathrm{TD}},\widehat{\theta}_{\mathrm{TD}})\gets\textsc{Best}(\mathcal{E}^{T})$
\end{algorithmic}
\end{minipage}}
\end{minipage}

\end{figure}

\vspace{-15pt}
\begin{figure}[h]
\centering
\begin{minipage}[t]{0.49\linewidth}
\rule{\linewidth}{1pt}
\end{minipage}
\hfill
\begin{minipage}[t]{0.49\linewidth}
\rule{\linewidth}{1pt}
\end{minipage}
\end{figure}
\vspace{-5 pt}

\paragraph{Tree Policy.}
Both variants use the same MCTS tree policy. Starting from the root, the algorithm selects child nodes by UCT with an evaluation-dependent exploration factor $\rho_t=1-B_t/B_{\max}$. Each node stores visit count $N_v$ and value $Q_v=-\widehat{L}_v$, so lower loss corresponds to higher tree value. Progressive widening is applied during selection: when $\lfloor N_v^{\alpha}\rfloor$ exceeds the current number of children of $v$, an additional child is generated before the search continues. After a new child is evaluated, its value is backed up along the path from the child to the root by updating visit counts and propagating the best child value. The tree policy is controlled by three hyperparameters: $\lambda_0$, the base exploration coefficient in UCT; $\alpha$, the progressive-widening exponent; and $d_{\max}$, the maximum search depth.

\paragraph{Expansion Operators.}
We keep the MCTS-AHD operator set fixed,
\[
    \mathcal{O}=\{\mathrm{I1},\mathrm{E1},\mathrm{E2},\mathrm{M1},\mathrm{M2},\mathrm{S1}\}.
\]
The operator $\mathrm{I1}$ creates the first root child; $\mathrm{E1}$ expands the root by combining candidates sampled from different root subtrees; $\mathrm{E2}$ expands a non-root node using an elite reference; $\mathrm{M1}$ and $\mathrm{M2}$ produce local variants of the selected node; and $\mathrm{S1}$ uses the sequence of nodes from the root to the selected node as context. In BU, an operator directly generates $\theta_u$ and then annotates it with a description $D_u$. In TD, the same operator first proposes $K_u$ and then realizes it as $\theta_u$ for evaluation.

\paragraph{Matched Node-Expansion Budget.}
The bottom-up implementation uses two LLM calls per expanded node,
\[
    \theta_u \rightarrow D_u,
\]
where code is generated first and the description is produced after evaluation. The top-down implementation also uses two LLM calls per expanded node,
\[
    K_u \rightarrow \theta_u,
\]
where knowledge is generated first and then implemented as code. Therefore, each expanded node uses two LLM calls and one empirical evaluation in both variants. With the operator schedule $\mathrm{E2}\times1$, $\mathrm{M1}\times k$, $\mathrm{M2}\times k$, and $\mathrm{S1}\times1$, each selected leaf expands $2k+2$ children, corresponding to $4k+4$ LLM calls and $2k+2$ evaluations, before accounting for any additional progressive-widening expansions.

\newpage
\paragraph{\emph{Sparse Evaluation for Population-Based Search}.}
We also implement sparse-evaluation variants of the population-based algorithms. Let $\eta\in[0,1]$ be the evaluation ratio. For a generated batch $\mathcal{C}$ with $|\mathcal{C}|=n$, define
\[
b_\eta(n)=
\begin{cases}
0, & \eta=0,\\
\min\{n,\max\{1,\lfloor \eta n\rceil\}\}, & \eta>0,
\end{cases}
\]
where $b_\eta(n)$ is the number of candidates evaluated immediately. The remaining candidates are kept as unevaluated artifacts.

\newcommand{\SA}{\mathrm{SA}_{\eta,\mathcal{D}}}
\newcommand{\SPBUCX}{\textsc{SpCX}_{\Theta}}
\newcommand{\SPBUMT}{\textsc{SpMT}_{\Theta}}
\newcommand{\SPTDCX}{\textsc{SpCX}_{\mathcal{K}}}
\newcommand{\SPTDMT}{\textsc{SpMT}_{\mathcal{K}}}

\begin{figure}[h]
\centering

\begin{minipage}[t]{0.49\linewidth}
\captionsetup{type=algorithm}
\rule{\linewidth}{1pt}\vspace{-5pt}
\captionof{algorithm}{Sparse bottom-up population search}
\vspace{-5pt}\rule{\linewidth}{1pt}\vspace{5pt}
\label{alg:sparse_bu_impl}
\resizebox{\linewidth}{!}{%
\begin{minipage}{1.03\linewidth}
\begin{algorithmic}
\Require $T$, $M$, $M_0$, $\mu$, $\eta$
\State $\mathcal{P}_{\Theta}^{0}\gets\SA(\textsc{Init}_{\Theta}(M_0,H^0))$
\For{$t=1,\ldots,T$}
    \State $\mathcal{B}^t\gets\textsc{SparsePair}(\mathcal{P}_{\Theta}^{t-1},M)$
    \State $\mathcal{C}_{\Theta}^{t,\mathrm{CX}}\gets\SPBUCX(\mathcal{B}^t,H^{t-1})$
    \State $\mathcal{P}_{\Theta}^{t,\mathrm{CX}}\gets\textsc{ReplaceSp}(\SA(\mathcal{C}_{\Theta}^{t,\mathrm{CX}}))$
    \State $R_t\gets\textsc{LT}(R_{t-1},\mathcal{B}^t,\mathcal{C}_{\Theta}^{t,\mathrm{CX}})$
    \State $M_{\mathrm{MT}}\gets\max(1,\lfloor\mu M\rfloor)$
    \State $\mathcal{C}_{\Theta}^{t,\mathrm{MT}}\gets\SPBUMT(\theta_t^\star,R_t,H^{t,\mathrm{CX}},M_{\mathrm{MT}})$
    \State $\mathcal{P}_{\Theta}^{t}\gets\textsc{ExtendSp}(\mathcal{P}_{\Theta}^{t,\mathrm{CX}},\SA(\mathcal{C}_{\Theta}^{t,\mathrm{MT}}))$
    \State $H^t\gets\textsc{Update}_{\mathrm{BU}}(H^{t-1},\mathcal{P}_{\Theta}^{t},R_t)$
\EndFor
\State \Return $\widehat{\theta}_{\mathrm{BU}}\gets\textsc{BestEval}(\mathcal{P}_{\Theta}^{T})$
\end{algorithmic}
\end{minipage}}
\end{minipage}
\hfill
\begin{minipage}[t]{0.49\linewidth}
\captionsetup{type=algorithm}
\rule{\linewidth}{1pt}\vspace{-5pt}
\captionof{algorithm}{Sparse top-down population search}
\vspace{-5pt}\rule{\linewidth}{1pt}\vspace{5pt}
\label{alg:sparse_td_impl}
\resizebox{\linewidth}{!}{%
\begin{minipage}{1.03\linewidth}
\begin{algorithmic}
\Require $T$, $M$, $M_0$, $\mu$, $\eta$
\State $\mathcal{P}_{\mathcal{K}}^{0}\gets\SA(\textsc{Init}_{\mathcal{K}}(M_0,H^0))$
\For{$t=1,\ldots,T$}
    \State $\mathcal{B}^t\gets\textsc{SparsePair}(\mathcal{P}_{\mathcal{K}}^{t-1},M)$
    \State $\mathcal{C}_{\mathcal{K}}^{t,\mathrm{CX}}\gets\SPTDCX(\mathcal{B}^t,H^{t-1})$
    \State $\mathcal{P}_{\mathcal{K}}^{t,\mathrm{CX}}\gets\textsc{ReplaceSp}(\SA(\mathcal{C}_{\mathcal{K}}^{t,\mathrm{CX}}))$
    \State $R_t\gets\textsc{LT}(R_{t-1},\mathcal{B}^t,\mathcal{C}_{\mathcal{K}}^{t,\mathrm{CX}})$
    \State $M_{\mathrm{MT}}\gets\max(1,\lfloor\mu M\rfloor)$
    \State $\mathcal{C}_{\mathcal{K}}^{t,\mathrm{MT}}\gets\SPTDMT(K_t^\star,R_t,H^{t,\mathrm{CX}},M_{\mathrm{MT}})$
    \State $\mathcal{P}_{\mathcal{K}}^{t}\gets\textsc{ExtendSp}(\mathcal{P}_{\mathcal{K}}^{t,\mathrm{CX}},\SA(\mathcal{C}_{\mathcal{K}}^{t,\mathrm{MT}}))$
    \State $H^t\gets\textsc{Update}_{\mathrm{TD}}(H^{t-1},\mathcal{P}_{\mathcal{K}}^{t},R_t)$
\EndFor
\State \Return $(\widehat{K}_{\mathrm{TD}},\widehat{\theta}_{\mathrm{TD}})\gets\textsc{BestEval}(\mathcal{P}_{\mathcal{K}}^{T})$
\end{algorithmic}
\end{minipage}}
\end{minipage}

\end{figure}

\vspace{-15pt}
\begin{figure}[h]
\centering
\begin{minipage}[t]{0.49\linewidth}
\rule{\linewidth}{1pt}
\end{minipage}
\hfill
\begin{minipage}[t]{0.49\linewidth}
\rule{\linewidth}{1pt}
\end{minipage}
\end{figure}
\vspace{-5 pt}

\paragraph{Sparse Assimilation.}
For a batch $\mathcal{C}=\{z_i\}_{i=1}^{n}$, sparse assimilation samples an evaluated subset $I_{\eta}\subseteq[n]$ with $|I_\eta|=b_\eta(n)$ and sets
\[
\SA(\mathcal{C})
=
\{(z_i,\widehat{L}_{\mathcal{D}}(\theta_i),\mathrm{eval}) : i\in I_\eta\}
\cup
\{(z_i,+\infty,\mathrm{unk}) : i\notin I_\eta\}.
\]
The value $+\infty$ here is only a placeholder for sorting; unevaluated candidates are not treated as failed candidates. The population keeps two parts,
\[
\mathcal{P}^t=\mathcal{P}_{\mathrm{eval}}^t\cup\mathcal{P}_{\mathrm{unk}}^t,
\qquad
|\mathcal{P}_{\mathrm{unk}}^t|\le S_{\mathrm{unk}},
\]
where $S_{\mathrm{unk}}$ is the reserved number of unevaluated slots. The elitist is updated only from $\mathcal{P}_{\mathrm{eval}}^t$.

\paragraph{Pair Selection under Uncertainty.}
Sparse pair selection returns $(z_i,z_j,\delta_{ij})$, where
\[
\delta_{ij}
=
\mathbb{I}\{s_i=\mathrm{eval},\,s_j=\mathrm{eval},\,\widehat{L}_i\ne\widehat{L}_j\}.
\]
If $\delta_{ij}=1$, the pair is ordered as $(z^-,z^+)$ by empirical score. If $\delta_{ij}=0$, the pair is unranked and no performance preference is inferred. Thus the sparse reflection operator has the form
\[
r_{ij}^t=
\begin{cases}
\Psi^{\mathrm{rank}}(z^-,z^+,\widehat{L}^-,\widehat{L}^+), & \delta_{ij}=1,\\
\Psi^{\mathrm{unk}}(z_i,z_j), & \delta_{ij}=0.
\end{cases}
\]
For BU, $z$ exposes code artifacts $A=\alpha(\theta)$; for TD, $z$ additionally exposes knowledge $K$, so unranked pairs can still support knowledge-level recombination.

\paragraph{Evaluation Budget.}
With full evaluation, each iteration evaluates
\[
E_{\mathrm{full}}=M+N,
\qquad
N=\max(1,\lfloor\mu M\rfloor).
\]
Under sparse evaluation, this becomes
\[
E_{\mathrm{sp}}
=
b_\eta(M)+b_\eta(N)
\approx
\eta(M+N),
\]
while the LLM-call budget remains unchanged. Therefore sparse evaluation isolates the effect of reducing empirical solver calls. The difference between BU and TD is in the information retained by unevaluated candidates:
\[
H_{\mathrm{BU}}^t\supseteq\{A_i:s_i=\mathrm{unk}\},
\qquad
H_{\mathrm{TD}}^t\supseteq\{K_i,A_i:s_i=\mathrm{unk}\}.
\]
This is why sparse evaluation is a natural stress test for knowledge-centric search: TD can still use unevaluated candidates as structural evidence through $K$, whereas BU mainly retains unevaluated code artifacts.

\appsubsection{Hyperparameters and Prompts}

\paragraph{Hyperparameters.}
Unless otherwise specified, all experiments in this paper use \texttt{GPT-4o-mini} (specifically, \texttt{GPT-4o-mini-2024-07-18}) as the default language model, since it is inexpensive, low-latency, broadly knowledgeable, and supports structured outputs. For all LLMs, unless explicitly stated otherwise, we use the default temperature setting. All reported results are averaged over 5 independent runs to improve statistical reliability. The framework-specific hyperparameters are summarized in the tables below. Note that MCTS-AHD generally consumes more tokens than ReEvo, and ReEvo also admits straightforward parallel evaluation whereas MCTS-AHD proceeds sequentially. To maintain a fair comparison under practical compute constraints, we therefore use a slightly smaller search budget for the MCTS-based settings.

\begin{table}[h]
\centering
\caption{Hyperparameters of the population-based search frameworks.}
\label{tab:pop_search_hparams}
\setlength{\tabcolsep}{4pt}
\renewcommand{\arraystretch}{1.2}
\resizebox{\linewidth}{!}{%
\begin{tabular}{|C{3.5cm}|C{1.55cm}|C{1.8cm}|C{1.55cm}|C{1.55cm}|C{1.85cm}|C{1.85cm}|C{1.85cm}|}
\hline
\textbf{Variant} & \textbf{Initial size} & \textbf{Iterations} & \textbf{Population size} & \textbf{Mutation rate} & \textbf{LLM calls / iteration} & \textbf{Total LLM calls} & \textbf{Codes generated} \\
\hline
Top-down ReEvo & 10 & 30 & 5 & 1.0 & 16 & 490 & 310 \\
\hline
Bottom-Up ReEvo & 10 & 30 & 5 & 1.0 & 16 & 490 & 310 \\
\hline
\end{tabular}%
}

\setlength{\tabcolsep}{4pt}
\renewcommand{\arraystretch}{1.2}
\resizebox{\linewidth}{!}{%
\begin{tabular}{|C{3.5cm}|C{1.55cm}|C{1.8cm}|C{1.55cm}|C{1.55cm}|C{1.85cm}|C{1.85cm}|C{1.85cm}|}
\hline
\textbf{Variant} & \textbf{Initial size} & \textbf{Max candidates} & \textbf{Elite size} & \textbf{$k$} & \textbf{$(\lambda_0,\alpha,d_{\max})$} & \textbf{Total LLM calls} & \textbf{Codes generated} \\
\hline
Top-down MCTS-AHD & 4 & 200 & 10 & 2 & $(0.1,0.5,10)$ & 400 & 200 \\
\hline
Bottom-up MCTS-AHD & 4 & 200 & 10 & 2 & $(0.1,0.5,10)$ & 400 & 200 \\
\hline
\end{tabular}%
}
\end{table}

Our study requires only a mid-range CPU and does not rely on GPUs, unless one chooses to run the LLM locally. Stronger CPUs can increase the degree of parallel code evaluation, but this does not materially affect the overall conclusions. Under the settings described above, a typical run of ReEvo (bottom-up or top-down) takes approximately 15--20 minutes, whereas a typical run of MCTS-AHD (bottom-up or top-down) takes about 40 minutes. The average API cost per run is around \$0.10 when using \texttt{GPT-4o-mini}, and both runtime and monetary cost increase accordingly when more expensive models are used.

\paragraph{Prompts for Top-Down FunSearch.}
Below we present the prompts used in our top-down FunSearch variant, where the model first proposes high-level ideas and then implements them as code.

\begin{prompt}[system-prompt]
You are an expert algorithm designer. You write clean, correct, and efficient Python code. Your goal is to minimize the objective score (lower is better).
\end{prompt}

\begin{prompt}[brainstorm-prompt]
\# Task: Implement \ph{func\_name}.

\ph{func\_sign}

\# Function Description: \ph{func\_desc}

\# Baseline: \ph{func\_seed}

Baseline score: \ph{baseline\_score}

\# Previous Ideas and Implementations:
\ph{top\_ideas\_with\_code}

\# Hint:
\ph{hint}

Brainstorm exactly \ph{n\_ideas} better ideas that can achieve a lower score than the current best. Each idea must be distinct and improve upon previous attempts. Do not repeat previous ideas.
\end{prompt}

\begin{prompt}[implement-prompt]
\# Task: Implement \ph{func\_name}.

\ph{func\_sign}

\# Function Description: \ph{func\_desc}

\# Baseline: \ph{func\_seed}

Baseline score: \ph{baseline\_score}

\# Current Best --- score: \ph{best\_score}
\ph{best\_idea}

\ph{best\_code}

\# Idea to Implement:
\ph{idea}

Implement the idea above as \ph{func\_name} to achieve a lower score than the current best.
\end{prompt}

\paragraph{Prompts for Top-Down ReEvo.}
Bottom-up ReEvo follows the standard ReEvo prompting scheme. Below we present the prompts used in our top-down variant, where the search object is knowledge rather than code. Gray placeholders denote runtime inputs.

\begin{prompt}[generator-system-prompt]
You are an expert in the domain of optimization heuristics. Your task is to design heuristics that can effectively solve optimization problems.
\end{prompt}

\begin{prompt}[reflector-system-prompt]
You are an expert in the domain of optimization heuristics. Your task is to give hints to design better heuristics.
\end{prompt}

\begin{prompt}[init-prompt]
\# Task: Write a \ph{func\_name} function for \ph{prob\_name}.

\ph{func\_desc}

\# Objective: \ph{objective\_desc}

\# Function Signature: \ph{func\_sign}

\# Baseline: \ph{func\_seed}

Baseline \ph{baseline\_score}

\# Hints: \ph{lt\_reflection}

Refer to the baseline above. Be very creative and give a meaningfully different and better \ph{func\_name} that achieves a lower score.
\end{prompt}

\begin{prompt}[short-term-reflection-prompt]
Below are two sets of knowledge about \ph{func\_name} for \ph{prob\_name}.

\ph{func\_desc}

\# Objective: \ph{objective\_desc}

Each knowledge set contains principles/insights used to design a heuristic. The second set led to better performance.

\# Worse knowledge --- \ph{worse\_score}
\ph{worse\_knowledge}

\# Better knowledge --- \ph{better\_score}
\ph{better\_knowledge}

Performance changed from \ph{worse\_score} to \ph{better\_score}. What is the worse knowledge getting wrong? Respond with one actionable hint, using less than 20 words.
\end{prompt}

\begin{prompt}[long-term-reflection-prompt]
Below is your prior long-term reflection on designing heuristics for \ph{prob\_name}.

\ph{prior\_lt\_reflection}

Below are some newly gained insights.

\ph{st\_reflections}

Write constructive hints for designing better heuristics, based on prior reflections and new insights and using less than 50 words.
\end{prompt}

\begin{prompt}[crossover-prompt]
\# Task: Write a \ph{func\_name} function for \ph{prob\_name}.

\ph{func\_desc}

\# Objective: \ph{objective\_desc}

\# Baseline: \ph{func\_seed}

Baseline \ph{baseline\_score}

\# Worse knowledge --- \ph{worse\_score}
\ph{worse\_knowledge}

\# Better knowledge --- \ph{better\_score}
\ph{better\_knowledge}

\# Reference implementation:
\ph{better\_code}

\# Reflection:
\ph{st\_reflection}

\# Synthesized knowledge:
Combine what works, discard what doesn't, and add one new insight of your own. Then, implement \ph{func\_name} based on your synthesized knowledge to achieve a lower score.
\end{prompt}

\begin{prompt}[mutation-prompt]
\# Task: Write a \ph{func\_name} function for \ph{prob\_name}.

\ph{func\_desc}

\# Objective: \ph{objective\_desc}

\# Baseline: \ph{func\_seed}

Baseline \ph{baseline\_score}

\# Prior reflection:
\ph{lt\_reflection\_plus\_hint}

\# Current best knowledge --- \ph{elitist\_score}
\ph{elitist\_knowledge}

\# Reference implementation:
\ph{elitist\_code}

\# New knowledge:
The current knowledge achieved score \ph{elitist\_score}. Try a different approach or formulation, not just a tweak. Then, implement \ph{func\_name} based on your new knowledge to achieve a lower score.
\end{prompt}

\paragraph{Prompts for Top-Down MCTS-AHD.}
Below we present the prompts used in the top-down MCTS-AHD variant. Each tree-expansion step first generates knowledge and then converts that knowledge into code through a shared implementation prompt.

\begin{prompt}[generator-system-prompt]
You are an expert in the domain of optimization heuristics. Your task is to design heuristics that can effectively solve optimization problems.
\end{prompt}

\begin{prompt}[i1-prompt]
\# Task: Write a \ph{func\_name} function for \ph{prob\_name}.

\ph{func\_desc}

\# Function Signature:
\ph{func\_sign}

\# Baseline: \ph{func\_seed}

Baseline score: \ph{baseline\_score}

Develop knowledge about \ph{prob\_name} that can guide the design of \ph{func\_name}.

Your response must contain two parts:

1. \textbf{Analysis}: What useful information from the input is the current approach ignoring or underusing? Think beyond the obvious. Considering: \ph{hint}

2. \textbf{Strategy}: Describe a new design idea that exploits the signal identified in your analysis.
\end{prompt}

\begin{prompt}[e1-prompt]
\# Task: Write a \ph{func\_name} function for \ph{prob\_name}.

\ph{func\_desc}

\# Function Signature:
\ph{func\_sign}

\# Baseline: \ph{func\_seed}

Baseline score: \ph{baseline\_score}

I have \ph{n\_nodes} existing sets of knowledge with their implementations:

\ph{knowledge\_nodes\_with\_code}

Develop fundamentally different knowledge about \ph{prob\_name} --- a fresh perspective, not a recombination of the above.

Your response must contain two parts:

1. \textbf{Analysis}: What useful information from the input is the current approach ignoring or underusing? Think beyond the obvious. Considering: \ph{hint}

2. \textbf{Strategy}: Describe a new design idea that exploits the signal identified in your analysis.
\end{prompt}

\begin{prompt}[e2-prompt]
\# Task: Write a \ph{func\_name} function for \ph{prob\_name}.

\ph{func\_desc}

\# Function Signature:
\ph{func\_sign}

\# Baseline: \ph{func\_seed}

Baseline score: \ph{baseline\_score}

\# Reference knowledge | score: \ph{reference\_score}
\ph{reference\_knowledge}

\ph{reference\_code}

\# Parent knowledge | score: \ph{parent\_score}
\ph{parent\_knowledge}

\ph{parent\_code}

\# Synthesized knowledge:
Combine the strongest insights from both and develop a deeper understanding.

Your response must contain two parts:

1. \textbf{Analysis}: What useful information from the input is the current approach ignoring or underusing? Think beyond the obvious. Considering: \ph{hint}

2. \textbf{Strategy}: Describe a new design idea that exploits the signal identified in your analysis.
\end{prompt}

\begin{prompt}[m1-prompt]
\# Task: Write a \ph{func\_name} function for \ph{prob\_name}.

\ph{func\_desc}

\# Function Signature:
\ph{func\_sign}

\# Baseline: \ph{func\_seed}

Baseline score: \ph{baseline\_score}

\# Current knowledge | score: \ph{node\_score}
\ph{node\_knowledge}

\ph{node\_code}

\# New knowledge:
The current strategy scored \ph{node\_score}. \ph{score\_guidance} Develop different knowledge --- not just a tweak.

Your response must contain two parts:

1. \textbf{Analysis}: What useful information from the input is the current approach ignoring or underusing? Think beyond the obvious. Considering: \ph{hint}

2. \textbf{Strategy}: Describe a new design idea that exploits the signal identified in your analysis.
\end{prompt}

\begin{prompt}[m2-prompt]
\# Task: Write a \ph{func\_name} function for \ph{prob\_name}.

\ph{func\_desc}

\# Function Signature:
\ph{func\_sign}

\# Baseline: \ph{func\_seed}

Baseline score: \ph{baseline\_score}

\# Current knowledge | score: \ph{node\_score}
\ph{node\_knowledge}

\ph{node\_code}

\# Refined knowledge:
The current strategy scored \ph{node\_score}. \ph{score\_guidance} Keep the same overall direction, but refine the quantitative aspects --- parameters, weights, thresholds.

Your response must contain two parts:

1. \textbf{Analysis}: What useful information from the input is the current approach ignoring or underusing? Think beyond the obvious. Considering: \ph{hint}

2. \textbf{Strategy}: Describe a new design idea that exploits the signal identified in your analysis.
\end{prompt}

\begin{prompt}[s1-prompt]
\# Task: Write a \ph{func\_name} function for \ph{prob\_name}.

\ph{func\_desc}

\# Function Signature:
\ph{func\_sign}

\# Baseline: \ph{func\_seed}

Baseline score: \ph{baseline\_score}

Below are \ph{n\_stages} stages of knowledge representing how understanding of \ph{prob\_name} deepened:

\ph{path\_knowledge\_with\_code}

\# Synthesized knowledge:
Combine the best insights from all stages and push further.

Your response must contain two parts:

1. \textbf{Analysis}: What useful information from the input is the current approach ignoring or underusing? Think beyond the obvious. Considering: \ph{hint}

2. \textbf{Strategy}: Describe a new design idea that exploits the signal identified in your analysis.
\end{prompt}

\begin{prompt}[implement-prompt]
\# Task: Write a \ph{func\_name} function for \ph{prob\_name}.

\ph{func\_desc}

\# Function Signature:
\ph{func\_sign}

\# Baseline: \ph{func\_seed}

Baseline score: \ph{baseline\_score}

\# Reference code 1 | score: \ph{ref\_score\_1}
\ph{ref\_code\_1}

\# Reference code 2 | score: \ph{ref\_score\_2}
\ph{ref\_code\_2}

\# Knowledge to implement:
\ph{knowledge}

Implement \ph{func\_name} based on the knowledge above to achieve a score lower than the baseline. Use the reference code(s) for structure and syntax guidance. Hint: \ph{hint}
\end{prompt}

\paragraph{Prompts for Top-Down ReEvo (Cross-Problem Transfer).}
Below we present the prompts used in the top-down cross-problem transfer variant, where successful source-problem knowledge is adapted to the target problem.

\begin{prompt}[transfer-system-prompt]
You are an expert in the domain of optimization heuristics. You transfer successful strategies from one problem to another.
\end{prompt}

\begin{prompt}[transfer-reflector-system-prompt]
You are an expert in the domain of optimization heuristics. You analyze why a heuristic works well or poorly when transferred across problems.
\end{prompt}

\begin{prompt}[init-prompt]
\# Source problem: \ph{src\_prob\_name}

Function: \ph{src\_func\_name}

\ph{src\_func\_sign}

\ph{src\_func\_desc}

\# Successful source knowledge:
\ph{source\_knowledge}

\# Target problem: \ph{tgt\_prob\_name}

Write a \ph{tgt\_func\_name} function for \ph{tgt\_prob\_name}.

\ph{tgt\_func\_desc}

\# Function signature:
\ph{tgt\_func\_sign}

\# Baseline: \ph{tgt\_func\_seed}

Baseline score: \ph{baseline\_score}

\# Hints:
\ph{lt\_reflection}

Identify which principles from the source knowledge transfer to \ph{tgt\_prob\_name} and which need rethinking. Considering: \ph{tgt\_hint} Then, implement \ph{tgt\_func\_name} based on your adapted knowledge.
\end{prompt}

\begin{prompt}[short-term-reflection-prompt]
\# Source problem: \ph{src\_prob\_name}

Function: \ph{src\_func\_name}

\ph{src\_func\_sign}

\ph{src\_func\_desc}

Two source-knowledge lineages were adapted to \ph{tgt\_prob\_name}.

\# Source knowledge A:
\ph{source\_knowledge\_a}

→ Adapted result scored: \ph{worse\_score}

\# Source knowledge B:
\ph{source\_knowledge\_b}

→ Adapted result scored: \ph{better\_score}

Which source principles transferred better to \ph{tgt\_prob\_name} and why? What does the worse adaptation miss about \ph{tgt\_prob\_name}? Respond with one actionable hint, using less than 20 words.
\end{prompt}

\begin{prompt}[long-term-reflection-prompt]
You are helping transfer heuristic strategies from \ph{src\_prob\_name} to \ph{tgt\_prob\_name}.

\# Prior transfer insights:
\ph{prior\_lt\_reflection}

\# New observations:
\ph{st\_reflections}

Summarize what transfers well and what doesn't between these problems. Write constructive hints for better adaptation, using less than 50 words.
\end{prompt}

\begin{prompt}[crossover-prompt]
\# Source problem: \ph{src\_prob\_name}

Function: \ph{src\_func\_name}

\ph{src\_func\_sign}

\ph{src\_func\_desc}

\# Source knowledge lineage A:
\ph{source\_knowledge\_a}

→ Adapted score: \ph{worse\_score}

\# Source knowledge lineage B:
\ph{source\_knowledge\_b}

→ Adapted score: \ph{better\_score}

\# Target problem: \ph{tgt\_prob\_name}

Write a \ph{tgt\_func\_name} function for \ph{tgt\_prob\_name}.

\ph{tgt\_func\_desc}

\# Function signature:
\ph{tgt\_func\_sign}

\# Baseline: \ph{tgt\_func\_seed}

Baseline score: \ph{baseline\_score}

\# Best so far on target:
\ph{best\_target\_score}

\# Transfer reflection:
\ph{st\_reflection}

Combine the best-transferring principles from both source knowledge sets. Then, implement \ph{tgt\_func\_name} based on the combined adapted knowledge.
\end{prompt}

\begin{prompt}[mutation-prompt]
\# Source problem: \ph{src\_prob\_name}

Function: \ph{src\_func\_name}

\ph{src\_func\_sign}

\ph{src\_func\_desc}

\# Source knowledge lineage to adapt:
\ph{source\_knowledge}

\# Target problem: \ph{tgt\_prob\_name}

Write a \ph{tgt\_func\_name} function for \ph{tgt\_prob\_name}.

\ph{tgt\_func\_desc}

\# Function signature:
\ph{tgt\_func\_sign}

\# Baseline: \ph{tgt\_func\_seed}

Baseline score: \ph{baseline\_score}

\# Best so far on target:
\ph{best\_target\_score}

\# Transfer insights:
\ph{lt\_reflection\_plus\_hint}

Adapt this source knowledge to \ph{tgt\_prob\_name}. Try a meaningfully different transfer strategy than previous attempts. Then, implement \ph{tgt\_func\_name} based on your adapted knowledge.
\end{prompt}

\paragraph{Prompts for Top-Down MCTS-AHD (Cross-Problem Transfer).}
Below we present the prompts used in the top-down cross-problem transfer variant of MCTS-AHD, where source-problem knowledge is adapted into target-problem knowledge before being implemented as code.

\begin{prompt}[transfer-system-prompt]
You are an expert in the domain of optimization heuristics. You transfer successful strategies from one problem to another.
\end{prompt}

\begin{prompt}[transfer-reflector-system-prompt]
You are an expert in the domain of optimization heuristics. You analyze why a heuristic works well or poorly when transferred across problems.
\end{prompt}

\begin{prompt}[i1-prompt]
\# Source problem: \ph{src\_prob\_name}

Function: \ph{src\_func\_name}

\ph{src\_func\_sign}

\ph{src\_func\_desc}

\# Successful source knowledge:
\ph{source\_knowledge}

\# Target problem: \ph{tgt\_prob\_name}

Write a \ph{tgt\_func\_name} function for \ph{tgt\_prob\_name}.

\ph{tgt\_func\_desc}

\# Function Signature:
\ph{tgt\_func\_sign}

\# Baseline: \ph{tgt\_func\_seed}

Baseline score: \ph{baseline\_score}

Adapt the source knowledge to develop knowledge about \ph{tgt\_prob\_name} that can guide the design of \ph{tgt\_func\_name}.

Your response must contain two parts:

1. \textbf{Analysis}: What transferable principles from the source problem is the current approach ignoring or underusing? Think beyond the obvious. Considering: \ph{tgt\_hint}

2. \textbf{Strategy}: A concrete strategy for \ph{tgt\_prob\_name} that exploits these transferred principles. Must differ structurally from previous attempts --- hyperparameter-only changes are \emph{not} new strategies.
\end{prompt}

\begin{prompt}[e1-prompt]
\# Source problem: \ph{src\_prob\_name}

Function: \ph{src\_func\_name}

\ph{src\_func\_sign}

\ph{src\_func\_desc}

\# Successful source knowledge:
\ph{source\_knowledge\_set}

\# Target problem: \ph{tgt\_prob\_name}

Write a \ph{tgt\_func\_name} function for \ph{tgt\_prob\_name}.

\ph{tgt\_func\_desc}

\# Function Signature:
\ph{tgt\_func\_sign}

\# Baseline: \ph{tgt\_func\_seed}

Baseline score: \ph{baseline\_score}

\# Existing adaptations on target:
\ph{target\_adaptation\_scores}

Develop fundamentally different knowledge about \ph{tgt\_prob\_name} by transferring overlooked principles from the source problem --- a fresh perspective, not a recombination of the above.

Your response must contain two parts:

1. \textbf{Analysis}: What transferable principles from the source problem is the current approach ignoring or underusing? Think beyond the obvious. Considering: \ph{tgt\_hint}

2. \textbf{Strategy}: A concrete strategy for \ph{tgt\_prob\_name} that exploits these transferred principles. Must differ structurally from previous attempts --- hyperparameter-only changes are \emph{not} new strategies.
\end{prompt}

\begin{prompt}[e2-prompt]
\# Source problem: \ph{src\_prob\_name}

Function: \ph{src\_func\_name}

\ph{src\_func\_sign}

\ph{src\_func\_desc}

\# Source knowledge A --- adapted score: \ph{reference\_score}
\ph{source\_knowledge\_a}

\# Source knowledge B --- adapted score: \ph{parent\_score}
\ph{source\_knowledge\_b}

\# Target problem: \ph{tgt\_prob\_name}

Write a \ph{tgt\_func\_name} function for \ph{tgt\_prob\_name}.

\ph{tgt\_func\_desc}

\# Function Signature:
\ph{tgt\_func\_sign}

\# Baseline: \ph{tgt\_func\_seed}

Baseline score: \ph{baseline\_score}

Synthesize the best-transferring principles from both source knowledge sets. Identify what applies to \ph{tgt\_prob\_name} and what needs rethinking. Considering: \ph{tgt\_hint}

Your response must contain two parts:

1. \textbf{Analysis}: What transferable principles from the source problem is the current approach ignoring or underusing? Think beyond the obvious.

2. \textbf{Strategy}: A concrete strategy for \ph{tgt\_prob\_name} that exploits these transferred principles. Must differ structurally from previous attempts --- hyperparameter-only changes are \emph{not} new strategies.
\end{prompt}

\begin{prompt}[m1-prompt]
\# Source problem: \ph{src\_prob\_name}

Function: \ph{src\_func\_name}

\ph{src\_func\_sign}

\ph{src\_func\_desc}

\# Source knowledge with transferable principles:
\ph{source\_knowledge}

\# Target problem: \ph{tgt\_prob\_name}

Write a \ph{tgt\_func\_name} function for \ph{tgt\_prob\_name}.

\ph{tgt\_func\_desc}

\# Function Signature:
\ph{tgt\_func\_sign}

\# Baseline: \ph{tgt\_func\_seed}

Baseline score: \ph{baseline\_score}

\# Current adapted knowledge | score: \ph{parent\_score}
\ph{parent\_knowledge}

Identify a principle from the source knowledge that the current adapted knowledge is missing or underusing.

Your response must contain two parts:

1. \textbf{Analysis}: What transferable principles from the source problem is the current approach ignoring or underusing? Think beyond the obvious. Considering: \ph{tgt\_hint}

2. \textbf{Strategy}: A concrete strategy for \ph{tgt\_prob\_name} that exploits these transferred principles. Must differ structurally from previous attempts --- hyperparameter-only changes are \emph{not} new strategies.
\end{prompt}

\begin{prompt}[m2-prompt]
\# Source problem: \ph{src\_prob\_name}

Function: \ph{src\_func\_name}

\ph{src\_func\_sign}

\ph{src\_func\_desc}

\# Source knowledge for parameter guidance:
\ph{source\_knowledge}

\# Target problem: \ph{tgt\_prob\_name}

Write a \ph{tgt\_func\_name} function for \ph{tgt\_prob\_name}.

\ph{tgt\_func\_desc}

\# Function Signature:
\ph{tgt\_func\_sign}

\# Baseline: \ph{tgt\_func\_seed}

Baseline score: \ph{baseline\_score}

\# Current adapted knowledge | score: \ph{parent\_score}
\ph{parent\_knowledge}

Refine the current knowledge by adjusting the strategy's parameters or priorities, guided by source problem insights.

Your response must contain two parts:

1. \textbf{Analysis}: What transferable principles from the source problem is the current approach ignoring or underusing? Think beyond the obvious. Considering: \ph{tgt\_hint}

2. \textbf{Strategy}: A concrete strategy for \ph{tgt\_prob\_name} that exploits these transferred principles. Must differ structurally from previous attempts --- hyperparameter-only changes are \emph{not} new strategies.
\end{prompt}

\begin{prompt}[s1-prompt]
\# Source problem: \ph{src\_prob\_name}

Function: \ph{src\_func\_name}

\ph{src\_func\_sign}

\ph{src\_func\_desc}

\# Source knowledge lineage for reference:
\ph{source\_knowledge\_set}

\# Target problem: \ph{tgt\_prob\_name}

Write a \ph{tgt\_func\_name} function for \ph{tgt\_prob\_name}.

\ph{tgt\_func\_desc}

\# Function Signature:
\ph{tgt\_func\_sign}

\# Baseline: \ph{tgt\_func\_seed}

Baseline score: \ph{baseline\_score}

\# Knowledge adaptation history (root → leaf):
\ph{path\_knowledge\_with\_scores}

Analyze how the adapted knowledge evolved. What transfer principles worked and what didn't? Synthesize the best insights.

Your response must contain two parts:

1. \textbf{Analysis}: What transferable principles from the source problem is the current approach ignoring or underusing? Think beyond the obvious. Considering: \ph{tgt\_hint}

2. \textbf{Strategy}: A concrete strategy for \ph{tgt\_prob\_name} that exploits these transferred principles. Must differ structurally from previous attempts --- hyperparameter-only changes are \emph{not} new strategies.
\end{prompt}

\begin{prompt}[implement-prompt]
\# Source problem: \ph{src\_prob\_name}

Function: \ph{src\_func\_name}

\ph{src\_func\_sign}

\ph{src\_func\_desc}

\# Source code for implementation reference:
\ph{source\_code}

\# Target problem: \ph{tgt\_prob\_name}

Write a \ph{tgt\_func\_name} function for \ph{tgt\_prob\_name}.

\ph{tgt\_func\_desc}

\# Function Signature:
\ph{tgt\_func\_sign}

\# Baseline: \ph{tgt\_func\_seed}

Baseline score: \ph{baseline\_score}

\# Previous adaptation 1 | score: \ph{ref\_score\_1}
\ph{ref\_code\_1}

\# Previous adaptation 2 | score: \ph{ref\_score\_2}
\ph{ref\_code\_2}

\# Knowledge to implement:
\ph{knowledge}

Implement \ph{tgt\_func\_name} that faithfully realizes the knowledge above, transferring patterns from the source code to \ph{tgt\_prob\_name}. The previous adaptations show what has been tried --- use them to avoid repeating mistakes, not as templates to copy. Hint: \ph{tgt\_hint}
\end{prompt}

\paragraph{Prompts for Dual Top-Down and Bottom-Up ReEvo.}
Below we present the prompts used in Dual ReEvo, where a top-down knowledge population and a bottom-up code population evolve cooperatively through crossover, grounding, and distillation.

\begin{prompt}[generator-system-prompt]
You are an expert in the domain of optimization heuristics. Your task is to design heuristics that can effectively solve optimization problems.
\end{prompt}

\begin{prompt}[reflector-system-prompt]
You are an expert in the domain of optimization heuristics. Your task is to give hints to design better heuristics.
\end{prompt}

\begin{prompt}[init-knowledge-prompt]
\# Task: Write a \ph{func\_name} function for \ph{prob\_name}.

\ph{func\_desc}

\# Function Signature:
\ph{func\_sign}

\# Baseline: \ph{func\_seed}

Baseline score: \ph{baseline\_score}

\# Hints:
\ph{lt\_reflection}

Describe a general principle for designing a good \ph{func\_name}. Keep it concise, at most 80 words, with no code inside the principle. Then, write one reference \ph{func\_name} implementing your principle.
\end{prompt}

\begin{prompt}[init-code-prompt]
\# Task: Write a \ph{func\_name} function for \ph{prob\_name}.

\ph{func\_desc}

\# Function Signature:
\ph{func\_sign}

\# Baseline: \ph{func\_seed}

Baseline score: \ph{baseline\_score}

\# Hints:
\ph{lt\_reflection}

Write a creative \ph{func\_name} that achieves a lower score than the baseline.
\end{prompt}

\begin{prompt}[implement-prompt]
\# Task: Write a \ph{func\_name} function for \ph{prob\_name}.

\ph{func\_desc}

\# Function Signature:
\ph{func\_sign}

\# Baseline: \ph{func\_seed}

Baseline score: \ph{baseline\_score}

\# Principle to implement:
\ph{knowledge}

Implement \ph{func\_name} following the principle above. Return the code only.
\end{prompt}

\begin{prompt}[knowledge-crossover-prompt]
\# Task: Write a \ph{func\_name} function for \ph{prob\_name}.

\ph{func\_desc}

\# Function Signature:
\ph{func\_sign}

\# Baseline: \ph{func\_seed}

Baseline score: \ph{baseline\_score}

\# Worse knowledge | score: \ph{worse\_score}
\ph{worse\_knowledge}

\# Better knowledge | score: \ph{better\_score}
\ph{better\_knowledge}

\# Reflection:
\ph{st\_reflection}

Synthesize a new principle that combines what works in both and discards what does not, adding one new insight of your own. Keep the principle at most 80 words, with no code inside it. Then, write one reference \ph{func\_name} implementing the synthesized principle.
\end{prompt}

\begin{prompt}[code-crossover-prompt]
\# Task: Write a \ph{func\_name} function for \ph{prob\_name}.

\ph{func\_desc}

\# Function Signature:
\ph{func\_sign}

\# Baseline: \ph{func\_seed}

Baseline score: \ph{baseline\_score}

\# Worse code | score: \ph{worse\_score}
\ph{worse\_code}

\# Better code | score: \ph{better\_score}
\ph{better\_code}

\# Reflection:
\ph{st\_reflection}

Write an improved \ph{func\_name} that achieves a lower score. Return the code only.
\end{prompt}

\begin{prompt}[grounding-prompt]
\# Task: Write a \ph{func\_name} function for \ph{prob\_name}.

\ph{func\_desc}

\# Function Signature:
\ph{func\_sign}

\# Baseline: \ph{func\_seed}

Baseline score: \ph{baseline\_score}

\# General principle:
\ph{knowledge}

\# Current best code | score: \ph{code\_score}
\ph{code}

\# Prior reflection:
\ph{lt\_reflection}

Write a new \ph{func\_name} that follows the general principle above and explores additional structure that the principle does not capture. The new code should be meaningfully different from the current best, not a tweak. Return the code only.
\end{prompt}

\begin{prompt}[distillation-prompt]
\# Task: Write a \ph{func\_name} function for \ph{prob\_name}.

\ph{func\_desc}

\# Function Signature:
\ph{func\_sign}

\# Baseline: \ph{func\_seed}

Baseline score: \ph{baseline\_score}

\# Current best knowledge | score: \ph{elite\_knowledge\_score}
\ph{elite\_knowledge}

\# Current best code | score: \ph{elite\_code\_score}
\ph{elite\_code}

\# Prior reflection:
\ph{lt\_reflection}

The best code above may embody strategies that the current best knowledge does not articulate. Distill a new, sharper principle that captures what makes the code effective, going beyond the current knowledge. Keep the principle at most 80 words, with no code inside it. Then, write one reference \ph{func\_name} implementing the distilled principle.
\end{prompt}

\begin{prompt}[short-term-knowledge-reflection-prompt]
Below are two knowledge principles for designing \ph{func\_name} for \ph{prob\_name}.

\ph{func\_desc}

The second principle led to better performance.

\# Worse knowledge | score: \ph{worse\_score}
\ph{worse\_knowledge}

\# Better knowledge | score: \ph{better\_score}
\ph{better\_knowledge}

What is the worse knowledge missing that the better one captures? Respond with one actionable hint, using fewer than 20 words.
\end{prompt}

\begin{prompt}[short-term-code-reflection-prompt]
Below are two \ph{func\_name} functions for \ph{prob\_name}.

\ph{func\_desc}

The second version performs better than the first.

\# Worse code | score: \ph{worse\_score}
\ph{worse\_code}

\# Better code | score: \ph{better\_score}
\ph{better\_code}

Respond with one hint for designing better heuristics, using fewer than 20 words.
\end{prompt}

\begin{prompt}[long-term-reflection-prompt]
You are refining a long-term reflection on how to design better \ph{func\_name} heuristics for \ph{prob\_name}.

\# Prior long-term reflection:
\ph{prior\_lt\_reflection}

\# New short-term hints from this iteration:
\ph{st\_reflections}

Update the long-term reflection by integrating the new hints with the prior one. Keep it actionable and concise, at most 50 words. Respond with the updated reflection only.
\end{prompt}

\paragraph{Prompts for Sparse Evaluation ReEvo.}
This setting reuses the same initialization, mutation, long-term reflection, and all fully evaluated pairwise prompts as standard ReEvo. The only new prompts arise when at least one candidate in a selected pair has not been evaluated yet; in that case, the model must reason under partial feedback.

\begin{prompt}[bottom-up-uncertain-short-term-reflection-prompt]
Below are two candidate \ph{func\_name} implementations for \ph{prob\_name}.

\ph{func\_desc}

\# Objective: \ph{objective\_desc}

At least one candidate has not been evaluated yet, so performance is unknown.

\# Candidate A | \ph{status\_a}
\ph{code\_a}

\# Candidate B | \ph{status\_b}
\ph{code\_b}

Compare them structurally and give one actionable hint for designing a stronger heuristic in fewer than 20 words.
\end{prompt}

\begin{prompt}[top-down-uncertain-short-term-reflection-prompt]
Below are two candidate knowledge sets for \ph{func\_name} on \ph{prob\_name}.

\ph{func\_desc}

\# Objective: \ph{objective\_desc}

At least one knowledge set has not been validated by evaluation yet.

\# Knowledge A | \ph{status\_a}
\ph{knowledge\_a}

\# Implementation A
\ph{code\_a}

\# Knowledge B | \ph{status\_b}
\ph{knowledge\_b}

\# Implementation B
\ph{code\_b}

Compare the two hypotheses and give one actionable hint for the more promising design direction in fewer than 20 words.
\end{prompt}

\begin{prompt}[bottom-up-uncertain-crossover-prompt]
\# Task: Write a \ph{func\_name} function for \ph{prob\_name}.

\ph{func\_desc}

\# Objective: \ph{objective\_desc}

\# Baseline: \ph{func\_seed}

Baseline \ph{baseline\_score}

\# Candidate A | \ph{status\_a}
\ph{code\_a}

\# Candidate B | \ph{status\_b}
\ph{code\_b}

\# Reflection:
\ph{st\_reflection}

\# Improved code:
Combine the promising ideas from both candidates, resolve obvious weaknesses, and write a new \ph{func\_name} with lower expected score.
\end{prompt}

\begin{prompt}[top-down-uncertain-crossover-prompt]
\# Task: Write a \ph{func\_name} function for \ph{prob\_name}.

\ph{func\_desc}

\# Objective: \ph{objective\_desc}

\# Baseline: \ph{func\_seed}

Baseline \ph{baseline\_score}

\# Knowledge A | \ph{status\_a}
\ph{knowledge\_a}

\# Implementation A
\ph{code\_a}

\# Knowledge B | \ph{status\_b}
\ph{knowledge\_b}

\# Implementation B
\ph{code\_b}

\# Reflection:
\ph{st\_reflection}

\# Synthesized knowledge:
Combine the most plausible principles from both candidates, add one clarifying insight, then implement \ph{func\_name} with lower expected score.
\end{prompt}
\newpage
\paragraph{System Prompts for SR.}
Below are the system prompts used for symbolic regression.

\begin{prompt}[sr-generator-system-prompt]
You are an expert scientist in data-driven equation discovery. You propose equation skeletons as Python functions, drawing on scientific prior knowledge of the problem domain (physics, biology, chemistry, materials, etc.) and refining them against observed data. Numeric constants in your skeletons are placeholders that will be fitted by a numerical optimizer afterwards --- you design the form, not the values.
\end{prompt}

\begin{prompt}[sr-reflector-system-prompt]
You are an expert scientist in data-driven equation discovery. You analyze fitted equation skeletons to identify which functional terms capture the underlying scientific relationship and which terms are redundant or missing.
\end{prompt}

\paragraph{System Prompts for PE.}
Below are the system prompts used for protein engineering.

\begin{prompt}[pe-generator-system-prompt]
You are an expert in protein sequence design and Bayesian-style experimental prioritization. You design vectorized acquisition formulas that guide a fixed Bayesian optimization loop for choosing which protein sequence to test next.
\end{prompt}

\begin{prompt}[pe-reflector-system-prompt]
You are an expert in protein sequence design and experimental prioritization. You analyze acquisition formulas to identify which Bayesian optimization trade-offs improve the sequences discovered under a fixed budget.
\end{prompt}

\appsection{Additional Experimental Results}
\label{app:add_exp}

\appsubsection{Supplementary Evaluation}

\paragraph{Cross-Backbone Transfer under Shared Function Signatures.}
We further study transfer from TSP ACO to TSP GLS, TSP POMO, and TSP DIFUSCO, where the function signature is shared but the semantic role of the generated heuristic changes across solver backbones. This setting tests whether the transferred object captures reusable design knowledge rather than merely matching an interface. As shown in Table~\ref{tab:app:tsp_transfer_results}, top-down improves over bottom-up in $8/9$ target-size settings for both ReEvo and MCTS-AHD. Using the nine matched settings as paired observations and testing $H_1:\mathrm{gap}_{\mathrm{TD}}<\mathrm{gap}_{\mathrm{BU}}$ with a one-sided Wilcoxon signed-rank test, we obtain $p=0.0059$ for ReEvo and $p=0.0098$ for MCTS-AHD. The gains are especially clear on TSP DIFUSCO, suggesting that top-down transfer is useful when the target solver changes substantially: the same signature can hide different algorithmic semantics, while knowledge-level transfer can preserve higher-level principles that remain useful across backbones.

\begin{table}[h]
\centering
\caption{Cross-backbone transfer from TSP ACO to TSP GLS, POMO, and DIFUSCO under shared function signatures. Entries report optimal gap (\%, lower is better), shown as mean $\pm$ standard deviation over five runs.}
\label{tab:app:tsp_transfer_results}
\setlength{\tabcolsep}{5pt}
\renewcommand{\arraystretch}{1.2}
\resizebox{\linewidth}{!}{%
\begin{tabular}{|l|*{9}{C{1.8cm}|}}
\hline
\multirow{2}{*}{\textbf{Method}}
  & \multicolumn{3}{c|}{\textbf{GLS}}
  & \multicolumn{3}{c|}{\textbf{POMO}}
  & \multicolumn{3}{c|}{\textbf{DIFUSCO}} \\
\cline{2-10}
 & 100 & 200 & 500 & 100 & 200 & 500 & 100 & 200 & 500 \\
\hline
ReEvo BU + CPT
  & $2.15 \pm 0.02$ & $4.04 \pm 0.07$ & $4.43 \pm 0.15$
  & $1.66 \pm 0.07$ & $4.29 \pm 0.04$ & $12.57 \pm 0.34$
  & $6.03 \pm 0.64$ & $7.47 \pm 0.20$ & $10.36 \pm 0.08$ \\
ReEvo TD + CPT
  & $2.36 \pm 0.07$ & $3.76 \pm 0.19$ & $3.74 \pm 0.16$
  & $1.48 \pm 0.05$ & $4.06 \pm 0.07$ & $10.91 \pm 0.25$
  & $5.02 \pm 0.16$ & $6.94 \pm 0.08$ & $8.12 \pm 0.27$ \\
\hline
MCTS-AHD BU + CPT
  & $2.40 \pm 0.05$ & $3.95 \pm 0.08$ & $4.32 \pm 0.13$
  & $1.49 \pm 0.04$ & $4.19 \pm 0.13$ & $11.41 \pm 0.23$
  & $6.38 \pm 0.12$ & $7.32 \pm 0.09$ & $9.94 \pm 0.08$ \\
MCTS-AHD TD + CPT
  & $2.23 \pm 0.05$ & $3.46 \pm 0.08$ & $4.01 \pm 0.17$
  & $1.74 \pm 0.08$ & $3.96 \pm 0.02$ & $11.06 \pm 0.13$
  & $5.66 \pm 0.25$ & $6.95 \pm 0.02$ & $8.04 \pm 0.32$ \\
\hline
\end{tabular}%
}
\end{table}

\paragraph{Sparse Evaluation on Discrete Location Problems.}
We additionally evaluate sparse population-based search on four discrete location (DL) problems: $p$-Center, $p$-Cover, $p$-Dispersion, and $p$-Median. Table~\ref{tab:app:dl_results} reports performance gap (\%, lower is better) on instance sizes $100$ and $200$. The results show that sparse TD and dual TD--BU remain competitive despite the reduced evaluation budget. In particular, ReEvo TD (sparse) improves over ReEvo BU (sparse) on $5/8$ settings, while ReEvo TD \& BU (sparse) achieves the best or tied-best result on $5/8$ settings, including all $p$-Dispersion and $p$-Median cases. This suggests that knowledge-level search is especially useful when empirical feedback is limited. At the same time, the mixed behavior on $p$-Center and $p$-Cover indicates that sparse evaluation can be task-dependent, making it a useful stress test for whether top-down abstractions preserve performance-relevant structure.

\begin{table}[h]
\centering
\caption{Performance gap (\%) on $p$-Center, $p$-Cover, $p$-Dispersion, and $p$-Median. Results are reported as mean $\pm$ standard deviation over five runs.}
\label{tab:app:dl_results}
\setlength{\tabcolsep}{5pt}
\renewcommand{\arraystretch}{1.2}
\resizebox{\linewidth}{!}{%
\begin{tabular}{|l|*{8}{C{1.9cm}|}}
\hline
\multirow{2}{*}{\textbf{Method}}
  & \multicolumn{2}{c|}{\textbf{$p$-Center}}
  & \multicolumn{2}{c|}{\textbf{$p$-Cover}}
  & \multicolumn{2}{c|}{\textbf{$p$-Dispersion}}
  & \multicolumn{2}{c|}{\textbf{$p$-Median}} \\
\cline{2-9}
 & 100 & 200 & 100 & 200 & 100 & 200 & 100 & 200 \\
\hline
ReEvo BU
  & $0.242 \pm 0.235$ & $0.142 \pm 0.342$
  & $0.132 \pm 0.198$ & $0.232 \pm 0.534$
  & $0.642 \pm 0.531$ & $0.724 \pm 0.699$
  & $0.424 \pm 0.442$ & $0.590 \pm 0.651$ \\
ReEvo TD
  & $0.314 \pm 0.497$ & $0.452 \pm 0.421$
  & $0.031 \pm 0.011$ & $0.065 \pm 0.023$
  & $0.566 \pm 0.552$ & $0.676 \pm 0.705$
  & $0.364 \pm 0.308$ & $0.536 \pm 0.514$ \\
ReEvo TD \& BU
  & $0.142 \pm 0.298$ & $0.164 \pm 0.178$
  & $0.077 \pm 0.068$ & $0.104 \pm 0.144$
  & $0.531 \pm 0.607$ & $0.656 \pm 0.713$
  & $0.069 \pm 0.133$ & $0.097 \pm 0.178$ \\
\hline
ReEvo BU (sparse)
  & $0.191 \pm 0.297$ & $0.009 \pm 0.315$
  & $0.022 \pm 0.017$ & $0.000 \pm 0.020$
  & $0.572 \pm 0.614$ & $0.693 \pm 0.538$
  & $0.171 \pm 0.138$ & $1.238 \pm 0.607$ \\
ReEvo TD (sparse)
  & $0.000 \pm 0.233$ & $0.029 \pm 0.386$
  & $0.000 \pm 0.040$ & $0.073 \pm 0.030$
  & $0.192 \pm 0.598$ & $0.225 \pm 0.267$
  & $0.369 \pm 0.306$ & $0.570 \pm 0.352$ \\
ReEvo TD \& BU (sparse)
  & $0.079 \pm 0.309$ & $0.000 \pm 0.306$
  & $0.265 \pm 0.236$ & $0.382 \pm 0.218$
  & $0.000 \pm 0.169$ & $0.000 \pm 0.098$
  & $0.000 \pm 0.196$ & $0.000 \pm 0.333$ \\
\hline
\end{tabular}%
}
\end{table}

For statistical comparison, we treat each problem--size pair as one matched block, giving eight paired settings in total. Using the mean gap in each block, a Friedman test over the full-evaluation variants gives a significant method effect ($p=0.0439$), with mean ranks $2.625$, $2.000$, and $1.375$ for ReEvo BU, ReEvo TD, and ReEvo TD \& BU, respectively, where lower rank is better.

\appsubsection{Sensitivity Analysis}

\paragraph{Sensitivity to LLM Backbones.}
Figure~\ref{fig:app_02} evaluates whether the top-down advantage is sensitive to the choice of LLM backbone. We compare bottom-up and top-down variants under two lightweight models, \texttt{gemini-2.5-flash-lite} and \texttt{claude-3-5-haiku}, on TSP transfer tasks. The trend is broadly stable across both ReEvo and MCTS-AHD: top-down search usually achieves lower optimal gap than its bottom-up counterpart, especially on TSP DIFUSCO and larger problem sizes. At the same time, the magnitude of improvement varies across solvers and models, indicating that top-down search still depends on the quality of the LLM's domain abstractions and their realization as code.

\begin{figure}[h]
  \centering
  \includegraphics[width=\textwidth]{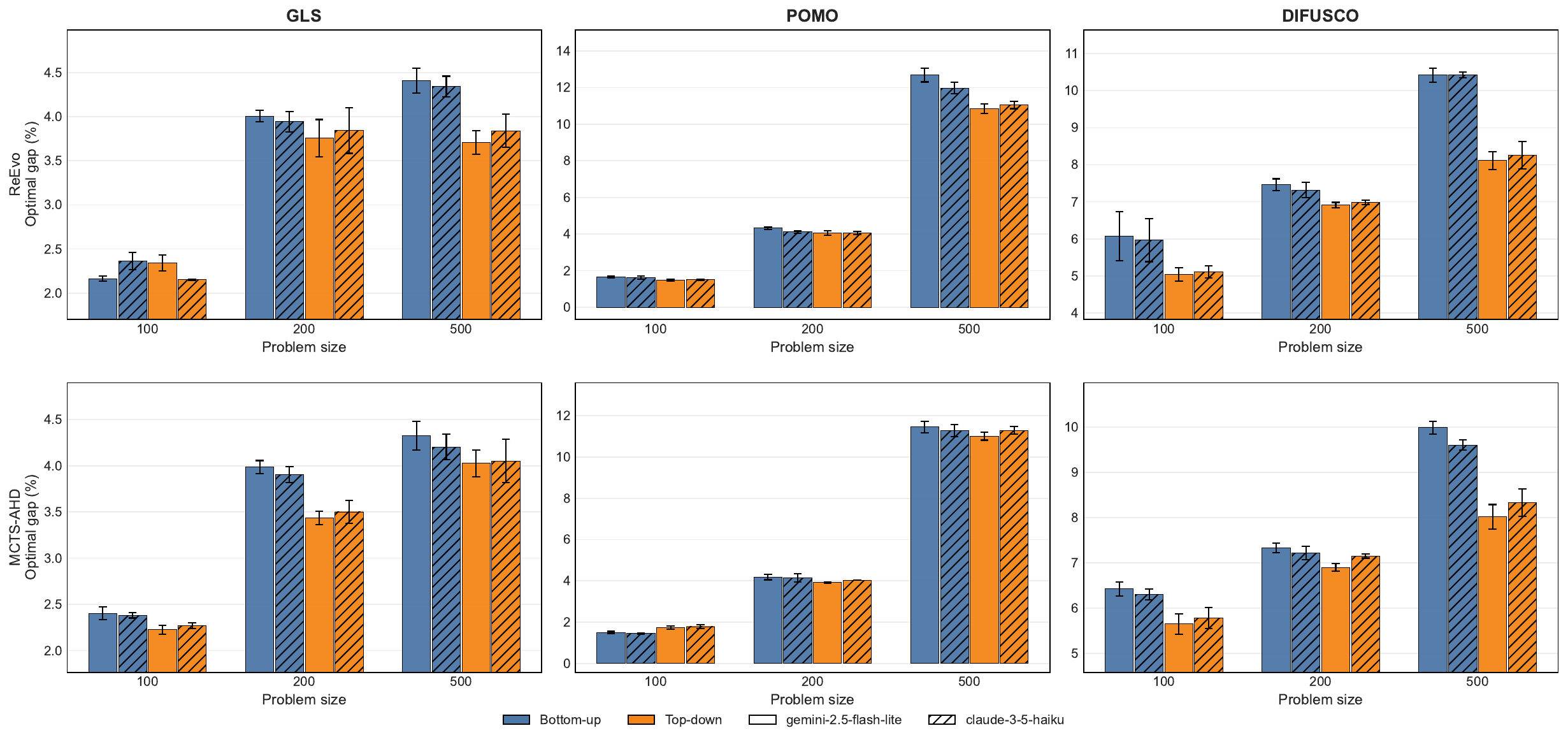}
  \caption{Sensitivity to LLM backbones on TSP transfer tasks. We compare bottom-up and top-down search using \texttt{gemini-2.5-flash-lite} and \texttt{claude-3-5-haiku} across ReEvo and MCTS-AHD. Results report optimal gap (\%, lower is better) on TSP GLS, POMO, and DIFUSCO over problem sizes $100$, $200$, and $500$.}
  \label{fig:app_02}
\end{figure}

\appsubsection{Failure Cases}
\label{app:failure}

\paragraph{Transfer under Misleading Source Information.}
We analyze a failure mode of top-down transfer where the target solver receives transferred information whose source-problem description is inaccurate or semantically mismatched. This setting is challenging for TD because the transferred object is knowledge $K$: if the source description induces an incorrect abstraction, the resulting prior can bias subsequent search before code is even evaluated. Table~\ref{tab:routing_results} shows this effect clearly. ReEvo TD underperforms ReEvo BU in $10/12$ settings, and MCTS-AHD TD underperforms MCTS-AHD BU in $9/12$ settings. The degradation is especially visible on TSP and CVRP, where misleading transferred knowledge can distort core assumptions about the routing structure. These results support the limitation discussed in Appendix~\ref{app:qa}: top-down search is effective when $K$ preserves task-relevant structure, but can be harmful when the knowledge abstraction is built from incorrect or misaligned problem information.

\begin{table}[h]
\centering
\caption{Failure cases for cross-problem transfer when the source-problem information is intentionally misleading or mismatched. Entries report optimal gap (\%, lower is better) on ACO TSP, CVRP, OVRP, and LVRP across problem sizes $100$, $200$, and $500$. TSP ACO
uses the TSP constructive heuristic as source, while all other settings use TSP ACO as source.}
\label{tab:routing_results}
\setlength{\tabcolsep}{5pt}
\renewcommand{\arraystretch}{1.2}
\resizebox{\linewidth}{!}{%
\begin{tabular}{|l|*{12}{C{1.25cm}|}}
\hline
\multirow{2}{*}{\textbf{Method}}
  & \multicolumn{3}{c|}{\textbf{TSP}}
  & \multicolumn{3}{c|}{\textbf{CVRP}}
  & \multicolumn{3}{c|}{\textbf{OVRP}}
  & \multicolumn{3}{c|}{\textbf{LVRP}} \\
\cline{2-13}
 & 100 & 200 & 500 & 100 & 200 & 500 & 100 & 200 & 500 & 100 & 200 & 500 \\
\hline
ReEvo BU + CPT
  & 0.75 & 4.21 & 9.08
  & 3.80 & 7.11 & 9.34
  & 4.39 & 10.42 & 20.96
  & 2.80 & 5.03 & 7.92 \\
ReEvo TD + CPT
  & 1.81 & 8.59 & 17.83
  & 6.58 & 11.17 & 10.98
  & 8.11 & 10.77 & 16.84
  & 2.97 & 5.14 & 5.86 \\
\hline
MCTS-AHD BU + CPT
  & 0.43 & 2.65 & 6.19
  & 4.91 & 8.67 & 10.27
  & 4.71 & 6.61 & 8.70
  & 2.10 & 4.51 & 6.42 \\
MCTS-AHD TD + CPT
  & 0.55 & 3.84 & 9.67
  & 5.63 & 9.82 & 11.13
  & 1.96 & 4.99 & 7.15
  & 5.40 & 8.75 & 8.88 \\
\hline
\end{tabular}%
}
\end{table}

\begin{tcolorbox}[
  enhanced,
  colback=white,
  colframe=gray!55,
  boxrule=0.6pt,
  arc=5pt,
  left=8pt,
  right=8pt,
  top=6pt,
  bottom=6pt
]
This section may be further updated as additional experiments are completed. We may also incorporate new experimental results, extended analyses, and supplementary discussions into subsequent revisions of the appendix.
\end{tcolorbox}

\appsection{Qualitative Analysis}
\label{app:qualitative_anal}

\appsubsection{Case Studies of Generated Code and Knowledge}

The best knowledge and code found for the ACO LVRP by ReEvo TD + CPT:

\begin{prompt}[best-knowledge]
The heuristic for the Duration-Limited Vehicle Routing Problem (LVRP) integrates several key principles: 

1. Inverse Cubed Distance Component emphasizes shorter distances, raising the score of closer edges significantly. 

2. Combined Demand Optimization incorporates customer demand and vehicle capacity, penalizing edges leading to nodes that exceed the vehicle's capacity if added to the route. 

3. Maximum Duration Constraint factors in the total distance of routes, with penalties for edges that lead to routes exceeding the total allowable duration (including return to depot). 

4. Connectivity Assessment evaluates edges based on their ability to connect to groups of customers, promoting routes that efficiently utilize both the depot and the customer demands. 

5. Path Diversity encourages exploration of diverse routes, especially underused connections, preventing early looping back to the depot and enhancing coverage.
\end{prompt}

\begin{python}
import numpy as np

def compute_heuristic_matrix(
    dist_mat: np.ndarray,
    demands: np.ndarray,
    vehicle_capacity: float,
    max_duration: float,
) -> np.ndarray:
    n_nodes = dist_mat.shape[0]
    epsilon = 1e-10  # to avoid division by zero
    heuristic_matrix = np.zeros((n_nodes, n_nodes))

    # Calculate the distance heuristic, emphasizing proximity
    distance_heuristic = 1 / (dist_mat + epsilon)**3  # Inverse cubed distance

    for i in range(n_nodes):
        for j in range(n_nodes):
            if i != j:
                # Factor in vehicle's capacity and customer demand
                if demands[j] + demands[i] <= vehicle_capacity:
                    capacity_penalty = 1
                else:
                    capacity_penalty = 0.5  # penalize if demand exceeds capacity

                # Factor in maximum duration constraint
                path_duration = dist_mat[i][0] + dist_mat[i][j] + dist_mat[j][0]
                duration_penalty = 1 if path_duration <= max_duration else 0.5  # penalize over max duration

                # Combine the factors to establish the heuristic score
                heuristic_score = (distance_heuristic[i][j] 
                                   * capacity_penalty 
                                   * duration_penalty)
                heuristic_matrix[i][j] = heuristic_score

    return heuristic_matrix
\end{python}

The best knowledge and code found for the ACO CVRP by MCTS-AHD TD + CPT:

\begin{prompt}[best-knowledge]
For the Capacitated Vehicle Routing Problem (CVRP), an effective heuristic should guide route construction toward compact and feasible solutions:

1. Inverse Distance Prioritization favors shorter edges by assigning larger scores to closer customer pairs, encouraging compact route construction.

2. Capacity-Aware Demand Scaling incorporates customer demand relative to vehicle capacity, reducing the attractiveness of nodes with larger demand and eliminating infeasible demand cases.

3. Depot-Based Clustering Assessment evaluates whether two customers form a structurally coherent pair with respect to their distances from the depot, promoting customer groupings that are geographically meaningful for route formation.

4. Angular Efficiency encourages transitions between customers that appear to lie in compatible routing directions, helping avoid inefficient turns or poorly aligned route extensions.

5. Combined Structural Scoring multiplies distance, demand, clustering, and angular factors into a single edge heuristic, ensuring that selected edges are not only short but also capacity-aware and structurally suitable for CVRP routes.
\end{prompt}

\begin{python}
import numpy as np

def compute_heuristic_matrix(
    dist_mat: np.ndarray,
    demands: np.ndarray,
    vehicle_capacity: float,
) -> np.ndarray:
    # Adding a small epsilon value to avoid division by zero
    epsilon = 1e-10
    n_nodes = dist_mat.shape[0]

    # Step 1: Inverse distance heuristic
    inverse_distance = 1 / (dist_mat + epsilon)

    # Step 2: Demand factor normalization
    demand_factor = 1 - (demands / vehicle_capacity)  # Penalty for higher demands
    demand_factor[demands > vehicle_capacity] = 0  # Set demand exceeding capacity to zero

    # Step 3: Calculate clustering coefficient
    clustering_coefficient = np.zeros((n_nodes, n_nodes))
    for i in range(n_nodes):
        for j in range(n_nodes):
            if i != j:
                # Simple clustering score based on distance from depot
                clustering_coefficient[i, j] = (1 - (dist_mat[0, i] + dist_mat[0, j]) / (dist_mat[i, j] + epsilon)) ** 2

    # Step 4: Angular efficiency (favoring closer nodes that are in the same routing direction)
    angular_efficiency = np.ones((n_nodes, n_nodes))  # Placeholder for angular efficiency
    for i in range(1, n_nodes):  # Ignore depot
        for j in range(1, n_nodes):  # Ignore depot
            angular_efficiency[i, j] = np.cos(np.pi * dist_mat[i, j] / np.max(dist_mat))  # Example of an angular concept

    # Step 5: Combine heuristic scores
    heuristic_matrix = (
        inverse_distance * 
        demand_factor[:, np.newaxis] * 
        clustering_coefficient * 
        angular_efficiency
    )

    # Step 6: Ensure no self-connections
    np.fill_diagonal(heuristic_matrix, 0)

    return heuristic_matrix
\end{python}

The best knowledge and code found for the ACO OVRP by MCTS-AHD TD + CPT:

\begin{prompt}[best-knowledge]
In Open Vehicle Routing Problem (OVRP), route construction is driven by customer-to-customer transitions that remain capacity-feasible while avoiding unnecessary return-to-depot bias:

1. Inverse Squared Distance emphasizes short transitions, strongly favoring nearby customers to maintain compact routes.

2. Residual Capacity Filtering evaluates whether a candidate customer can be feasibly served given the remaining capacity, down-weighting or excluding infeasible extensions.

3. Progressive Demand Accumulation approximates the load already carried along the route, discouraging transitions that would likely violate capacity constraints later.

4. Demand-Aware Scoring prioritizes customers with smaller demands relative to remaining capacity, enabling more flexible route continuation.

5. Exploration Weighting promotes connections toward less dominant regions by incorporating a global interaction term over distances and demands, reducing premature concentration on local clusters.
\end{prompt}

\begin{python}
import numpy as np

def compute_heuristic_matrix(
    dist_mat: np.ndarray,
    demands: np.ndarray,
    vehicle_capacity: float,
) -> np.ndarray:
    epsilon = 1e-5  # Small value to avoid division by zero
    n_nodes = dist_mat.shape[0]

    # Check if demand array matches distance matrix dimensions
    if n_nodes != demands.shape[0]:
        raise ValueError("Demand array must match the size of the distance matrix.")

    # Initialize heuristic matrix
    heuristic_matrix = np.zeros_like(dist_mat)

    for i in range(n_nodes):  # Current node
        for j in range(n_nodes):  # Next possible node
            if i != j:
                # Check feasibility of visiting node j from node i
                if i == 0:
                    total_demand = 0  # At depot, no demand yet
                else:
                    total_demand = np.sum(demands[np.where(dist_mat[i, :] < dist_mat[i, j])])

                # Ensure we can pick up the demand of node j
                remaining_capacity = vehicle_capacity - total_demand
                if demands[j] <= remaining_capacity:
                    # Inverse squared distance for closer nodes
                    distance_score = 1 / (dist_mat[i, j] ** 2 + epsilon)
                    # Demand factor to prioritize less demand
                    demand_score = (remaining_capacity - demands[j]) / vehicle_capacity

                    # Visit weight to encourage connecting lesser-visited nodes
                    visit_weight = np.sum(1 / (dist_mat[i, :] + epsilon) * (demands - np.min(demands)))

                    # Combine the factors into the heuristic score
                    heuristic_score = distance_score * demand_score * (1 + visit_weight)  # Emphasizing distance and demand

                    # Assign to the heuristic matrix
                    heuristic_matrix[i, j] = heuristic_score

    # Normalize by row sums to ensure comparability
    row_sums = np.sum(heuristic_matrix, axis=1, keepdims=True)
    heuristic_matrix /= (row_sums + epsilon)  # Avoid division by zero if all scores are zero

    # No self-connections
    np.fill_diagonal(heuristic_matrix, 0)

    return heuristic_matrix
\end{python}

The best knowledge and code found for the GLS TSP by ReEvo TD sparse:

\begin{prompt}[best-knowledge]
To improve edge prioritization in the Traveling Salesman Problem, we can leverage advanced clustering techniques along with a few heuristics regarding edge structure. The following key strategies will inform our new heuristic: 

1. Geometric Distance Analysis: Rather than purely using distance, we can adjust weights by considering the relative positions of cities in 2D space, minimizing poorly connected edges due to geographical arrangements. 

2. Connection Density: By analyzing local connectivity density and adjusting penalties based on edge density within clusters, we can identify edges that bridge sparsely connected clusters. 

3. Edge-to-Clustering Ratio: An edge-centric measure that captures how the edge compares to others in its locality, using both distance and connectivity strength, can reveal less reliable edges. 

4. Integration of Angles and Distances: By focusing explicitly on angular deviation, which can signal potential inefficiencies, along with using lengths, we can comprehensively analyze structural integrity. 

5. Cluster Quality: We can assess the quality of cities within clusters, favoring those induced by DBSCAN where clusters display compactness, ensuring stronger structural edges are prioritized instead of weak cross-cluster links. 

Overall, our approach will consider both local structures and larger cluster formations for more reliable results, guiding the heuristic towards edges that critically impact overall tour efficiency.
\end{prompt}

\begin{python}
import numpy as np
from sklearn.cluster import DBSCAN
from numpy.linalg import norm

def compute_penalty_guide(dist_mat: np.ndarray) -> np.ndarray:
    n = dist_mat.shape[0]  # number of cities
    edge_priority = np.zeros_like(dist_mat)

    # Step 1: Perform clustering using DBSCAN on distance matrix (treating distances as a metric)
    dbscan = DBSCAN(eps=1.5, min_samples=3, metric='precomputed')
    cluster_labels = dbscan.fit_predict(dist_mat)

    # Step 2: Evaluate edges based on distance and connection density
    for i in range(n):
        for j in range(n):
            if i == j:
                continue  # No self-loop

            edge_distance = dist_mat[i][j]  

            # Angle and density-based penalty
            if cluster_labels[i] != cluster_labels[j]:
                edge_priority[i][j] += edge_distance * 3  # Strong penalty for cross-cluster edges

            # Local neighborhood analysis
            local_distances = np.delete(dist_mat[i], i)  # Remove self-referencing
            nearest_neighbors_indices = np.argsort(local_distances)[:5]  # Get 5 nearest neighbors
            local_avg_distance = np.mean(local_distances[nearest_neighbors_indices])
            connection_density = np.sum(local_distances[nearest_neighbors_indices] < local_avg_distance) / len(nearest_neighbors_indices)
            deviation = edge_distance - local_avg_distance

            # Apply penalties
            if deviation > 0:
                angle_penalty = np.arccos(np.clip((edge_distance**2 - local_avg_distance**2 - local_avg_distance**2) / (-2 * local_avg_distance * local_avg_distance), -1, 1))
                edge_priority[i][j] += edge_distance * (1 + (deviation / local_avg_distance) + angle_penalty * 3 + (1 - connection_density) * 2)

    # Normalize edge priority values while preventing division by zero
    edge_priority = edge_priority / (edge_priority.max() + 1e-9)
    return edge_priority
\end{python}

The best knowledge and code found for the DIFUSCO TSP Diagonal by ReEvo TD\&BU (multi-task):

\begin{prompt}[best-general-knowledge]
To improve edge selection in TSP across various distributions, leverage logarithmic transformations on distances for closer neighbors, while applying systematic negative penalties for distant edges. Enhance local density awareness by calculating neighbor proximity, and normalize biases relative to the local connectivity context. Aim for a balanced bias impact across the graph to adapt to distinct spatial patterns without overfitting to any single distribution.
\end{prompt}

\begin{python}
import torch

def compute_edge_bias(distance_matrix: torch.Tensor) -> torch.Tensor:
    num_nodes = distance_matrix.shape[0]
    bias_matrix = torch.zeros_like(distance_matrix)
    K = min(num_nodes // 10 + 1, 5)  # Adaptive K based on node density
    
    # Pre-compute normalized distances to avoid repeated calculations
    max_distance = distance_matrix.max().item() + 1e-10
    scaled_distance = distance_matrix / max_distance
    
    # Calculate average distance and variance to capture structure
    avg_distance = distance_matrix.mean(dim=1)
    var_distance = distance_matrix.var(dim=1)
    
    # Enhanced bias calculation with local clustering and proximity
    for i in range(num_nodes):
        # Identifying top K nearest neighbors, considering only distance for selection
        top_k_indices = torch.topk(distance_matrix[i], k=K, largest=False).indices
        top_k_distances = distance_matrix[i, top_k_indices]
        
        # Calculate statistical characteristics of distances 
        local_mean = avg_distance[i]
        local_var = var_distance[i] 
        density_adjustment = torch.exp(-local_var / (local_mean + 1e-10))  # Local density effect 
        
        # Applying logarithmic transformations for biasing close distances
        bias_matrix[i, top_k_indices] = -torch.log1p(top_k_distances) - (top_k_distances / (local_mean + 1e-10))
        bias_matrix[i, top_k_indices] += density_adjustment * (1.5 - (top_k_distances / (local_mean + 1e-10)))  
        
        # Aggressive negative bias for distant nodes, imposing stronger penalties
        for j in range(num_nodes):
            if j not in top_k_indices:
                penalty = -torch.exp(5 * scaled_distance[i, j] - (local_var / (local_mean + 1e-10))) * density_adjustment  # Adjusted decay based on local clustering
                bias_matrix[i, j] += penalty  
                
        # Centralizing bias for balance across the row
        bias_mean = bias_matrix[i].mean()  # Centering around zero
        bias_matrix[i] -= bias_mean
    
    return bias_matrix
\end{python}

\appsubsection{Case Studies of Knowledge Evolution}

\begin{figure}[h]
  \centering
  \includegraphics[width=\textwidth]{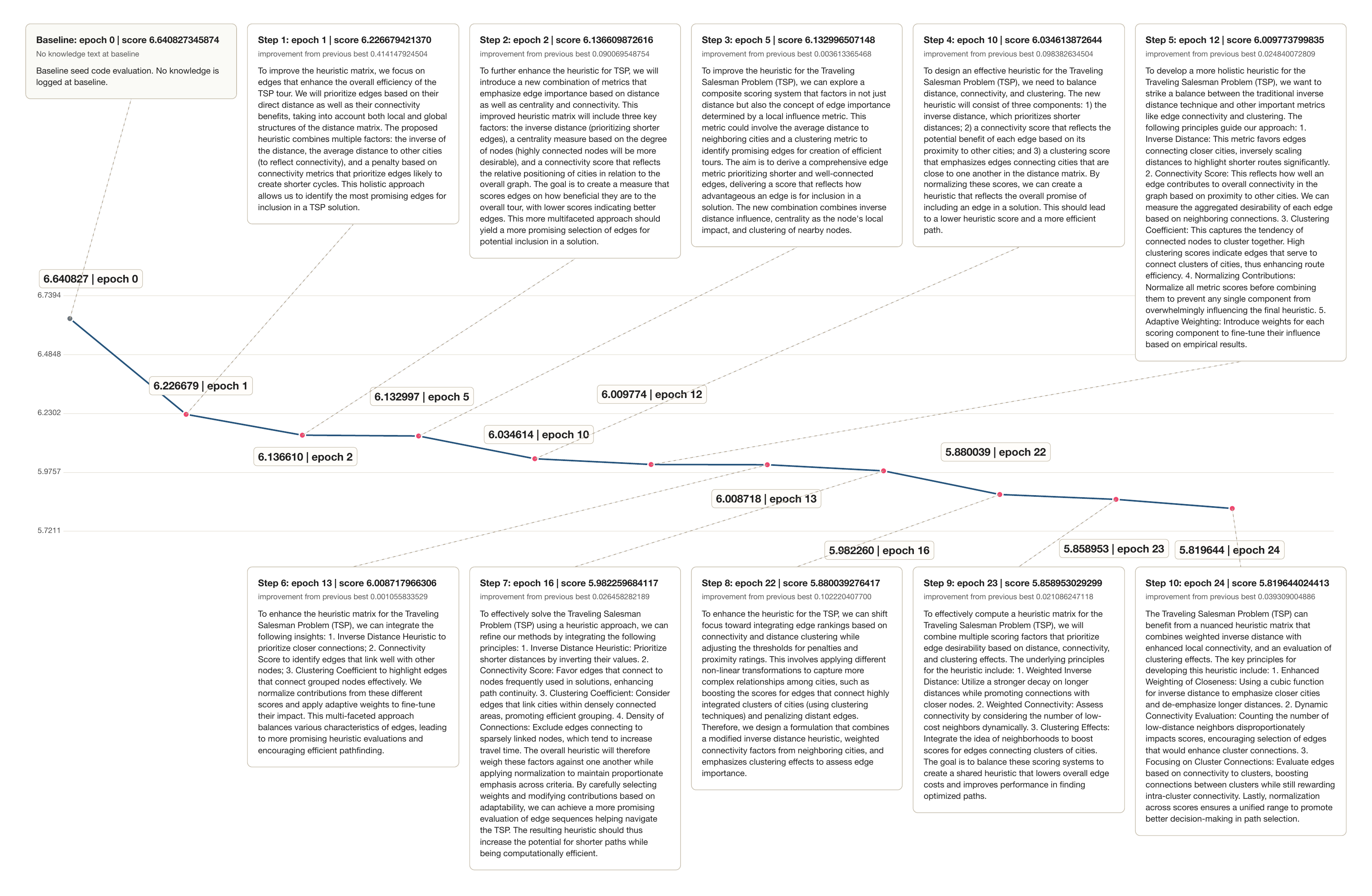}
\caption{Knowledge evolution path for TSP heuristic search on the ACO backbone, showing how selected knowledge states are progressively refined over time and gradually improve the score.}
\end{figure}

\newpage
\appsection{Additional Discussions}
\label{app:add_discuss}

\paragraph{Limitations.}
The main limitation of our approach is the \emph{fidelity of knowledge abstraction}. Top-down AHD is useful when a knowledge state $K$ preserves the performance-relevant structure of a heuristic, but it can fail when $K$ is too coarse, misleading, or difficult to realize as executable code $\theta$. In such cases, the compression benefit of searching in $\mathcal{K}$ may be outweighed by the distortion introduced by abstraction, and bottom-up code-level search can be preferable. This limitation is especially visible when the task description or transferred source information induces an incorrect prior, as discussed in Appendix~\ref{app:failure}.

A second limitation is that top-down search still depends on the LLM's domain understanding and code-generation reliability. Although our implementation keeps LLM-call budgets matched across BU and TD variants, the quality of generated knowledge, the faithfulness of the realization step $K\mapsto\theta$, and the robustness of generated programs can vary across models, tasks, and solver backbones. Finally, our experiments focus on algorithmic design settings where candidate heuristics can be automatically evaluated. Extending the framework to domains with expensive, noisy, delayed, or safety-critical feedback requires additional safeguards beyond the current evaluation protocol.

\paragraph{Broader Impacts.}
This work aims to improve automatic heuristic design by making reusable algorithmic knowledge an explicit search object. A positive impact is that such systems may reduce the amount of manual effort required to design effective heuristics for combinatorial optimization and scientific discovery tasks. The top-down formulation may also make AHD more interpretable, since intermediate knowledge states $K$ expose the design principles that lead to generated programs, rather than only returning opaque code artifacts.

At the same time, automatic heuristic generation can produce incorrect, brittle, or inefficient code if used without validation. In high-stakes applications, generated heuristics should not be deployed solely based on training-set performance; they should be tested under distribution shift, stress cases, and implementation-level constraints. The framework may also increase computational demand through repeated LLM calls and solver evaluations, although sparse evaluation is one step toward reducing empirical evaluation cost. More generally, our results should be interpreted as evidence for a research direction in AHD, not as a guarantee that LLM-generated heuristics are reliable in all downstream settings.

\paragraph{LLM Usage.}
LLMs are a core methodological component of this work. They are used as conditional generators for code $\theta$, knowledge states $K$, short-term reflections, long-term reflections, and realization steps from knowledge to executable implementations. In the bottom-up variants, the LLM primarily generates or modifies code artifacts $A=\alpha(\theta)$ and produces reflections from evaluated code. In the top-down variants, the LLM first generates or refines knowledge $K$ and then realizes it as code $\theta$ for empirical evaluation. All reported comparisons control the number of LLM calls and evaluation budgets as described in Appendix~\ref{app:setup}.

We use LLMs only within the automatic heuristic design pipeline and for ordinary writing or editing assistance during manuscript preparation. The latter does not affect the scientific claims, experimental results, or methodological contributions. All generated heuristics are evaluated by executable benchmark scripts, and the reported results are based on empirical scores rather than on LLM self-assessment.

\paragraph{Reproducibility.}
To support reproducibility, we implement all variants under a unified codebase with shared evaluator interfaces, matched budgets, and task-specific configuration files. For each run, we log the generated code $\theta$, the empirical score $\widehat{L}_{\mathcal{D}}(\theta)$, and, for top-down variants, the associated knowledge state $K$. We also save the best discovered artifact, intermediate populations or elite archives when applicable, and the hyperparameters controlling population size, mutation rate, evaluation budget, sparse-evaluation ratio, and tree-search parameters. The prompts, benchmark scripts, configuration files, and raw experimental logs will be released with the codebase to enable independent verification and extension of the reported results.

\newpage

\begin{center}
This page is intentionally left blank so the paper reaches 75 pages, a beautiful number to me.

You can find my homepage and other publications at: \url{https://haiau2501.github.io/}
\end{center}


\end{document}